\documentclass{article}

\usepackage[preprint]{neurips_2026}

\usepackage[utf8]{inputenc} % allow utf-8 input
\usepackage[T1]{fontenc}    % use 8-bit T1 fonts
\usepackage{hyperref}       % hyperlinks
\usepackage{url}            % simple URL typesetting
\usepackage{booktabs}       % professional-quality tables
\usepackage{amsfonts}       % blackboard math symbols
\usepackage{nicefrac}       % compact symbols for 1/2, etc.
\usepackage{microtype}      % microtypography
\usepackage{xcolor}         % colors
\usepackage{wrapfig}        

\author{Vladimir R. Kostic \\ {CSML, Istituto Italiano di Tecnologia} \\ {University of Novi Sad} \\ {\tt \small vladimir.kostic@iit.it} \And
Karim Lounici \\
CMAP-Ecole Polytechnique \\
{\tt \small karim.lounici@polytechnique.edu} \And
H\'el\`ene Halconruy \\ {SAMOVAR, T\'el\'ecom Sud-Paris} \\ 
{MODAL'X, Universit\'e Paris Nanterre}\\
{\tt \small helene.halconruy@telecom-sudparis.eu } \And 
Timoth\'ee Devergne \\
{CSML \& ATSIM, Istituto Italiano di Tecnologia} \\ {\tt \small timothee.devergne@iit.it} \And
Michele Parrinello \\ {ATSIM, Istituto Italiano di Tecnologia} \\ {\tt \small michele.parrinello@iit.it} \And
Massimiliano Pontil \\ {CSML, Istituto Italiano di Tecnologia} \\ {AI Centre, University College London} \\ {\tt \small massimiliano.pontil@iit.it}
}

\usepackage{microtype}
\usepackage{graphicx}
\usepackage{subcaption}
\usepackage{booktabs} % for professional tables
\usepackage{enumitem}
\usepackage{cancel}

\usepackage{dsfont}

\usepackage{thm-restate}
\usepackage{hyperref}

\usepackage{amsmath}
\usepackage{amssymb}
\usepackage{mathtools}
\usepackage{amsthm}

\theoremstyle{plain}
\newtheorem{theorem}{Theorem}[section]
\newtheorem{proposition}[theorem]{Proposition}
\newtheorem{lemma}[theorem]{Lemma}

\theoremstyle{definition}

\theoremstyle{remark}

\usepackage[disable,textsize=tiny]{todonotes}
\usepackage{amsmath,amsfonts,bm}

\def\ceil#1{\lceil #1 \rceil}
\def\floor#1{\lfloor #1 \rfloor}
\def\1{\bm{1}}

\def\eps{{\varepsilon}}

\DeclareMathAlphabet{\mathsfit}{\encodingdefault}{\sfdefault}{m}{sl}
\SetMathAlphabet{\mathsfit}{bold}{\encodingdefault}{\sfdefault}{bx}{n}

\newcommand{\R}{\mathbb{R}}

\newcommand{\softmax}{\mathrm{softmax}}

\DeclareMathOperator*{\esssup}{\ensuremath{\text{\rm ess\,sup}}}

\DeclareMathOperator*{\dom}{\ensuremath{\text{\rm dom}}}
\DeclareMathOperator*{\range}{\ensuremath{\text{\rm Im}}}

\DeclareMathOperator{\tr}{\ensuremath{\text{\rm tr}}}

\DeclareMathOperator*{\Spec}{\ensuremath{\text{\rm Sp}}}

\providecommand{\norm}[1]{\lVert#1\rVert}

\providecommand{\abs}[1]{\lvert#1\rvert}

\newcommand{\scalarp}[1]{{\langle #1\rangle}}

\newcommand{\N}{\mathbb N}

\renewcommand{\top}{{\raisebox{-0.25ex}{\scalebox{0.6}{$\mathsf{T}$}}}}
\newcommand{\EE}{\ensuremath{\mathbb E}}
\newcommand{\PP}{\ensuremath{\mathbb P}}

\newcommand{\Id}{I}
\newcommand{\w}[1]{\widehat{#1}} %General estimator
\newcommand{\TZ}{Z}  %Koopman injection   
\newcommand{\TP}{P}  % projector 
\newcommand{\TPt}{P_{\param}}  % projector on RKHS param
\newcommand{\TPT}{P_{\trunc}}  % projector truncated
\newcommand{\TPTo}{P_{\trunc}^{\bot}}  % projector truncated orthogonal
\newcommand{\X}{\mathcal{X}} % state space
\newcommand{\RKHS}{\mathcal{H}} % RKHS
\newcommand{\ErrSpc}{\error_m(\param)} % Spectral error

\newcommand{\Errwhp}{\mathcal{E}_{\trunc,n}(\delta)} % Error whp
\newcommand{\ErrwhpXi}{\mathcal{E}^{'}_{\trunc,n}(\delta)} % Error whp
\newcommand{\Errwhpv}{\mathcal{E}^{''}_{n}(\delta)} % Error whp 

\newcommand{\im}{\pi} %invariant measure
\newcommand{\bim}{\pi'} %invariant biased measure
\newcommand{\ebim}{\widehat{\pi}'} %emp biased invariant measure

\newcommand{\Lii}{\mathcal{L}^2_\im(\X)}

\newcommand{\Liii}{\mathcal L^2_\im}
\newcommand{\Wiix}{\mathcal{W}^{\shift}_{\im}(\X)}
\newcommand{\Wii}{\mathcal{W}^{\shift}_{\im}}
\newcommand{\Wiip}[1]{\mathcal{W}^{\shift,#1}_{\im}} % p-fractional energy space

\newcommand{\bWii}{\mathcal W^{1}_{\bim}(\X)}
\newcommand{\Liishort}{\mathcal{L}^2_\im} % L2 space

\newcommand{\TO}{T} %Transfer operator
\newcommand{\GE}{L} %Infinitesimal generator
\newcommand{\bGE}{L'} %Infinitesimal generator bieased
\newcommand{\EGE}{\widehat{\GE}} %Infinitesimal generator estimator
\newcommand{\bw}{v}  %bias weight

\newcommand{\ebw}{\widehat{\bw}}  %bias weight

\newcommand{\PT}{U} %potential
\newcommand{\bPT}{U'} %biased potential
\newcommand{\bias}{V} %bias potential

\newcommand{\regnn}{\alpha}

\newcommand{\error}{\mathcal{E}}
\newcommand{\shift}{\mu}
\newcommand{\trunc}{m} % truncation choose m or r
\newcommand{\truncbis}{m+1} % truncation choose m or r
\newcommand{\energy}{\mathfrak E}

\newcommand{\energykernel}{\energy^\shift}

\newcommand{\resolventinv}{(\Id \!-\! \shift^{-1}\GE) }
\newcommand{\eresolventshort}{\widehat{R}_\shift}
\newcommand{\resolventshort}{R_\shift} 
\newcommand{\resolventinvshort}{R_\shift}
 
\newcommand{\cv}{z}
\newcommand{\CV}{Z}
\newcommand{\NNek}{\cv^\param}
\newcommand{\NNe}{\cv_\param}
\newcommand{\NNE}{\CV_\param}

\newcommand{\NNop}{ \TZ_\param}
\newcommand{\EVop}{ \MX{D}_\param}
\newcommand{\param}{\theta}
\newcommand{\Param}{\Theta}
\newcommand{\cst}{\tau}
\newcommand{\exploss}{{\mathcal L}}

\newcommand{\emploss}{\widehat{\mathcal L}}

\newcommand{\Wiishort}{\mathcal{W}^{\shift}_\im}

\newcommand{\MX}[1]{\textsc{#1}}

\newcommand{\refun}{\psi} % resolvent eigenfunction
\newcommand{\reval}{\nu} % resolvent eigenvalue
\newcommand{\erefun}{\w{\psi}} % resolvent eigenfunction
\newcommand{\ereval}{\w{\nu}} %empirical resolvent eigenvalue
\newcommand{\gefun}{\refun}
\newcommand{\egefun}{\erefun}
\newcommand{\eval}{\lambda}
\newcommand{\evalp}{\eval^\param}
\newcommand{\geval}{\lambda} % generator eigenvalue
\newcommand{\egeval}{\widehat{\geval}}

\newcommand{\gap}{\ensuremath{\text{\rm gap}}}
\newcommand{\idfun}{f}%identity function gradient

\newcommand{\Cx}{\MX{C}}

\newcommand{\ECx}{\widehat{\MX{C}}}

\newcommand{\Wx}{\MX{W}}
\newcommand{\wx}{w_\param}
\newcommand{\EWx}{\widehat{\MX{W}}}

\newcommand{\Cxt}{\Cx_{\param}} %covariance NN param
\newcommand{\Wxt}{\Wx_{\param}} % injection NN param
\newcommand{\ECxt}{\w{\Cx}_{\param}} % empirical covariance NN param
\newcommand{\EWxt}{\w{\Wx}_{\param}} % empirical injection NN param
\newcommand{\Nt}{{N}_{\param}}
\newcommand{\ENt}{\w{N}_{\param}}

\newcommand{\Qt}{Q_{\param}}
\newcommand{\EQt}{\w{Q}_{\param}}

\newcommand{\massi}[2][]{\todo[color=brown!20,#1]{{\bf MP:} #2}}

\newcommand{\MP}[1]{\textcolor{brown}{#1}}

\title{Langevin-Informed Transfer Learning:\\ Replacing 
Target Samples by Black-Box Feedback}

\begin{document}

\maketitle

\begin{abstract}
Many scientific and machine learning systems, from molecular dynamics to diffusion models and beyond, are governed by stochastic dynamics with low-dimensional structure, evolving on slow timescales. However, target trajectories, used to identify and interpret such dynamics, are often inaccessible: only biased or static samples that explore the underlying manifold are available. 
%dynamics' manifold are available.
We introduce Langevin-Informed Transfer Learning (LITL), a framework for recovering target Langevin dynamics from biased source samples using only black-box feedback. \massi{bit unclear what black-box feedback} LITL learns the leading spectral structure of the target infinitesimal generator and the projected drift through Dirichlet representation learning, enabling kinetic reconstruction in spectral form and slow-manifold gradient field estimation. 
We further introduce a spherical variant well suited to steering normalized latent representations commonly used in learning systems  toward desired objectives.  We establish finite-sample guarantees for eigenvalue, eigenfunction, and projected drift estimation in Sobolev norms, thereby ensuring generalization of these quantities and their first-order derivatives.
%thus proving generalization of all objects in function and first order derivatives values. 
Empirically, LITL recovers physical transition timescales from biased molecular simulations, builds kinetic structure from static samples of generative models, reconstructs spherical symmetries of physical systems, and enables post-hoc latent steering of trained neural networks under black-box feedback. Together, these results position spectral operator learning as a practical framework for recovering stochastic dynamics under distribution shift and unlock applications across machine learning and the physical sciences.
\end{abstract}

\section{Introduction}
%\helene[]{below, a proposition adapted from Vladi’s Slack notes}
Many systems across science and modern machine learning evolve through stochastic dynamics shaped by a small number of slow collective modes. In molecular dynamics, these modes govern metastable states and transition pathways \cite{Schutte2023, Bonati2021}; in modern machine learning, they relate to semantic structures in generative \cite{zhu2023boundary} and transformer models \cite{geshkovski2025mathematical, kim2025metastable}. In both contexts, Langevin-type dynamics underlie these stochastic gradient flows, further being backbone of Bayesian sampling, diffusion models, and latent generative processes \cite{chaudhari2018stochastic,ma2015complete,sohl2015deep,song2021score,welling2011bayesian}.  Across these domains, the key object controlling long-timescale behavior is the \textbf{infinitesimal generator}, whose dominant spectral structure encodes geometry, kinetics, and transition pathways between metastable configurations of the underlying dynamics. \massi{add 1-2 reference} 
Understanding such systems requires recovering this spectral structure, a task intractable with classical methods
%a task well beyond classical numerical methods 
due to the curse of dimensionality that calls for data-driven approaches.    
%\helene[]{For this part, I’m reusing one of the intro paragraphs in red.}
%Understanding such systems requires recovering the spectral structure of the underlying infinitesimal generator, whose dominant eigenfunctions encode long-timescale organization, coarse-grained dynamics, and transition mechanisms. 
In practice, however, the target dynamics are often inaccessible: one may only observe \textit{static or biased samples together with black-box evaluative feedback}, while target trajectories, gradients, or direct dynamical observations remain unavailable. 
\noindent
So, a fundamental challenge persists:
\begin{center}
%\vspace{-.1truecm}
\textit{How can one recover a target system’s dynamics without access to its trajectories?}
%\vspace{-.1truecm}
\end{center} 
This setting arises naturally in biased molecular simulations \cite{laio_metad}, generative models producing equilibrium samples without kinetics \cite{bioemu2025}, physical systems where potentials are reconstructed from measurements \cite{jamet2001magnetic}, and post-hoc latent steering problems with purely evaluative black-box feedback. Existing transfer-learning and adaptation methods predominantly rely on retraining, fine-tuning, or direct optimization of parametric objectives \cite{pan2010survey,long2015learning, ganin2016domain}. In contrast, recent spectral approaches in dynamical systems focus on learning transfer operators or infinitesimal generators from trajectory data \cite{LGBPP12,KLUS2020132416,colbrook2023rigorous,Bruder2021, kostic2024laplace}, but do not address transfer under distribution shift or recovery of target dynamics from static biased samples. Moreover, existing generator-learning methods at best provide guarantees only in $\mathcal L^2$-type geometries and all fail to recover projected drift fields \cite{KLUS2020datadrivenkoopmangenerator, Cabannes_2023_Galerkin, ZHANG2022diffusionNN, devergne2024biased, kostic2024learning}, limiting controlled reconstruction of target dynamics and kinetic transfer.

\noindent
We introduce \textbf{Langevin-Informed Transfer Learning (LITL)}, a framework for recovering and transferring target Langevin dynamics from biased source samples using only black-box feedback. Rather than addressing transfer learning in full generality, 
%\helene[]{to distinguish LITL from general transfer learning}\massi{very nice!}
LITL focuses specifically on transferring stochastic dynamics by recovering the target generator’s slow spectral structure together with a projected drift on a learned slow manifold. The method combines importance weighting, Sobolev/Dirichlet representation learning, and generalized eigendecomposition to reconstruct physically meaningful dynamics from static or biased data without requiring target trajectories, differentiable objectives, or retraining. In this sense, LITL reframes transfer learning as transfer of stochastic dynamics: by learning the generator geometry rather than task-specific parameters, it enables post-hoc steering through geometry-aware Langevin flows in latent spaces. This perspective is particularly compelling for modern machine learning systems operating on normalized embeddings and diffusion-like latent dynamics, where semantic organization is intrinsically geometric and metastable. In this setting, we further introduce a streamlined spherical formulation enabling efficient gradient-free transfer on normalized latent manifolds. To relate LITL to existing transfer learning methods, see App. \ref{app:related_work}.

\noindent
On the theoretical side, we establish finite-sample guarantees for eigenvalue, eigenfunction, and projected-drift estimation in Sobolev norms. Empirically, LITL recovers physical transition timescales from biased molecular simulations, extracts kinetic structure from static GenAI samples, reconstructs spherical dynamics of cobalt nanoparticles from black-box potential evaluations, and enables post-hoc fairness-aware latent steering without retraining. Together, these results position spectral operator learning as a practical framework for recovering and transferring stochastic dynamics under distribution shift, while highlighting broader implications for machine learning systems governed by Langevin-type dynamics.

\textbf{Contributions.} This paper makes three main contributions. %spanning methodology, theory, and applications:
\textit{(i) Methodology:} We introduce LITL, a framework for recovering and transferring stochastic dynamics from biased samples using black-box feedback.
\textit{(ii) Theory:} To our knowledge, we provide the first finite-sample guarantees for eigenvalue, eigenfunction, and projected-drift estimation 
%with neural networks 
in Sobolev norms.
\textit{(iii) Applications:} We demonstrate effectiveness across molecular modeling, physical systems, and post-hoc alignment tasks.

\textbf{Paper organization.}~ Sec.~\ref{sec:problem} provides background on Langevin dynamics and generator learning. Sec.~\ref{sec:method} presents our method. Sec.~\ref{sec:bounds} establishes finite-sample bounds, and Sec.~\ref{sec:experiments} reports  experiments. The appendix covers related work, additional background, and proofs of the main results.
%of the main results presented in the main body of the paper. %Each section is expanded in the supplementary material.% on molecular dynamics and a fairness application.

\section{Background and problem setup}
\label{sec:problem}
%\massi{If space permits add a short opening sentence}

\textbf{Langevin dynamics and spectral learning.} 
{Many systems exhibit multiple timescales, with long-term behavior governed by a few slow modes encoded in the generator’s spectrum. We recover this structure from samples via the generator and its resolvent.
A canonical model for such dynamics is the} overdamped Langevin equation 
describes a system evolving in a state space $\mathcal X\subseteq\mathbb R^d$ under a potential $U:\mathbb R^d\to\mathbb R$ at {inverse} temperature $\beta>0$:
\begin{equation}
\label{eq:langevin}
dX_t=-\nabla \PT(X_t),dt+\sqrt{2\beta^{-1}}\,dW_t,\qquad X_0=x,
\end{equation}
where $(W_t)_{t}$ is a standard $d$-dimensional Brownian motion. It admits a unique strong solution defining a Markov process $X\,=\,(X_t)_{t\ge0}$ with transfer operators $(\TO_t)_{t\ge0}$ given by 
%satisfying 
$\TO_t f(x)\,=\,\mathbb{E}[f(X_t)| X_0=x]$. The law of $X_t$ converges to the Boltzmann measure
%distribution 
$\pi(dx)=e^{-\beta U(x)}dx/(\int e^{-\beta U(x)}dx)$, and the evolution is governed by the \textit{infinitesimal generator} $\GE:=\lim_{t\to0^+}(\TO_t-I)/t$, which acts on smooth functions as
$
\GE f=-\langle\nabla U,\nabla f\rangle+\beta^{-1}\Delta f,
$
is linked to the transfer operators through $\TO_t=e^{t\GE}$, and is self-adjoint on $\Liishort$ by integration by parts. If the spectrum is discrete, the generator has the decomposition
\begin{equation}
\label{eq:evd}
\GE=\textstyle{\sum_{i\in{\mathbb N}_0}}\lambda_i\,\gefun_i\otimes\gefun_i,
\end{equation}
%\helene[]{as suggested by Massi, I add a physical interpretation of the spectrum of IG}
where the eigenvalues $(\lambda_i)_{i\in\mathbb N_0}$ determine the characteristic timescales of the system. In particular, eigenvalues close to $0$ correspond to slow, metastable dynamics, while large negative eigenvalues describe rapidly decaying modes, making the spectral decomposition a natural tool for identifying coherent structures and long-time behavior in the dynamics. While learning $\GE$ directly is challenging due to its unbounded nature,  its resolvent $\resolventshort =(I -\shift^{-1}\GE)^{-1}$, $\mu > 0$, is compact for confining potentials and shares the same eigenfunctions, making it more amenable for spectral learning, see App.~\ref{app:background}.

\textbf{Transfer learning via biased dynamics.}~
Sampling from complex distributions is typically impacted by slow mixing due to metastability. To accelerate exploration, we can simulate dynamics under a biased potential $\bPT = \PT + \bias$, where $\bias: \mathbb{R}^d \to \mathbb{R}$ flattens energy barriers. Let $\bGE$ and $\bim$ denote the generator and invariant measure of the biased (source) dynamics. Our objective is to recover the spectral decomposition of the (unbiased) \textit{target} generator $\GE$ from trajectories generated under $\bGE$. At the generator level, the bias enters as a linear additive perturbation, $\bGE = \GE - \langle \nabla \bias, \nabla(\cdot)\rangle$, which is in principle straightforward to debias. This linear structure is lost at the level of transfer operators, since the associated semigroups do not factorize ($e^{\bGE} \neq e^{\GE} e^{-\langle \nabla \bias, \nabla(\cdot)\rangle}$), limiting the applicability of transfer-operator-based methods. A widely used debiasing strategy is importance weighting, which relies on the Radon--Nikodym derivative $\frac{d\im}{d\bim}(x) = \frac{e^{\beta \bias(x)}}{\int e^{\beta \bias(x)} d\bim(x)}$, see App.~\ref{app:biased_dynamics},
% \begin{equation}\label{eq:radon_nikodym}
% \frac{d\im}{d\bim}(x) = \frac{e^{\beta \bias(x)}}{\int e^{\beta \bias(x)} d\bim(x)},
% \end{equation}
and corrects for the distribution shift induced by the bias. 

\textbf{Sobolev spaces and energy norms.} {To learn the generator or its resolvent from data, we need a function space where they are \textit{well-behaved} and \textit{estimable from samples}. Since the generator acts through gradients, this leads to Sobolev (energy) norms combining values and derivatives, which both regularize the problem and admit gradient-based empirical estimation under importance weighting.}
Our approach relies on energy (equivalently graph) norms induced by the generator: $\|f\|_{\Wii}^2 = \langle{\resolventinv}f,f\rangle_{\Liishort} = \norm{f}_{\Liishort}^2+\norm{\nabla f}_{\Liishort}^2/(\shift\beta)$, which defines the Sobolev space
\begin{equation}\label{eq:sobolev_space}
\Wii =\bigl\{f\in\Lii \,\big|\, \norm{f}_{\Wiishort}^2=\norm{f}_{\Liishort}^2+{\norm{\nabla f}_{\Liishort}^2}/{(\shift\beta)}<\infty\bigr\}.
\end{equation}
More generally, for $p>0$, we define the fractional Sobolev norm $\|f\|_{\Wiip{p}}^2 := \langle{\resolventinv}^{p}f,f\rangle_{\Liishort}$ and space $\Wiip{p}=\{f\in\Lii\mid\|f\|_{\Wiip{p}}<\infty\}$. These norms form a Sobolev scale: they interpolate between the $\Lii$ norm (at $p=0$) and the energy norm (at $p=1$). Crucially, these norms can be expressed as expectations involving function gradients, enabling empirical estimation from biased trajectory data through importance weighting.

\section{LITL methodology}
\label{sec:method}

%\helene[]{I’m rewriting the beginning of the section here (up to the theorem 3.1). The structure has changed and follows the 2 steps of LITL (the previous version is shown below in black) to make its presentation more pedagogical. This is a response to the criticism from at least two reviewers regarding the “steep learning curve” of the presentation.}
%\color{blue}

\noindent
We now turn the operator perspective into a practical procedure for learning spectral structure of the target dynamics from biased data. 
%   In short, LITL learns a Sobolev-orthonormal basis of slow target modes from biased samples, and then estimates the target generator and projected force field in it.

\noindent
\textbf{LITL in a nutshell.}
%\massi{This section is nice though i am not sure it is good to refer to equations that have not yet appeared}
Given samples from the biased distribution $\bim$, LITL proceeds in two stages.
(1) \textit{Representation learning.} We learn neural features $\NNe =(\NNek_i)_{i=0}^{m}$ to approximate the leading invariant subspace of the resolvent $\resolventshort =(I-\shift^{-1}\GE)^{-1}$ with their span $\RKHS_\param$ by minimizing the loss \eqref{eq:population_loss}. 
(2) \textit{Projected spectral estimation and dynamics recovery.} We project the operator $\resolventinvshort$ onto $\RKHS_\param$ and compute its empirical spectral decomposition via \eqref{eq:empirical_gep}. Finally, using $\nabla\PT = \GE f$ for $f(x)=x$, we recover the projected dynamics by regressing coordinate functions in the learned spectral basis.

\noindent
\textbf{Step 1. Representation learning.}
We first construct a low-dimensional representation that captures the slow modes of the target dynamics. To this end, we consider the resolvent $\resolventshort$, whose leading eigenfunctions define the invariant subspace of interest, and correct for the distribution shift between the source $\bim$ and target $\im$ using importance weights $\bw=e^{\beta\bias}$ with mean $\overline{\bw}=\EE_{x'\sim\bim}\bw(x')$. We introduce neural networks $\NNek_i:\X\to\R$, $0\leq i\leq m$, defining the representation space $\RKHS_\param=\mathrm{span}(\NNek_i)_{i=0}^{m}\subseteq \Wii$, and aim to approximate a spectral decomposition of $\resolventshort$ in this space,
\[
\resolventshort\approx\NNE \EVop\NNE^*,\quad\text{with the orthonormality constraint}\quad \NNE^*\NNE \approx \MX{I}, 
\]
where $\NNE\colon \R^{m+1}\to\Wii$ is the injection operator mapping $v\in\R^{m+1}$ to function $ v^\top\NNe(\cdot)$,  
\[
\EVop=\mathrm{diag}\big(e^{-(d^\param_0)^2},\dots,e^{-(d^\param_m)^2}\big),\quad\text{and}\quad \NNE^*f=\Big(\langle\resolventshort^{-1}f,\NNek_j\rangle_{\Lii}\Big)_{0\leq j\leq m}.
\]
% for diagonal $\EVop=\mathrm{diag}\big(e^{-(d^\param_0)^2},\dots,e^{-(d^\param_m)^2}\big)$ and
% \begin{equation*}
% \NNE^*f=\Big(\langle\resolventshort^{-1}f,\NNek_j\rangle_{\Lii}\Big)_{0\leq j\leq m},
% \end{equation*}
% with the orthonormality constraint $\NNE^*\NNE = \MX{I}$.

Since $\resolventshort$ is a compact self-adjoint operator, this leads to the population loss
\begin{equation}
\exploss_{\regnn}(\param)
= \overline{\bw}\Big(\|\resolventshort{-}\NNE \EVop\NNE^*\|_{\mathrm{HS}(\Wii)}^2
-\|\resolventshort\|_{\mathrm{HS}(\Wii)}^2\Big)+\regnn\,\overline{\bw}\,\|\NNE^*\NNE-\MX{I}\|_{F}^2, \label{eq:population_loss_hs}
\end{equation}
where $\overline{\bw}$ accounts for the importance weighting induced by the bias potential. \massi{Should we add some more intuition/explanation for the factor $\bar{v}$ and how the loss is derived? Also double heck formula (4).}
All operator terms in this loss can be expressed as expectations under $\bim$. Indeed, for $1 \leq i,j\leq m$,
\[
\scalarp{\NNek_i,({-}\GE)\NNek_j}_{\Liishort}
=\EE_{x'\sim\bim} [\bw(x')\,\nabla \NNek_i(x')^\top \nabla \NNek_j(x')],
\]
implying that $\NNE^*\NNE$ corresponds to the covariance matrix in $\Wii$ space
\begin{equation}\label{eq:Wx}
\Wxt = \overline{\bw}\NNE^*\NNE=\EE_{x'\sim\bim}[ \wx(x')\wx(x')^\top], \;\text{with}\;
% \end{equation}
% where
% \begin{equation}\label{eq:w_features_main}
\wx(x')=\Big[\sqrt{\bw(x')}\NNe(x')\,\Big\vert\,\sqrt{\tfrac{\bw(x')} \shift\beta} J_\param(x')\Big],
\end{equation}
and {$J_\param(x')_=\big[\partial_1\NNe(x')\,\big\vert\cdots\big\vert\,\partial_d\NNe(x')\big]\in\R^{{(m+1)}\times d}$ is the Jacobian of the neural network embedding.}
%Similarly,
On the other hand, the embedding of the operator $\resolventshort$ in the $\Wii$ reduces to the $\Liishort$-covariance matrix
\begin{equation}\label{eq:Cx}
\Cxt=\overline{\bw}\,\NNE^*\resolventshort \NNE=\EE_{x'{\sim}\bim}[ \bw(x')\NNe(x')\NNe(x')^\top].
\end{equation}
This yields the equivalent form of the loss
\begin{equation}\label{eq:population_loss}
\exploss_\regnn(\param)
=\mathrm{tr}\big[(\EVop\Wxt)^2
-2\,\overline{\bw}\EVop\,\Cxt{+}\regnn (\Wxt {-}\overline{\bw}\,\MX{I})^2\big].
\end{equation}

%\textit{Representation learning loss.}
To estimate \eqref{eq:population_loss} from data, we construct an unbiased empirical loss using source samples. The easiest way is to take two independent samples to estimate the products of random variables $(\EVop\Wxt)^2$, $\overline{\bw}(\EVop\Cxt)$, $\Wxt^2$ and $\overline{\bw}\Wxt$. This amounts to splitting each batch in two and estimating the loss %in a balanced way 
as
\begin{align}
\emploss_\regnn(\param)=\tr\!\! \Big[& \tfrac{1}{2}(\EWx_\param^1\EVop \EWx^2_\param\EVop {+} \EWx_\param^2\EVop \EWx^1_\param\EVop) 
{-}\widehat{\bw}^1\ECx_\param^2\EVop{-}\widehat{\bw}^2\ECx_\param^1\EVop { +} \regnn(\EWx_\param^1-\widehat{\bw}^1\MX{I})(\EWx_\param^2{-}\widehat{\bw}^2\MX{I}) \Big],\label{eq:emploss_2} 
\end{align}
where~$\ECx_\param^k$~and~$\EWx_\param^k$~are the empirical covariances in~\eqref{eq:Wx} and~\eqref{eq:Cx}, respectively, for empirical distributions 
$\ebim_1=\ceil{ b/2}^{-1}\sum_{i=1}^{\ceil{ b/2}}\delta_{x_i'}$ and $\ebim_2=\floor{ b/2}^{-1}\sum_{\ceil{ b/2}+1}^b\delta_{x_{b+1}'}$, respectively, while $\widehat{\bw}^{k}=\EE_{x'\sim\ebim_k}\bw(x')$, $k{\in}[2]$. Notably, the computational complexity of the loss \eqref{eq:emploss_2} is of the order $\mathcal{O}(b\,m^2d)$, where $d$ is the state dimension and $b$ the sample batch size. Alternatively to \eqref{eq:emploss_2}, the products of expectations in the population loss \eqref{eq:population_loss} can be estimated from iid samples $(x'_i)_{i\in[b]}$ via U-statistics yielding computational complexity $\mathcal{O}(b^2\!md)$, see App. \ref{app:empirical_loss} for details.

The following result
provides theoretical guarantees for our approach. The proof is provided in~\ref{app:biased_learning}.

\begin{restatable}{theorem}{ThmMainDNN}\label{thm:main_dnn}
Given a compact operator $\resolventshort$, $\shift\,{>}\,0$, if for all $\param\,{\in}\,\Param$, $(\NNek_i)_{i=0}^{m}\,{\subseteq}\, \bWii$,   then 
%for all $\param\in\Param$
\begin{equation}\label{eq:dnn_bound}
\EE\big[\emploss_\regnn(\param)\big] = \overline{\bw}^2\,\exploss_\regnn(\param)\geq -{\textstyle{\sum_{i=0}^{m}}\tfrac{\shift^2\overline{\bw}^2}{(\shift-\geval_i)^2}},
\end{equation}
where, if $\eval_{m+1}{<}\eval_{m}$ equality holds iff $\mathrm{span}(\gefun_i)_{i=0}^{m}=\mathrm{span}(\NNek_i)_{i=0}^{m}$, or equivalently $\TP_m =\TP_\param$, with $\TP_m$ and $\TP_\param$ being the orthogonal projectors onto $\mathrm{span}(\gefun_i)_{i=0}^{m}$ and $\mathrm{span}(\NNek_i)_{i=0}^{m}$, respectively. Additionally, if $\regnn\,>\,0$ at optimum $\TP_m =\TP_\param = \NNE\NNE^*$ and $\evalp_i = d^\param_i$, $0\leq i \leq m$, 
up to the ordering of indices.
\end{restatable}
This theorem provides a justification for minimizing the loss in \eqref{eq:population_loss} with stochastic optimization algorithms to obtain an approximation of the leading invariant subspace of the resolvent $\resolventshort$, with a natural metric to quantify the distance to the optimum being
\begin{equation}\label{eq:rep_learnig_error}
\error_m(\param)= \| \TP_m - \TP_\param\|_{\Wiishort{\to}\Wiishort}
%\|\TP_m\resolventshort {-} \NNE\EVop\NNE^*\|_{\Wiishort{\to}\Wiishort} 
+ \|\NNE^*\NNE-I\|.    
\end{equation}
The main advantage of this representation learning method is that it exploits knowledge of the underlying process. Knowing that the dynamics are Langevin, and given only the change between the source and target potentials (bias) $\bias$ and the diffusion coefficient $\beta$, we can learn representations that approximate the slow manifold of the target process from source samples. \\
We next show how, once the slow manifold is learned, to compute estimators of the generator’s spectral decomposition and the target potential’s gradient.

\noindent
\textbf{Step 2. Projected spectral estimation and dynamics recovery.}
Once the representation $\RKHS_\param$ approximates the leading invariant subspace, we estimate the spectral structure of the target dynamics within this space. Using the fact that the generator and resolvent share eigenfunctions, with eigenvalues related by $\geval_i=\shift(1-1/\reval_i)$, we consider the decomposition
\begin{equation}\label{eq:resolvent_dec_true}
\resolventshort=\textstyle{\sum_{j\in{\mathbb N}_0}}\,\reval_j\,\gefun_j\otimes\gefun_j.
\end{equation}
%We then project $\resolventinvshort$ onto $\RKHS_\param$ and estimate its empirical spectral decomposition from samples $(x_j')_{j\in[n]}$ by solving the generalized eigenvalue problem \eqref{eq:empirical_gep}. This yields estimators of the leading eigenvalues and eigenfunctions of $\GE$ restricted to the learned subspace.
Since after learning  $\param$ the bias of representation space $\error_m(\param)$ becomes sufficiently small, \massi{elaborate a bit on why we expect this to be the case?}we have that $\TP_m\resolventinvshort= \TP_m\resolventinvshort\TP_m{\approx}\TP_\param\resolventinvshort\TP_\param$, the task is to empirically estimate the operator $\TP_\param\resolventinvshort\TP_\param$ from some samples $(x_j')_{j=1}^n$. We can either use the same samples for training representations, or  a set of new independent samples, in which case the statistical analysis simplifies.  

Recalling that the orthogonal projector in $\Wii$ onto $\RKHS_\param$ is given by $\TPt=\NNE(\NNE^*\NNE)^\dagger\NNE^*$, we have that $\TP_\param\resolventinvshort\TP_\param=\overline{\bw}\,\NNE\Wxt^\dagger\Cxt\Wxt^\dagger\NNE^*$. But then, since due to the ortho-normality regularization of the loss ($\regnn>0$) we can assume that $\Wxt$ are invertible,  using the property of low-rank eigenvalue problems, see \cite{horn2012matrix},%\karim{missing reference here}\helene{is it Higham-Lin? Maybe classical reference is enough \cite{horn2012matrix}?}
we have that the non-zero eigenvalues $\reval^\param_i$ and their eigenfunctions $\gefun^\param_i$, $0\leq i \leq m$, of $\TP_\param\resolventshort\TP_\param$ are characterized by $\Cxt u_i = \reval^\param_i \Wxt u_i$, where $\gefun^\param_i =\NNe(\cdot)^\top u_i$, $0 \leq i \leq m$. Therefore, the slow eigenpairs of the target generator $\GE$ are estimated by solving generalized eigenvalue problem
\begin{equation}\label{eq:empirical_gep}
\ECxt \widehat{u}_i = \ereval_i \EWxt \widehat{u}_i,\quad\text{ s.t. }\quad \widehat{u}_i^\top\EWxt \widehat{u}_i =\ebw,\;0 \leq i \leq m,
\end{equation}
where $\ebw = \sum_{i=1}^n\bw(x_i')$ and defining  $\egeval_i=\shift(1-1/\ereval_i)$ and $\egefun_i =\NNe(\cdot)^\top \widehat{u}_i$ for $0 \leq i \leq m$.

As a result, we obtain the empirical operators that approximate $\TP_m\resolventshort$ and $\TP_m\GE$, respectively, as 
\begin{equation}\label{eq:empirical_operators}
\eresolventshort=\textstyle{\sum_{i=0}^m}\ereval_i\,\erefun_i\otimes\erefun_i,\quad\text{ and }\quad\EGE=\textstyle{\sum_{i=0}^m}\egeval_i\,\erefun_i\otimes\erefun_i.
\end{equation}

Finally, to recover the projected dynamics, we  empirically estimate the gradient using that $\partial_k \PT=-\GE \idfun_k$ holds for $k$-th coordinate projection of the state $\idfun_k\colon x\mapsto x_k$, $1 \leq k \leq d$. Namely, by definition of the $\Wiishort$-norm, we have that 
% \begin{equation*}
% \nabla \PT=-\textstyle{\sum_{i\ge0}}\,\geval_i\gefun_i
% \bigl[\EE_{x\sim\im}[\gefun_i(x)x_k{+}\tfrac{1}{\shift\beta}\partial_k\gefun_i(x)] \bigr]_{k\in[d]},
% \end{equation*}
\begin{equation*}
\partial_k \PT=-\textstyle{\sum_{i=0}^\infty}\,\geval_i\gefun_i
\EE_{x\sim\im}[\gefun_i(x)x_k +\tfrac{1}{\shift\beta}\partial_k\gefun_i(x)],\;1 \leq k \leq d,
\end{equation*}
which, by replacing $\GE$ by $\EGE$ and empirically estimating the  expectation, after some algebra, yields that the weight matrix $\w{G}$ of the  target gradient's projection onto the slow manifold can be computed as
\begin{equation}\label{eq:emp_gradient}
\widehat{\MX{G}}=-\textstyle{\sum_{j=1}^n} \tfrac{\bw(x'_j)}{\ebw}\big[x_j \NNe(x'_j)^\top +\tfrac{1}{\shift\beta}J_\param(x'_j)^\top\big]\,\,
%}{\sum_{j\in[n]}\bw(x'_j)}
\w{\Omega}\w{\Lambda}\w{\Omega}^\top\!,
\end{equation}
where, $%\w{\MX{U}}
\w{\Omega}= \left[\widehat{u}_0\mid\cdots  \mid \widehat{u}_{\trunc}\right]$, $\w{\Lambda} = \mathrm{diag}\left(\egeval_0,\cdots  ,\egeval_{\trunc}\right)$, and, as above,  $J_\param$ denotes the Jacobian of $\NNe$.

This provides an estimate of the gradient of the target potential on the slow manifold $\mathrm{span}(\gefun_i)_{i=0}^{m}$.

\textbf{Inference.} Before presenting high-probability error bounds for the resulting spectral decomposition and gradinet estimation in Sec.  \ref{sec:bounds}, we note that the target Langevin dynamics can be estimated by exponentiating $\EGE$ to forecast the evolution of conditional probabilities
$$
\PP[X_t {\in} B \vert X_0=x]=\textstyle{\sum_{i=0}^m}\, e^{\egeval_i t}\,\widehat{u}_i^\top\NNe(x)\;\sum_{x_j'\in B}\tfrac{\bw(x_j')}{\ebw}\,\widehat{u}_i^\top\NNe(x_j')
%}{\sum_{j\in [n]}\bw(x_j')},
$$
for any $t \ge 0$ and measurable set $B \subseteq \mathcal{X}$, under the same temperature $\beta$ as the source data. This approximation of the transfer operator semi-group $e^{t\EGE}$ is detailed in App. \ref{app:background}. Moreover, once the target gradient is estimated as $\nabla \PT \approx \widehat{\nabla \PT} = \widehat{\MX{G}}\NNe(\cdot)$ with $\widehat{\MX{G}} \in \mathbb{R}^{d \times (\truncbis)}$, it enables the simulation of new Langevin dynamics  $d\w{X}_t = -\widehat{\MX{G}}\NNe(\w{X}_t)dt + \sqrt{2\beta^{-1}}dW_t$. 
\color{black}

\textbf{Spherical LITL.} \massi{this subsection is maybe a bit fast: "mode collapse" is introduced here for the first time and a bit unclear. The last bits "This geometric formulation..." }Unlike classical Langevin dynamics in Euclidean space, which requires a confining potential to admit an invariant distribution, Langevin dynamics on a sphere admits a well-defined invariant measure even in the absence of a potential.  This provides a principled, geometry-respecting way to steer the system while avoiding the instability and mode collapse that can arise in unconstrained Euclidean settings. In particular, if the source dynamics is pure diffusion on the sphere ($\bPT\,=\,\,0$) the uniform measure is its invariant distribution. This is crucial in our setting since \textit{static samples from the reference distribution are readily available by direct normalization of isotropic random vectors}, avoiding long trajectories simulation or stochastic integrators tuning. 
In this context, transfer learning from the black-box potential is performed by short-horizon gradient updates on the sphere, 
$\nabla_{\mathbb{S}^{d-1}} \PT(x)=(\MX{I}{-}xx^\top)\nabla\PT(x)$,
thereby driving the system toward the wells of potential energy. 

This approach is motivated by the prevalence of normalized embeddings and cosine-based geometry in modern AI systems, models latent dynamics as spherical diffusion with a locally constructed, human-induced potential, enabling targeted, on-the-fly modulation of the invariant distribution without retraining the base model. This geometric formulation aligns naturally with contrastive representation learning, angular-margin classifiers, and embedding-based retrieval, where semantic information is encoded directionally rather than in magnitude. Thus, Spherical LITL provides a principled and sample-efficient mechanism for human control of learned representations, enhancing adaptability and robustness, while preserving the global structure of pretrained representations.

\section{Statistical bounds}\label{sec:bounds}

We recall that spectral estimation enables analysis and understanding of the dynamics, while gradient estimation enables efficient simulation of the learned dynamics. We establish statistical guarantees for both spectral estimation and gradient estimation of target dynamics based on biased samples. To this end, we make 
the following assumptions:
\vspace{-.015truecm}
\begin{enumerate}[label={\rm \textbf{(BE)}},leftmargin=7ex,nolistsep]
\item\label{eq:BE} \emph{Bounded embeddings.} There exist $\cst \geq 1$ such that $\esssup_{x'\sim\bim} |\partial_k \NNek_{j}(x')|\leq \cst<\infty$, $\forall\, k=0,\ldots,d$ and $j= 0, \ldots \trunc$, with the convention $\partial_0 \NNe = \NNe$;
\end{enumerate}
\vspace{-.195truecm}
\begin{enumerate}[label={\rm \textbf{(CP)}},leftmargin=7ex,nolistsep]
\item\label{eq:CP} \emph{Confining potential.} The potential $\PT$ satisfies $\PT(x)\ge 2\alpha\|x\|^{2}\mathbf{1}_{\{\|x\|\ge R\}}-c$, $\forall x$, for some $\alpha>0$, $R>0$, $c\in\R$, and the bias obeys $\|\bias\|_\infty<\infty$;
\end{enumerate}
\vspace{-.195truecm}
\begin{enumerate}[label={\rm \textbf{(RA)}},leftmargin=7ex,nolistsep]
\item\label{eq:RA} \emph{Gradient regularity.} $\exists\,p\geq 1$ such that $\forall k\in[d]$, $\norm{\partial_k \PT}_{\Wiip{p}}<\infty$, that is $\norm{\nabla \PT}_{\Wiip{p}}<\infty$.
\end{enumerate}

While these typical assumptions are discussed in detail in App.~\ref{app:bounds_ass}, here we  briefly note that they are all satisfied in the prototypical example of spherical Langevin dynamics. %Furthermore, crucially, \ref{eq:CP} implies that  $x\sim\im$ and $x'\sim\bim$ are $K$-sub-Gaussian with $K^2\lesssim(\beta \alpha)^{-1}$, allowing one to   derive concentration bounds for the gradient estimation. 

Let $\varepsilon_{n}(\delta){=} \cst e^{2\beta\norm{\bias}_\infty}\!\sqrt{\frac{m\,d}{\shift\,\beta\,n}\ln( \frac{m}{\delta})}$, $\sigma_{\im}^2 {=} \EE_{x \sim \im}\norm{x}^2$ and $\gap_{i}{=}\min_{\ell\neq i}\abs{\reval_\ell{-}\reval_i}{=}\tfrac{\shift}{(\shift {-} \geval_i)} \displaystyle{\min_{\ell \neq i}} \tfrac{|\geval_i {-} \geval_\ell|}{|\shift {-} \geval_\ell|}$ be the resolvent's spectral gap. The following key statistical result is proved in App.~\ref{app:proof_main_stat_1}.

\begin{theorem}
\label{thm:main_stat_1}
Let Assumptions \ref{eq:BE} and \ref{eq:CP} be satisfied. Assume that $\beta(\shift-\lambda_m) \geq 1$ and $n \gtrsim d\vee e^{3\beta\norm{\bias}_\infty/2}\cst^2 md/(\shift\beta)$.  Then w.p.a.l. $1-\delta$, for any $1 \leq i \leq m$:
\begin{align*}\label{eq:eval_rel_error_main_eig}
&\frac{|\geval_i{-}\egeval_i|}{|\geval_i|} \lesssim \Big(\sqrt{\tfrac{\shift}{|\geval_i|}} + \sqrt{\tfrac{|\geval_i|}{\shift}}\Big)^2\Big(\ErrSpc{+}\varepsilon_{n}(\delta)\Big)\;\text{ and }\;
%\\ 
%&
\norm{\erefun_i {-} \refun_{i}}_{\Wii} \lesssim \frac{\ErrSpc}{\gap_i}+ \frac{\varepsilon_{n}(\delta)}{[\gap_i {-} 3\ErrSpc]_+} \notag
\end{align*}

Moreover, if \ref{eq:RA} holds, then $\forall s{\in} [0,1]$  w.p.a.l. $1-\delta$:
\begin{equation*}\label{eq:eval_rel_error_main_force}
\norm{ \nabla \PT - \widehat{\nabla\PT}}_{\Wiip{s}} \lesssim  \reval_{\truncbis}^{\frac{p-s}{2}}
\norm{\nabla\PT}_{\Wiip{p}} +\shift\,\sigma_{\im}\,{ \reval_{\trunc}^{-\frac{3+s}{2}}} {\big(\ErrSpc+\varepsilon_{n}(\delta)\big)}.
\end{equation*}
\end{theorem}
\noindent
%\textbf{Discussion.}~
The above eigenvalue bound
demonstrates that the shift parameter $\shift$ acts as a \textbf{spectral focusing mechanism}. By tuning $\shift$ close to a target eigenvalue $|\geval_i|$, the relative error is minimized since the factor $(\sqrt{\shift/|\geval_i|} {+} \sqrt{|\geval_i|/\shift})^2$ achieves its minimum value of $4$ when $\shift \approx |\geval_i|$. This allows targeted estimation of specific spectral components by appropriately selecting $\shift$. Furthermore, the eigenfunction bound depend on whether the  representation learning error is dominated by the \textbf{spectral gap}, an expected necessary condition.  The gradient bound reveals two competing sources of error: the bias term $\reval_{\trunc}^{(p-s)/2}\norm{\nabla\PT}_{\Wiip{p}}$ from spectral truncation, and the statistical term \massi{there is also $\mathcal{E}_m(\theta)$} $\sigma_{\im}\overline{\varepsilon}_{n}(\delta)$ from finite-sample estimation. Our approach hinges on accurately learning the dominant eigenspace of the generator through representation learning: when the data exhibits low-dimensional structure, reflected in a spectral gap $\reval_{\trunc} {\gg} \reval_{\trunc+1}$, the leading $\trunc$ eigenfunctions capture the essential dynamics. This spectral truncation to the dominant $\trunc$-dimensional eigenspace is mathematically \textbf{analogous to coarse-graining} procedures in statistical mechanics \cite{peter2009multiscale}, where high-frequency modes are integrated out to yield effective dynamics on slow collective coordinates. \massi{I found this unclear} In this context, the only error to the approximated projected gradient is the second term. However, even if the task is to reconstruct the complete gradient, whenever $\reval_{\truncbis}^{2} \ll \sigma_{\im}{\varepsilon}_{n}(\delta)$, the bias becomes negligible so that the rate is dominated by ${\varepsilon}_{n}(\delta)$, which is essentially parametric. Finally, we also note that if the true gradient belongs to $\Wiip{p}$, then it can be accurately estimated in the Sobolev norm $\Wiip{s}$ for any $s<p$, which is a stronger result than a guarantee solely in the $\Lii$ norm.

\section{Experiments}
\label{sec:experiments}
In this section, we present four experiments, for classical and spherical Langevin dynamics, each highlighting the diverse capabilities of the LITL framework.

\textbf{Double well potential:} 
In this first experiment, we test our method on a 1D-toy system for which ground truth references can be numerically computed. \textit{The aim is to compare our method to the baselines on spectral information recovery, and demonstrate performance on recovery of distribution flow for which no baseline exists.} We work with a double well potential $U(x)$ to which we add a perturbation potential $V(x)$ which will induce more barrier crossing events from one minimum to the other, see Figure \ref{fig:double_well} in App. \ref{app:1d2w}. We simulate the dynamics under $U'(x)=(U{+}V)(x)$ and keep the values of the position and of $V(x)$ with time. %Our goal is to recover the eigenpairs of the infinitesimal generator of the dynamics.  
\begin{table}[h!]
\vspace{-0.0cm}
\centering\small
\setlength{\tabcolsep}{4pt} % default ~6pt
\begin{tabular}{|c|c|c|c|c|}
\hline
Method & $\lambda_1 {=} 0.22$ & $\lambda_2 {=} 15.37$ & $\lambda_3{=}16.06$ & $\lambda_4{=}47.31$ \\ 
\hline
\cite{devergne2024biased} & 0.621 $\pm$ 0.01  & 46.6 $\pm$ 0.3  & 49.1 $\pm$ 0.5 & 146 $\pm$ 2 \\ \hline
\cite{ZHANG2022diffusionNN} & 1 $\pm$ 2  & 2 $\pm$ 2 & 2 $\pm$ 2 & 8 $\pm$ 7\\ \hline
LITL & \textbf{0.23 $\pm$ 0.02} & \textbf{15.315 $\pm$ 0.01} &  \textbf{16.09 $\pm$ 0.04} &  \textbf{46.9 $\pm$ 0.1} \\ 
\hline
\end{tabular}
\caption{Eigenvalue estimation in the double well experiment.}
\label{tab:evals}
\end{table}
We compare the results obtained with LITL and with, up to our knowledge, the only two existing methods for the same task \cite{devergne2024biased} and \cite{ZHANG2022diffusionNN}. While Table \ref{tab:evals} shows that LITL is the only method that accurately estimated higher eigenvalues, in  Figure \ref{fig:double_well} of App.~\ref{app:1d2w} we show that the same holds for eigenfunction estimation. We note that results of \cite{ZHANG2022diffusionNN} suffer from
uncontrolled collapse of learned subspace to slower modes, likely caused by larger number of
sensitive hyper-parameters.
% In Figure \ref{fig:double_well}\vladi{maybe we keep only table here and have figures in the appendix?} we show the accuracy of the estimated eigenfunctions by computing uncertainties over 4 different trials with different random initialization. The results of \cite{ZHANG2022diffusionNN} are omitted due to frequent uncontrolled collapse of learned subspace to slower modes, likely caused by larger number of sensitive hyper-parameters. The other two methods perform very well on the first eigenfunction, which characterizes the transition between the two wells, the only slow mode of the system. But then, as we increase the eigenfunction index, the variance between trials for the second method increases, while LITL's one remains. 
% This is even clearer when we look at the eigenvalues estimation in Table \ref{tab:evals}, where \cite{devergne2024biased} and \cite{ZHANG2022diffusionNN} do not manage to estimate accurately higher eigenvalues. 

\noindent
Finally, we can also learn the forces acting on the system using $\widehat{\nabla\PT} {=} \widehat{\MX{G}}\NNe(\cdot)$ with $\widehat{\MX{G}}$ given in \eqref{eq:emp_gradient} and perform simulations of the true potential $\PT$ for a lower cost, which is shown in Figures \ref{fig:double_well_drift} and \ref{fig:1D2W_flow} of App.~\ref{app:exp}, however LITL remains a unique method that solves this task.

% \begin{figure}[t!]
%     \centering
%     \vspace{-0.3cm}
%     %\includegraphics[width=0.65\linewidth]{figs/figure_double_well.png}
%     \includegraphics[width=0.3\linewidth]{figs/double_well_2_legend.png}\includegraphics[width=0.65\linewidth]{figs/double_well_2.png}
%     \caption{Top left panel: the different potentials involved in this experiment in units of $1/\beta$. Other panels: comparison of the eigenfunctions obtained with LITL and with the methods presented in \cite{devergne2024biased} and \cite{ZHANG2022diffusionNN}. Error bars have been computed using 4 different models with different random initial parameters}
%     \label{fig:double_well}
%     \vspace{-0.7cm}
% \end{figure}

\textbf{Alanine-Dipeptide molecule:} \textit{The second experiment is a controlled molecular transfer benchmark,  demonstrating how LITL transfers data samples from biased (non-physical) simulations of Alanine-Dipeptide to recover physical transition rates.} The system has 21 dimensional state made of dihedral angles, and due to the curse of dimensionality spectral decompositions cannot be numerically computed. The key message is that this simple molecule displays two main slow modes, of which the  slower one is identified by the dihedral angle $\phi$ and cannot be observed without the addition of a bias potential during a simulation, which then {requires the transfer learning} method. Hence, we trained the LITL model with and without reweighting and observed that for the slowest mode, the eigenvalue changes from -0.56 in the non-reweighted model (learning without transfer yields the physically impossible  timescale $\approx$2) to -0.029 in the reweighted  model (giving the physically meaningful timescale $\approx$350).

\textbf{Chignolin miniprotein:} \textit{In this experiment we demonstrate how LITL can equip generative molecular models with kinetic capaabilities}. Modern GenAI, such as BioEmu~\cite{bioemu2025}, can now produce equilibrium protein conformations at unprecedented scale, but evaluating whether these samples reproduce the \emph{correct kinetics} remains a major challenge. Standard spectral methods such as TICA \cite{schutte_tica} require trajectory data and therefore cannot be applied to static samples generated by foundation models. In contrast, LITL learns slow dynamical modes directly from equilibrium samples, providing spectral collective variables that identify metastable states and transition pathways without costly simulations.

\begin{figure}
    \centering
    \includegraphics[width=0.75\linewidth]{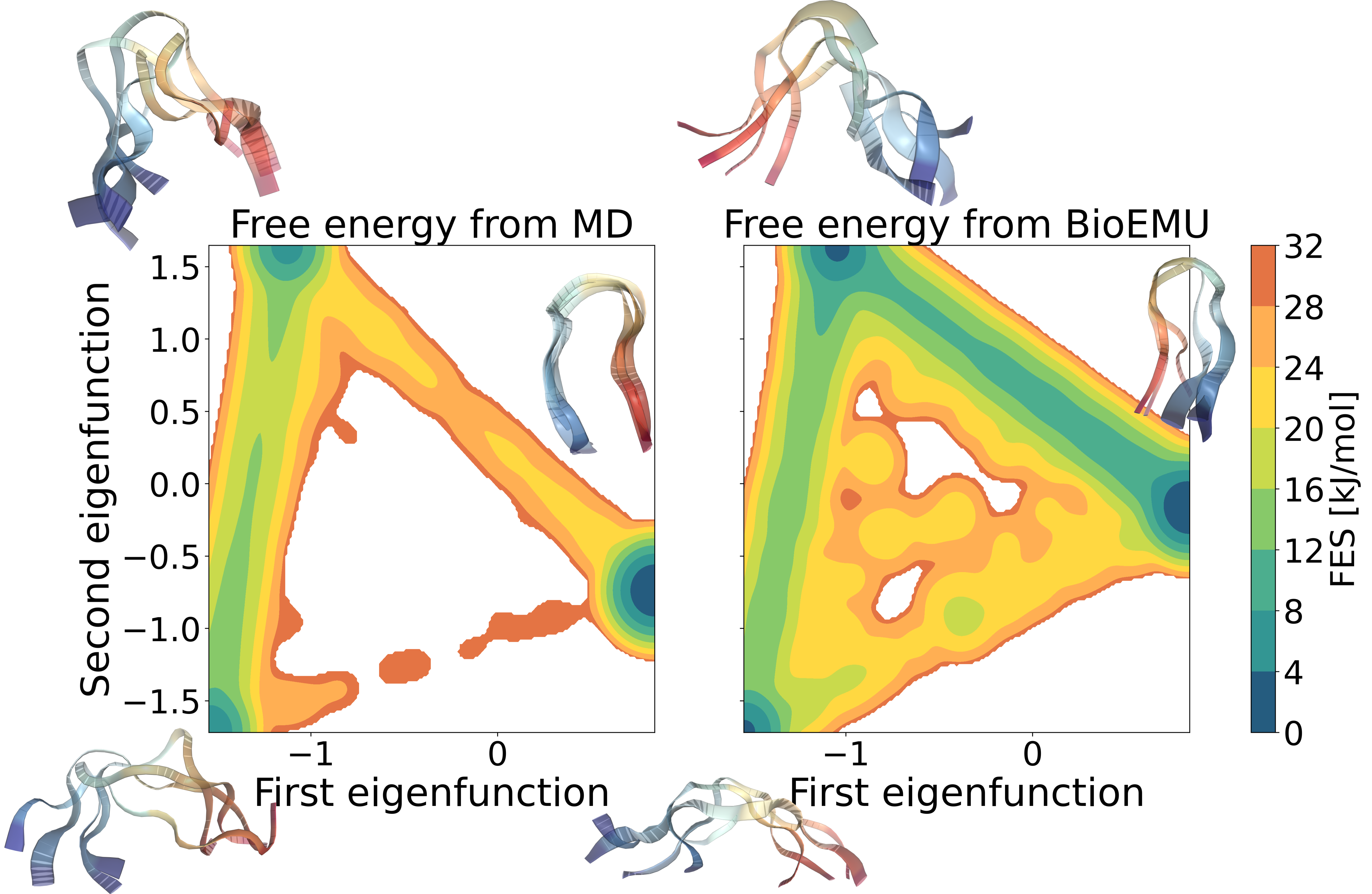}
    \vspace{-0.2cm}
    \caption{Free-energy landscapes learned from MD and BioEmu samples using LITL. State occupation probabilities computed from the recovered generators reveal kinetic mismatch between BioEmu and physical MD dynamics.}
    \label{fig:chignolin}
    \vspace{-0.5cm}
\end{figure}

We train LITL using graph neural networks \cite{schnet} on (i) long MD trajectories of the Chignolin miniprotein \cite{chignolin} and (ii) static samples generated by BioEmu \cite{bioemu2025}. Projecting onto the two slowest learned eigenfunctions reveals the underlying free-energy landscape and metastable organization (Fig.~\ref{fig:chignolin}).

Qualitatively, both models recover the same dominant folding mechanisms: the first mode corresponds to folding, while the second captures transitions between unfolded and intermediate hairpin states. Quantitatively, however, the recovered kinetics differ substantially. In particular, BioEmu underestimates the free-energy barrier between unfolded and folded states, leading to significantly different metastable occupancies (folded/intermediate/unfolded: $60/24/16\%$ for BioEmu versus $77/8/15\%$ for MD).

These results show that LITL not only extracts physically meaningful slow modes from static generative samples, but also enables direct kinetic comparison between generative and physical molecular dynamics models.

\textbf{Cobalt nanoparticle dynamics.} \textit{Here we demonstrate how spherical LITL can recover structured physics of magnetization using only measurements.} We learn the spectral properties of the Langevin generator on a sphere $\mathbb{S}^2$, modeling the magnetization dynamics of a cobalt nanoparticle under the potential $U(\mathbf{m}) \,{=}\, {-}K_u (\mathbf{m} \,{\cdot}\, \mathbf{n}_u)^2 \,{+}\, K_c (m_x^2 m_y^2 \,{+}\, m_y^2 m_z^2 \,{+}\, m_z^2 m_x^2)$, $\mathbf{n}_u$ being easy axis of uniaxal anisotropy. \citep{garanin1997fokker,garcia1998langevin}. For uniaxal and cubic anisotropy constants $K_u {=} 1.0$ and $K_c {=} 0.3$, respectively, four metastable states are separated by energy barriers. Understanding the thermodynamics and kinetics of cobalt nanoparticles, in particular the metastable states and transition rates revealed by the Langevin generator, helps set switching fields for memory devices, control blocking temperatures and data retention times, and optimize heat dissipation in hyperthermia applications. LITL's ability to infer these crucial objects using only potential evaluations is critical to experimental scenarios where anisotropy landscapes are reconstructed from measurements such as micro-SQUID experiments \cite{jamet2001magnetic}. 

\begin{figure}[h!]
    \centering
\includegraphics[width=0.85\linewidth]{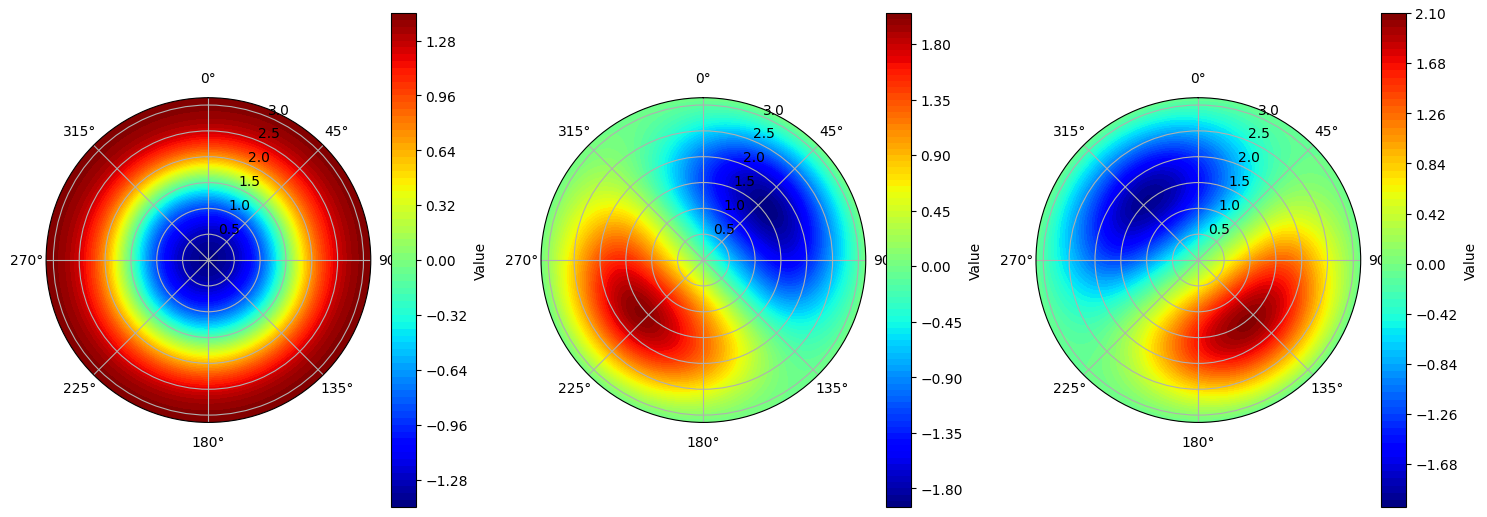}
    \caption{\textit{Leading eigenfunctions of the Langevin generator for a cobalt nanoparticle potential} in polar coordinate representation $\psi(\theta,\varphi)$ associated to eigenvalues $|\egeval_1|{\approx}1.36$ and $|\egeval_2|{\approx}|\egeval_3|{\approx}2.5$. The first non-trivial eigenfunction $\psi_1$ reveals the slowest transition pathway between metastable states, while couple $(\psi_2, \psi_3)$ captures faster symmetric modes orthogonal to the first one. 
    }
    \label{fig:eigenfunctions}
\end{figure}

Figure \ref{fig:eigenfunctions} shows results from the LITL estimator that achieves  $\approx 5\cdot10^{-3}$ errors w.r.t. numerically computed ground truth for both time-scales and their eigenfunctions (in $L^2$ norm) using $n=10^4$ uniform samples. Remarkably, {our method was able to recover over ten eigenfunctions with the correct symmetry groups of the model}, see App. \ref{app:cobalt}, {purely from black-box potential evaluations} and without prior knowledge of the functional form, directly impacting applications from data storage (bit stability) to hyperthermia (heat dissipation pathways). 

\textbf{Fairness-aware transfer via latent diffusion.} \textit{Finally, we present a proof-of-concept experiment in which we assess the LITL's performance in the challenging setting of transfer learning related to algorithmic fairness.} To this end, we chose the Adult Income dataset ({\it https://archive.ics.uci.edu/dataset/2/adult}), a standard benchmark for fairness in binary classification, and consider the difficult unawareness case \cite{chen2019fairness,donini2018empirical,zemel2013learning}. The task is to predict whether an individual’s annual income exceeds \$50K, using demographic and employment-related features. Following common practice, we consider \emph{gender} as the sensitive attribute and measure fairness by Equal Opportunity (EO) given by the \emph{True Positive Rate (TPR) difference} between males and females. The dataset is split into the \textit{training} (60\%), \textit{audit} (25\%) and \textit{test} (15\%), with the sensitive variable removed from the training set. The audit data is used to return black-box feedback on model's compliance with desired fairness criterion, typically at most $5\%$ TPR disparity between the groups.

The classifier $f_{\omega}(x)\,{=}\,\softmax(a_{\omega}^\top e_{\omega}(x))$ 
learned via empirical risk minimization on the training set, typically achieves strong predictive performance but does not pass the audit check. Here, $\omega$ denotes the classifier's parameters, associated with the latent (inner) representation $e_{\omega}(x) \in \mathbb{S}^{d-1}$, normalized to lie on the sphere, and the last layer weights $a_\omega\in\R^{d}$. To remedy this, we consider Spherical LITL method and transfer $n=5k$ random representations uniformly sampled on a sphere $\mathbb{S}^{31}$ to the representations that tend to minimize the TPR gap. That is, while $\bPT(e)=0$, the feedback from the audit set $\bias(e){=}{-}\PT(e)$ is computed for each $e$ as local TPR gap on the representation manifold of the audit set, see App.~\ref{app:fairness} for details. 

While the source generator $\bGE$ of pure diffusion, the Laplace-Beltrami operator, captures the intrinsic geometry and dominant modes of variation of the learned representation subject to pure randomness, the target Langevin generator $\GE$'s eigenfunctions define a hierarchy of fairness related slow and fast directions on the latent manifold. In particular, the leading nontrivial eigenfunctions correspond to large-scale, semantically meaningful variations in the representation, while higher modes capture localized or noisy fluctuations. This multiscale spectral structure provides a natural coordinate system for analyzing and adjusting the model’s fairness behavior in a controlled and interpretable way. In particular, we use learned gradient field $\w{\nabla}_{\mathbb{S}^{d{-}1}}\PT(e){=}(\MX{I}{-}ee^\top)\w{G}\NNe(e)$ to perform short-horizon, geometry-respecting updates on the latent sphere with a small $\gamma>0$ steps, steering the pretrained model $f_{\omega}$ toward configurations with reduced TPR disparity while preserving the dominant representational structure learned during initial training. 

Figure~\ref{fig:pareto_adult} reports the resulting Pareto curves for different step sizes $\gamma$, illustrating the trade-off between accuracy and fairness induced by the LITL intervention. Each point on the curve corresponds to a different number of LITL steps. As the intervention increases, the TPR difference significantly decreases to acceptable 1\%, while accuracy degrades gracefully. 

\begin{figure*}[t!]
    \centering
\includegraphics[width=0.95\linewidth]{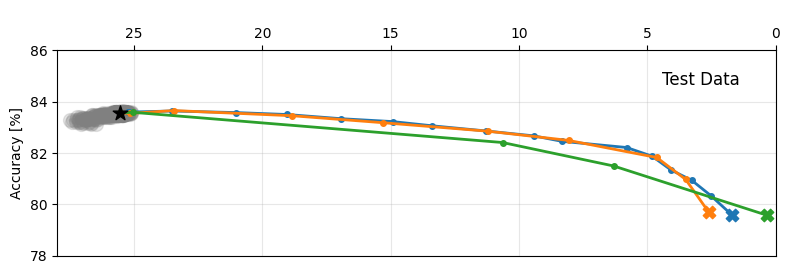}
    \caption{Accuracy--fairness Pareto curve for the Adult Income experiment. Each point corresponds to a different number of LITL steps. LITL hyperparameters are calibrated on audit set, and then used to achieve substantial reduction in TPR disparity with a limited decrease in accuracy on a test set.}
\label{fig:pareto_adult}
\end{figure*}

Moreover, as shown in App.~\ref{app:fairness}, LITL analysis reveals that around $50\%$ of the TPR gap is due to the model, which can be eliminated practically without the loss of accuracy by using 16 slowest eigenfunctions of $\GE$. On the other hand, the intrinsic bias in the data can be corrected with additional 16 eigenfucntions which efficiently balance competing objectives TPR gap and accuracy.

\smallskip

\begin{table}
\centering\small
\setlength{\tabcolsep}{4pt} 
\begin{tabular}{|l|c|c|}
\hline
Method & Rel. Acc. [\%] $\uparrow$ & Rel. TPR gap  [\%] $\uparrow$  \\
\hline \hline
\textbf{LITL (full)} & \textbf{97.11 $\pm$ 1.00} & \textbf{94.16 $\pm$ 4.04} \\ \hline
LITL (local proxy) & 99.91 $\pm$ 0.09 & 8.02 $\pm$ 2.41 \\
\hline
LITL (no feedback) & 99.27 $\pm$ 0.37 & 59.88 $\pm$ 5.08 \\ \hline
LITL (bad truncation) & 95.68 $\pm$ 0.89 & 70.68 $\pm$ 5.72 \\ \hline \hline
Retraining baseline & 101.32 $\pm$ 0.21 & 94.60 $\pm$ 4.30 \\ \hline
Fine-tuning baseline & 100.25 $\pm$ 0.17 & 90.11 $\pm$ 5.78 \\ \hline
Latent correction baseline & 98.34 $\pm$ 0.61 & 32.33 $\pm$ 2.51 \\ \hline
\end{tabular}
\caption{Ablation study and baseline comparisons in the fairness experiment over 10 trials with recorded relative accuracy and TPR gap change w.r.t. the original classifier.}
\label{tab:fairness}
\end{table}

To further support above conclusions, we perform ablation study by changing key components of LITL, and compare to adequate ML baselines.  As one can observe in Table \ref{tab:fairness}, (i) replacing the generator-based update collapses performance ($\leq10\%$),  showing that gains are not due to latent post-processing (here via local finite differences); (ii) removing audit signal degrades performance, yet still yields $\approx50 – 60\%$ improvement, significantly outperforming zero-order baselines; (iii) incorrect truncation degrades performance, confirming the importance of recovering the correct slow eigenspace. 
%In addition, we note that item (ii) is theoretically grounded: LITL learns Laplace–Beltrami eigenfunctions, and a constant spherical potential projected onto low-frequency modes induces a smooth, geometry-aware “spectral smoothing” flow, diffusing a substantial portion of the model bias.  To conclude, 
Together, these results show that \textit{importance weighting, spectral structure, and generator-based updates} are all necessary for the observed $\approx94\%$ fairness improvement. Concerning baselines, since no standard fairness methods operate under identical constraints, we follow \cite{donini2018empirical} and include retraining and fine-tuning in an unawareness setting (model excludes sensitive features but uses them in loss regularization). We also include a zero-order latent correction (nearest-neighbor weighted ridge regression of spherical gradients) under the same constraints. All methods use matched tuning budgets. As we can observe, LITL achieves performance comparable to retraining/fine-tuning but \textit{without access to sensitive features, gradients, or retraining}, and clearly outperforms black-box latent correction.

\section{Conclusions}
%Impact and limitations}
\label{sec:conclusions}

We introduced Langevin-Informed Transfer Learning (LITL), a framework for recovering and transferring target stochastic dynamics from biased samples using only black-box feedback. By learning the target generator’s slow spectral structure and projected drift with finite-sample guarantees, LITL enables principled coarse-graining, kinetic reconstruction without target trajectories, and post-hoc steering of ML models without retraining. Conceptually, LITL reframes transfer learning as the recovery and manipulation of stochastic dynamics rather than parameter adaptation.  In other words:
\begin{center}
\vspace{-.1truecm}
{\em Instead of transferring weights or features, LITL transfers generator structure, \\enabling geometry-aware steering through learned slow manifolds.}
\vspace{-.1truecm}
\end{center}
Our results suggest broader relevance beyond molecular systems. Since Langevin-type dynamics also underlie stochastic gradient flows and diffusion-based generative models, spectral operator learning may provide a new route for analyzing and steering latent stochastic dynamics in modern AI systems through gradient-free feedback.

{\bf Limitations.}
%At the same time, 
LITL currently focuses on reversible Langevin dynamics and assumes informative black-box evaluations together with identifiable low-dimensional slow structure. Extending the framework to non-reversible or time-dependent stochastic dynamics remains an important direction for future work.

\bibliography{neurips2026}
\bibliographystyle{plain}

%%%%%%%%%%%%%%%%%%%%%%%%%%%%%%%%%%%%%%%%%%%%%%%%%%%%%%%%%%%%%%%%%%%%%%%%%%%%%%%
%%%%%%%%%%%%%%%%%%%%%%%%%%%%%%%%%%%%%%%%%%%%%%%%%%%%%%%%%%%%%%%%%%%%%%%%%%%%%%%
% APPENDIX
%%%%%%%%%%%%%%%%%%%%%%%%%%%%%%%%%%%%%%%%%%%%%%%%%%%%%%%%%%%%%%%%%%%%%%%%%%%%%%%
%%%%%%%%%%%%%%%%%%%%%%%%%%%%%%%%%%%%%%%%%%%%%%%%%%%%%%%%%%%%%%%%%%%%%%%%%%%%%%%
%\newpage
\appendix

\vspace{.5truecm}
\begin{center}
{\Large \bf Appendix}    
\end{center}

%\onecolumn
The appendix is organized as follows.

\begin{itemize}
\item Sec.  \ref{app:related_work} provides a detailed discussion of related work and its comparison to LITL. 
%contains a discussion on positioning of LITL in the landscape of related works.
\item Sec.  \ref{app:background} presents additional details on the framework, with particular emphasis on Langevin dynamics and the definition of Dirichlet forms and energy-based norms.
\item Sec.  \ref{app:biased_dynamics} details the transfer learning setup and the importance weighting mechanism for debiasing.
\item Sec.  \ref{app:loss} presents details on the derivation of our training loss and spectral learning.
\item Sec.  \ref{app:bounds} establishes concentration bounds and proves our main statistical guarantees for eigenvalue, eigenfunction, and gradient estimation.
\item Sec.  \ref{app:exp} presents additional numerical experiments, expanding on those in the main body of the paper. 

\end{itemize}

\section{Related work and positioning} %of LITL}
\label{app:related_work}

\begin{table}[h!]
\centering
\small
\begin{tabular}{lcccc}
\hline
\textbf{Learning} 
& \textbf{Transfer} 
& \textbf{Needs Target} 
& \textbf{Needs} 
& \textbf{Representative} \\
\textbf{Method} 
& \textbf{Mechanism} 
& \textbf{Samples or} 
& \textbf{Retraining} 
& \textbf{References} \\
&&\textbf{Gradients}&&\\
\hline

Fine-tuning /  
& Parameters /  
& Yes 
& Yes/No 
& \cite{yosinski2014transferable}\\
Feature Transfer &Features&&& \cite{pan2010survey} \\

\hline
Domain Adaptation /  
& Reweighted Loss 
& Yes/No 
& Yes 
& \cite{ben2010theory}\\
Covariate Shift&&&& \cite{shimodaira2000improving} \\

\hline
Self-supervised /  
& Representations /  
& Yes 
& Yes 
& \cite{jing2020self}\\
Meta-learning&Updates&&&\cite{finn2017maml} \\

\hline
Adversarial Training 
& Robust Parameters 
& Yes 
& Yes 
& \cite{salman2020do} \\

\hline
SGLD /  
& Posterior  
& Yes 
& Yes 
& \cite{welling2011bayesian}\\
Bayesian Sampling&Samples&&&\cite{ma2015complete} \\

\hline 
Latent   
& Learned   
& Yes 
& No 
& \cite{sohl2015deep}\\
Diffusion&Latent&&& \cite{ho2020denoising}\\ 
Models&Dynamics&&& \cite{rombach2022high} \\

\hline
\textbf{LITL (Ours)} 
& \textbf{Generator Eigenspace} 
& \textbf{No} 
& \textbf{No} 
& -- \\
\hline\\
\end{tabular}
\caption{Comparison of LITL with related transfer learning, robustness, sampling, and diffusion-based approaches. Latent diffusion models transfer learned generative dynamics through retraining in latent spaces, whereas LITL transfers spectral structure of the underlying Langevin generator without requiring target gradients, retraining, or access to target labels.}
\label{tab:litl_comparison}
\vspace{-0.5cm}
\end{table}

\paragraph{Overview.}
Table~\ref{tab:litl_comparison} summarizes the relationship between \emph{Langevin Informed Transfer Learning (LITL)} and major paradigms in transfer learning, robustness, Bayesian sampling, and generative modeling. Existing approaches predominantly transfer \emph{parameters, representations, losses, or optimization heuristics} and almost universally require access to target gradients, labels, or explicit retraining. In contrast, LITL transfers \emph{spectral structure of the Langevin generator}, enabling post-training adaptation without target gradients, labels, or retraining.

\vspace{-0.5em}
\paragraph{Transfer learning, domain adaptation, and meta-learning.}
Classical transfer learning focuses on reusing parameters or representations learned on a source task to accelerate optimization or improve generalization on a target task, typically through fine-tuning or feature reuse \citep[see][and references therein]{pan2010survey, yosinski2014transferable}. Domain adaptation and covariate shift methods reweight losses or samples to correct distribution mismatch, assuming access to target data and gradients \cite{shimodaira2000improving, ben2010theory}. Meta-learning provides a different setting in which shared representations are learned from multiple related tasks \cite{baxter2000model,finn2017maml,jing2020self,maurer2016benefit}. While effective, these methods fundamentally rely on retraining and explicit optimization against a target objective. LITL departs from this paradigm by transferring \emph{dynamical structure rather than parameters or losses}, enabling adaptation to unseen objectives without optimization in the target domain.

\vspace{-0.5em}
\paragraph{Bayesian sampling and Langevin dynamics.}
Stochastic Gradient Langevin Dynamics (SGLD) and related MCMC methods provide approximate Bayesian inference by injecting noise into gradient-based optimization \cite{welling2011bayesian, ma2015complete}. These approaches assume direct access to the target energy or its stochastic gradients and are not designed for controlled transfer across objectives.  LITL is intrinsically linked to these methods,  but addresses a distinct inverse problem: learning and transferring the generator itself, rather than sampling from a known energy. This enables principled post-hoc adaptation without modifying step sizes, noise schedules, or loss functions.

\vspace{-0.5em}
\paragraph{Diffusion and score-based generative models.}
Score-based diffusion models and denoising diffusion probabilistic models learn gradients of log-densities (the potential) in order to generate samples by reversing a predefined noising process \cite{ho2020denoising}. Latent diffusion models further improve scalability by learning generative dynamics in compressed latent spaces \cite{rombach2022high}.  LITL is related to score-based diffusion models in that learn gradients of log-densities which can be used to generate samples. However, the similarity is superficial. Diffusion models learn pointwise, time-dependent score fields using samples from the target distribution or its noisy marginals \cite{sohl2015deep}. In contrast, LITL never observes samples from the target Gibbs/Boltzman distribution and does not learn a score field directly. Instead, it learns spectral representations of the Langevin generator from biased trajectories and forward-only side information, and reconstructs the target drift only after projection onto the dominant slow eigenspaces. This yields a \textit{spectrally filtered score} that governs long-time dynamics rather than pointwise density matching. As a result, LITL enables the creation of unseen kinetic models with finite-sample guarantees, capabilities that are fundamentally outside the scope of existing diffusion and score-based approaches. We summarize this comparison in Table \ref{tab:litl_vs_diffusion}.

\begin{table}[h!]
\centering
\small
\setlength{\tabcolsep}{5pt}
\begin{tabular}{l p{4.75cm} p{4.95cm}}
%\begin{tabular}{lcc}
\hline
\textbf{Aspect} 
& \textbf{Diffusion / Score-Based Models} 
& \textbf{LITL (Ours)} \\
\hline
Task
& Generate samples from a given data distribution
& Create dynamics from log density discrepancies \\
Observed data
& Target samples or noisy marginals
& Samples from a source Gibbs distribution \\
Learned object
& Score field $\nabla \log p_t(x)$
& Generator eigenspaces \& projected score field\\
Score estimation
& Full, pointwise
& Spectrally projected (slow manifold) \\
Time dependence
& Time-dependent
& Time-homogeneous \\
Underlying dynamics
& Artificial noising process
& Physical / optimization-induced Langevin dynamics \\
Target gradients
& Required
& Not required \\
Target samples
& Required
& Not required \\
Transfer across objectives
& Not supported
& Supported via transfer of the potential $U' \rightarrow U$ \\
Theoretical guarantees
& Limited
& Finite-sample spectral and gradient bounds \\
\hline
\end{tabular}
\vspace{.1truecm}
\caption{Compact comparison between diffusion/score-based generative models and Langevin-Informed Transfer Learning (LITL). Diffusion models learn scores to reproduce observed data distributions, whereas LITL transfers kinetic structure to construct unseen dynamics without access to target samples or gradients.}
\label{tab:litl_vs_diffusion}
\end{table}

\vspace{-0.5em}
\paragraph{Spectral methods, Koopman operators, and generator learning.}
A large body of work studies data-driven approximation of Koopman operators and infinitesimal generators using kernel methods and deep learning \cite{LGBPP12, KLUS2020132416, Bruder2021, Bevanda2021, Mardt2018, colbrook2023rigorous}. These methods successfully identify slow modes and metastable states but typically assume access to trajectories generated by the dynamics of interest, which is typical bottleneck of  these approaches. So, to accelerate rare event sampling, practitioners use biased simulations with methods such as metadynamics and its variants\cite{laio_metad,invernizzi2020rethinking}, umbrella sampling \cite{torrie_US}, and adaptive biasing force \cite{abf}. Furthermore, machine learning has been employed to learn collective variables from data \cite{kTICA,Bonati2021,chen2019}. While effective for sampling, these methods do not provide a principled framework or statistical guarantees for transferring learned spectral information from biased to unbiased dynamics. Crucially, transfer operator methods are fundamentally incompatible with biased sampling: while the infinitesimal generator depends linearly on the bias potential, the transfer operator does not, making principled debiasing impossible.  Importance reweighting techniques that correct distribution shift at the generator level fail to properly debias transfer operator estimates. Recent work therefore focuses on learning the generator directly, often via its resolvent, which is compact and shares eigenfunctions with the generator \cite{kostic2024learning, kostic2024laplace, ZHANG2022diffusionNN}. While \cite{kostic2024learning} provide statistical guarantees in RKHS settings without transfer, and \cite{devergne2024biased} address biased data using neural features but lacks strong finite sample guarantees. Notably, none of these works explore transfer learning across objectives without retraining. LITL extends this line of work by improving the performance and theoretical analysis, as well as providing a principled mechanism for transferring spectral structure across biased and unbiased Langevin systems, with explicit control of representation and optimization error.

\vspace{-0.5em}
\paragraph{Summary perspective.}
LITL operates at the level of distributional geometry and kinetic structure. By learning dominant eigenspaces of the Langevin generator from source samples or models and forward-only side information, it enables principled post-training adaptation to unseen objectives. This places LITL orthogonally to existing transfer learning, robustness, Bayesian sampling, and diffusion-based approaches, motivating one to further explore it in combination with those.

\section{Background: Langevin dynamics and Dirichlet forms}\label{app:background}

\paragraph{Langevin dynamics (LD) and evolution operators}
As a reminder, we consider the Langevin equation
\begin{equation}\label{Eq:langevin_2}
dX_t=-\nabla U(X_{t})dt+\sqrt{2\beta^{-1}}dW_t \quad \text{and} \quad X_0=x,
\end{equation}
which describes the dynamics of the random  vector $X_t$ in the state space $\X\subseteq\R^d$, governed by a potential $U:\mathbb{R}^d \to \mathbb{R}$, where $(W_t)_{t\ge 0}$ is a $\R^p$-dimensional standard Brownian motion. Under the usual conditions \citep[see e.g.][]{oksendal2013} that $a$ and $b$ are globally Lipschitz and sub-linear, the SDE \eqref{Eq:langevin_2} admits a unique strong solution $X=(X_t)_{t\ge 0}$ that is a Markov process to which we can associate the semi-group of Markov \textit{transfer operators} $(\TO_t)_{t\geq 0}$ defined, for every $t \geq 0$, as
\begin{equation}\label{eq:TO2}
[\TO_t f](x):=\EE[f(X_t)|X_0=x],\;\;x\in\X, \, f\colon \X\to \R.
\end{equation}

For stable processes, the distribution of $X_t$ converges to an \textit{invariant measure}
%, i.e. there exist a probability measure 
$\im$ on $\X$, such that $X_0\sim\im$ implies that $X_t\sim\im$ for all $t\geq0$. In such cases, one can define the semi-group on $\Lii$, and characterize the process  by the \textit{infinitesimal generator} of the semi-group $(\TO_t)_{t\ge 0}$,
\begin{equation}\label{eq:gen_lim}
\GE f:= \lim_{t\to0^+}\frac{\TO_t f-f}{t} 
\end{equation}
defined
on the Sobolev space $\mathcal W_{\pi}^{1,2}$ of functions in $\Lii$ whose gradients are in $\Lii$, too, i.e. $\GE\colon\Lii\to\Lii$ and $\dom(\GE)=\Wii$. The transfer operator and the generator are linked to each other by the formula $\TO_t=\exp(t\GE)$.

% We can evolve observables of the dynamical system in time through the TO. Indeed for any $f\in \Lii$,
% \begin{equation}
% \TO_t f
% = \EE[f(X_t)\mid X_0=x]
% = \textstyle{\sum_{i\in\mathbb N_0}} e^{t\geval_i}\,\gefun_i(x)\scalarp{\gefun_i,f}.
% \end{equation}
% The above represenation is called the Dynamic Mode decompoisition.

After defining the infinitesimal generator for Markov processes by \eqref{eq:gen_lim}, we provide its explicit form for solution processes of \eqref{Eq:langevin_2}. Given a smooth function $f\in\mathcal C^2(\X,\R)$, Itô's formula \citep[see for instance][p. 495]{Bakry2014} provides for $t\in\R_+$,
\begin{align}
f(X_t)-f(X_0)
\nonumber
&=\int_0^t\sum_{i=1}^d\partial_i f(X_{s})dX_s^i+\tfrac{1}{2}\int_0^t\sum_{i,j=1}^d\partial_{ij}^2 f(X_{s})d\langle X^{i},X^j\rangle_s\\
\nonumber
&=\int_0^t \nabla f(X_{s})^{\top}dX_s+\tfrac{1}{2}\int_0^t \mathrm{Tr}\big[X_s^\top (\nabla^2 f)(X_{s})X_s\big]ds,
\end{align}
where $\nabla^2 f$ is the Hessian matrix of $f$. Recalling \eqref{Eq:langevin_2}, we get
\begin{align}\label{Eq: Ito formula}
f(X_t)
\nonumber
&=f(X_0)+\int_0^t\bigg[-\nabla U(X_s)^\top \nabla f(X_s)+\beta^{-1}\mathrm{Tr}\big[ (\nabla^2 f(X_s))\big]\bigg]ds +\sqrt{2\beta^{-1}}\int_0^t\mathrm{div} f(X_s)dW_s.
\end{align}
Provided $f$ is smooth enough, the expectation of the last stochastic integral vanishes so that we get
\begin{equation*}
\EE[f(X_t)|X_0=x]=f(x)+\int_0^t\EE\big[-\nabla U(X_s)^\top\nabla f(X_s)+\Delta f(X_s)\big|X_0=x\big]ds
\end{equation*}
Recalling that $\GE f=\underset{t\rightarrow 0^+}\lim(\TO_tf-f)/t$, we get for every $x\in\X$,
\begin{align}\label{Eq: computation L}
\GE f(x)
&= \lim_{t\to 0^+} \frac{\EE[f(X_t)\mid X_0=x]-f(x)}{t} \nonumber \\
&= \lim_{t\to 0^+} \frac{1}{t}
\int_0^t 
\EE\Big[
-\nabla U(X_s)^\top \nabla f(X_s)
+ \beta^{-1}\Delta f(X_s)
\,\Big|\, X_0=x
\Big] ds \nonumber \\
&= -\nabla U(x)^\top \nabla f(x)
+ \beta^{-1}\Delta f(x).
\end{align}

which provides the closed formula for the IG associated with the solution process of \eqref{Eq:langevin_2}. The operator $\GE$ is closed on its domain $\mathrm{dom}(\GE)$ which is equal to the Sobolev space
defined for $\shift>0$ by
\begin{equation}\label{Eq:Sobolev_space}
\Wii=\Big\{f\in\Lii\;\vert\; \norm{f}_{\Wiishort}^2=\langle{\resolventinv}f,f\rangle_{\Liii}=\norm{f}_{\Liishort}^2+\frac{1}{\shift\beta}\, \norm{\nabla f}_{\Liishort}^2<\infty\Big\}.
\end{equation}
An integration by parts argument yields
\begin{equation*}
\int (\GE f)\, g \, d\pi
= - \beta^{-1} \int \langle \nabla f , \nabla g \rangle \, d\pi
= \int f\, (\GE g)\, d\pi ,
\end{equation*}
showing that $\GE$ is self-adjoint on $\Liishort$.

\paragraph{Spectral decomposition and resolvent.}
When the spectrum of $\GE$ is discrete, the generator admits the spectral decomposition
\begin{equation}\label{eq:evd_app}
\GE=\sum_{i\in\mathbb N_0}\lambda_i\,\gefun_i\otimes\gefun_i,
\end{equation}
which yields, for any $f:\mathcal X\to\mathbb R$,
\begin{equation}\label{eq:transfer_spectral}
\TO_t f = \EE[f(X_t)\mid X_0=x] = \textstyle{\sum_{i=0}^\infty} e^{t\geval_i}\,\gefun_i(x)\scalarp{\gefun_i,f}.
\end{equation}
The characteristic timescales of the system are determined by the inverse generator eigenvalues $1/|\lambda_i|$. 

The resolvent operator $(I-\shift^{-1}\GE)^{-1}$ for $\shift>0$ is a crucial object in our analysis. For confining potentials, the resolvent is a compact operator on $\Lii$, unlike the unbounded generator $\GE$. This compactness property makes the resolvent more tractable for statistical learning while preserving the spectral information of the generator. Indeed, the resolvent shares the same eigenfunctions as $\GE$ with transformed eigenvalues: if $(\lambda_i, \gefun_i)$ are eigenpairs of $\GE$, then $((1-\lambda_i/\shift)^{-1}, \gefun_i)$ are eigenpairs of the resolvent.

%If $\EGE$ is available, we can define estimators $\widehat{\TO}_t = e^{t\EGE}$ and for any observable $f\in \Lii$: \begin{equation}\label{eq:transfer_spectral_emp} \widehat{\TO}_t f = \textstyle{\sum_{i\in\N_0}} e^{t\egeval_i}\,\egefun_i(x)\scalarp{\egefun_i,f}. \end{equation}

\paragraph{Empirical approximation of the transfer operator.}
\massi{I rephrased a bit this paragraph}
If $\EGE = \sum_{i=0}^m \egeval_i \, \egefun_i \otimes \egefun_i$ is available, we can define estimators $\widehat{\TO}_t = e^{t\EGE}$ and for any observable $f\in \Lii$: \begin{equation}\label{eq:transfer_spectral_emp} \widehat{\TO}_t f = \textstyle{\sum_{i\in\N_0}} e^{t\egeval_i}\,\egefun_i(x)\scalarp{\egefun_i,f}. \end{equation}
In particular, for the indicator function $f = \mathds{1}_B$ of a measurable set $B \subseteq \mathcal{X}$, the conditional probability under the target dynamics is
\begin{equation}
\PP[X_t \in B \vert X_0=x] = \TO_t \mathds{1}_B(x) = \sum_{i=0}^m e^{\geval_i t}\,\gefun_i(x)\int_B \gefun_i(y)\,d\im(y).
\end{equation}
Since the target invariant measure $\im$ is unknown, we cannot compute this integral directly. Instead, using biased samples $\{x'_j\}_{j=1}^n \sim \bim$ and importance weights $\bw(x'_j)$, we approximate:
\begin{align}
\int_B \egefun_i(y)\,d\im(y) 
&= \int_B \egefun_i(y)\,\frac{d\im}{d\bim}(y)\,d\bim(y) \approx \frac{\sum_{j=1}^n \bw(x'_j)\mathds{1}_B(x'_j)\egefun_i(x'_j)}{\sum_{j=1}^n\bw(x'_j)} 
= \frac{\sum_{x'_j \in B}\bw(x'_j)\widehat{u}_i^\top\NNe(x'_j)}{\sum_{j=1}^n\bw(x'_j)}.
\end{align}
This yields the empirical conditional probability
\begin{equation}\label{eq:empirical_conditional_prob}
\widehat{\PP}[X_t \in B \vert X_0=x] = \sum_{i=0}^m e^{\egeval_i t}\,\widehat{u}_i^\top\NNe(x) \frac{\sum_{x'_j \in B}\bw(x'_j)\widehat{u}_i^\top\NNe(x'_j)}{\sum_{j=1}^n\bw(x'_j)}.
\end{equation}

\paragraph{Dirichlet forms and (fractional) Sobolev spaces.} 
Our approach relies on an energy-based formulation of the IG learning problem, grounded in the theory of Dirichlet forms. This framework provides a natural variational structure and allows us to exploit the functional-analytic properties of the underlying generator. We introduce the (shifted) energy functional defined for $f\in\Lii$ by
\begin{equation}\label{eq:energy_functional_app}
\energykernel_{\im}(f) := \lim_{t\to 0}\int_\X \shift^{-1}\bigl(f(f-\TO_t f)\bigr)\,\frac{d\im}{t},
\end{equation}
whenever the limit exists. The collection of such functions defines the domain $\mathrm{dom}(\energykernel)$ of the energy. The associated \emph{Dirichlet form} is the bilinear form obtained by polarization: for any $f,g \in \mathrm{dom}(\energykernel)$,
\begin{equation}\label{eq:dirichlet_bilinear_app}
\energykernel(f,g)
= - \int_\X f\, (\shift^{-1}\GE g)\, d\im
= \int_\X (-\shift^{-1}\GE f)\, g \, d\im.
\end{equation}
In particular, for every $f \in \mathrm{dom}(\energykernel)$, the energy satisfies $\energykernel(f) = \energykernel(f,f)$.

We define the \emph{energy norm} by
\begin{equation}\label{eq:energy_norm_app}
\|f\|_{\Wii}^2
:= \langle {\resolventinv}f, f\rangle_{\Liishort}
= \|f\|_{\Liishort}^2 + \frac{1}{\shift\beta}\|\nabla f\|_{\Liishort}^2,
\end{equation}
where ${\resolventinv}^{-1}$ denotes the resolvent associated with the generator $\GE$.
This norm coincides with the standard Sobolev norm and characterizes both the natural energy space and the domain of the infinitesimal generator through its associated Dirichlet form.

This definition extends naturally to (possibly higher) fractional orders. For any $p>0$, we introduce the $p$-th \emph{energy norm} by
\begin{equation}\label{eq:fractional_norm_app}
\|f\|_{\Wiip{p}}^2
:= \langle {\resolventinv}^{p}f,f \rangle_{\Liishort}.
\end{equation}
The resulting scale of norms interpolates between $\Liii$ regularity and higher-order smoothness induced by the generator. This construction leads to the definition of the \emph{fractional Sobolev spaces}: for any $p>0$,
\begin{equation}\label{eq:frac_sobolev_space_app}
\Wiip{p}
:= \big\{ f \in \Liii\,|\, \|f\|_{\Wiip{p}} < \infty \big\}.
\end{equation}
These spaces form the natural functional framework for our analysis, capturing the fractional regularity encoded by the generator and its associated energy structure.

\section{Transfer learning from biased dynamics}\label{app:biased_dynamics}

\paragraph{Biased Langevin dynamics.}
To accelerate exploration in the presence of metastability or rare events, practitioners introduce a bias potential $\bias:\mathbb{R}^d \to \mathbb{R}$ and simulate the modified dynamics under the perturbed potential
\begin{equation}\label{eq:biased_potential_app}
\bPT(x) := \PT(x) + \bias(x), \quad x \in \X.
\end{equation}
The resulting biased Langevin equation is
\begin{equation}\label{eq:biased_langevin_app}
dX_t' = -\nabla \bPT(X_t') dt + \sqrt{2\beta^{-1}} dW_t', \quad X_0' = x,
\end{equation}
where $(W_t')_{t\geq 0}$ is a standard $d$-dimensional Brownian motion. This biased process converges to the biased Boltzmann distribution
\begin{equation}\label{eq:biased_measure}
\bim(dx) = \frac{e^{-\beta \bPT(x)}}{\int_{\X} e^{-\beta \bPT(x)} dx} dx = \frac{e^{-\beta(\PT(x) + \bias(x))}}{\int_{\X} e^{-\beta(\PT(x) + \bias(x))} dx} dx.
\end{equation}
Let $\bGE$ denote the infinitesimal generator of the biased dynamics, with associated transfer operators $(\TO'_t)_{t\geq 0}$.

\paragraph{Generator perturbation.}
The effect of the bias on the generator is explicit and additive. From the general form \eqref{Eq: computation L}, the biased generator acts as
\begin{equation}\label{eq:biased_generator}
\bGE f = -\langle\nabla \bPT, \nabla f\rangle + \beta^{-1}\Delta f = -\langle\nabla \PT + \nabla \bias, \nabla f\rangle + \beta^{-1}\Delta f.
\end{equation}
Therefore, the perturbation decomposes as
\begin{equation}\label{eq:generator_perturbation}
\bGE = \GE - \langle\nabla \bias, \nabla(\cdot)\rangle.
\end{equation}
This linear dependence on the bias is a key structural property that enables principled debiasing at the generator level.

\paragraph{Transfer operator perturbation does not factorize.}
By contrast, the transfer operator perturbation is highly nonlinear. Recall that the transfer operators are related to the generators via $\TO_t = e^{t\GE}$ and $\TO'_t = e^{t\bGE}$. One might hope that
\begin{equation*}
\TO'_t = e^{t\bGE} = e^{t(\GE - \langle\nabla\bias, \nabla(\cdot)\rangle)} \stackrel{?}{=} e^{t\GE} \cdot e^{-t\langle\nabla\bias, \nabla(\cdot)\rangle} = \TO_t \cdot e^{-t\langle\nabla\bias, \nabla(\cdot)\rangle}.
\end{equation*}
However, since $\GE$ and $\langle\nabla\bias, \nabla(\cdot)\rangle$ do not commute in general (they do not share the same eigenstructure), the exponential does not factorize, so the above equality does not hold. 
%More precisely,
%\begin{equation}\label{eq:transfer_no_factorization}e^{t\bGE} = e^{t\GE - t\langle\nabla\bias, \nabla(\cdot)\rangle} \neq e^{t\GE} \cdot e^{-t\langle\nabla\bias, \nabla(\cdot)\rangle}.\end{equation}
This fundamental asymmetry between generator-level and transfer-operator-level perturbations explains why standard transfer operator learning methods--which rely on estimating covariance matrices of the form $\mathbb{E}[z(X_0)z(X_t)^\top]$--cannot be straightforwardly debiased through importance weighting. The temporal propagation encoded in the transfer operator obscures the simple additive structure available at the generator level.

\paragraph{Radon-Nikodym derivative and importance weighting.}
Comparing the unbiased invariant measure $\im$ with the biased measure $\bim$, we obtain
%\helene{same remark: $Z'$ instead of $\bar Z$ to be consistent with bias notation. \KLou{KL: I agree.}}
\begin{align*}\label{eq:radon_nikodym_derivation}
\frac{d\im}{d\bim}(x) & = \frac{e^{-\beta \PT(x)} / Z}{e^{-\beta \bPT(x)} / Z'}= \frac{e^{-\beta \PT(x)}}{e^{-\beta(\PT(x) + \bias(x))}} \cdot \frac{Z'}{Z} \\
& = e^{\beta \bias(x)} \cdot \frac{\int_{\X} e^{-\beta(\PT(x) + \bias(x))} dx}{\int_{\X} e^{-\beta \PT(x)} dx} = \frac{e^{\beta \bias(x)}}{\int_{\X} e^{\beta \bias(x)} d\bim(x)},
\end{align*}
where $Z = \int_{\X} e^{-\beta \PT(x)} dx$ and $Z' = \int_{\X} e^{-\beta \bPT(x)} dx$ are normalization constants. This Radon-Nikodym derivative makes the distribution shift explicit.

For any function $f \in L^1_{\im}(\X)$, we can reweight expectations under $\bim$ to recover expectations under $\im$:
\begin{equation}\label{eq:importance_reweighting}
\mathbb{E}_{x \sim \im}[f(x)] = \mathbb{E}_{x' \sim \bim}\left[\frac{d\im}{d\bim}(x') f(x')\right] = \mathbb{E}_{x' \sim \bim}\left[\frac{e^{\beta \bias(x')}}{\int_{\X} e^{\beta \bias(x)} d\bim(x)} f(x')\right].
\end{equation}
Defining the importance weight $\bw(x) := e^{\beta \bias(x)}$, this becomes
\begin{equation}\label{eq:importance_weight}
\mathbb{E}_{x \sim \im}[f(x)] = \frac{1}{\overline{\bw}} \mathbb{E}_{x' \sim \bim}[\bw(x') f(x')], \quad \text{where} \quad \overline{\bw} := \mathbb{E}_{x' \sim \bim}[\bw(x')].
\end{equation}
This importance weighting principle forms the foundation of our debiasing approach: quantities involving the unbiased generator $\GE$ and measure $\im$ can be estimated from biased trajectory data by appropriately reweighting with $\bw(x')$.

\section{Loss function derivation and spectral learning}\label{app:loss}

In this section we discuss in detail the proposed  \emph{representation learning} method to directly learn the dominant eigenspace (slow modes) of the generator through a spectral objective based on its resolvent. 

\paragraph{Spectral theory of the resolvent.}
We begin by establishing the spectral-theoretic foundation for our approach. Under standard confinement assumptions on the potential $\PT$, namely that $\PT(x) \geq 2\alpha|x|^2 \mathbf{1}_{{|x| \geq R}} - c$ for some $\alpha > 0$, $R > 0$, and $c \in \mathbb{R}$ (see Assumption \ref{eq:CP}), the generator $\GE$ admits a compact resolvent and a discrete spectrum. This follows from the compact embedding of the Dirichlet form domain into $\Lii$ \cite{davies1995spectral}. Consequently, $\GE$ is self-adjoint on $\Lii$ with eigenvalue decomposition
$$
\GE = \sum_{i \in \mathbb{N}_0} \geval_i\, \gefun_i \otimes \gefun_i,
$$
where $0 = \geval_0 > \geval_1 \geq \geval_2 \geq \cdots \to -\infty$ are the eigenvalues and $(\gefun_i)_{i \in \mathbb{N}_0}$ form an orthonormal basis of $\Lii$. The resolvent operator
$$
\resolventshort := (I - \shift^{-1}\GE)^{-1}, \quad \shift > 0,
$$
is compact and self-adjoint, sharing the same eigenfunctions as $\GE$ with eigenvalues
\begin{equation}\label{eq:resolvent_spectral_full}
\reval_i := \frac{\shift}{\shift - \geval_i} = \frac{1}{1 - \geval_i/\shift}.
\end{equation}
Since $\geval_i  \in  (-\infty,0]$ for all $i$, we have $\reval_i \geq 1$, with $\reval_0 = \shift$ (corresponding to the trivial constant eigenfunction) and $\reval_i \to 0$ as $i \to \infty$. The spectral decomposition of the resolvent is thus
$$
\resolventshort = \sum_{i \in \mathbb{N}_0} \reval_i\, \gefun_i \otimes \gefun_i.
$$
Crucially, the largest eigenvalues $\reval_0, \reval_1, \ldots, \reval_\trunc$ of $\resolventshort$ correspond to the smallest (in magnitude) eigenvalues $\geval_0, \geval_1, \ldots, \geval_\trunc$ of $\GE$. These are precisely the \emph{slow modes} that dominate the long-time dynamics and metastable structure of the system - the primary objects of interest in our learning framework.

\paragraph{Best rank-$(\trunc+1)$ approximation.}
We seek a rank-$(\trunc+1)$ approximation of $\resolventshort$ of the form
\begin{equation}\label{eq:rank_m_approx}
 \resolventshort\approx\NNE \EVop\NNE^*.
\end{equation}
Here, $\NNE = (\NNek_0,\dots,\NNek_\trunc)\colon \mathbb{R}^{\trunc+1} \to \Wii$ denotes the injection operator defined by $\NNE u = \sum_{j=0}^{\trunc} u_j \NNek_j$, where $(\NNek_j)_{j=0}^{\trunc}$ are neural network features parameterized by $\param$. The operator $\NNE^*\colon \Wii \to \mathbb{R}^{\trunc+1}$ is the adjoint, given by $(\NNE^* f)_j = \langle \resolventinv f, \NNek_j \rangle_{\Lii}$. Finally, $\EVop = \mathrm{diag}(e^{-(d^\param_0)^2}, \ldots, e^{-(d^\param_\trunc)^2})$ collects the trainable eigenvalue parameters.

By the Eckart-Young-Mirsky theorem, for a compact self-adjoint operator $A : \Wii \to \Wii$ with eigenvalues $(\sigma_i)_{i \in \mathbb{N}_0}$ and eigenvectors $(u_i)_{i \in \mathbb{N}_0}$, the best rank-$(\trunc+1)$ approximation is:
\begin{equation}\label{eq:eckart_young_op}
\min_{\mathrm{rank}(B) \leq \trunc+1} \|A - B\|_{\Wii \to \Wii}
= \left\|A - \sum_{i=0}^\trunc \sigma_i u_i \otimes u_i\right\|_{\Wii \to \Wii} = \sigma_{\trunc+1}.
\end{equation}
This holds for any unitarily invariant norm. For the Hilbert-Schmidt norm $\|A\|_{\mathrm{HS}(\Wii)}^2 := \sum_{i} \sigma_i^2$, the error is $\sum_{i > \trunc} \sigma_i^2$. While compactness and \eqref{eq:eckart_young_op} suffice theoretically, we use the Hilbert--Schmidt formulation because it admits a tractable reformulation via covariance matrices estimable from data.

\textbf{Application to the resolvent.}
Since the resolvent $\resolventshort$ is compact and self-adjoint in $\Wii$ under our confinement assumptions, the Eckart--Young--Mirsky theorem implies that its optimal rank-$(\trunc+1)$ approximation in operator norm is
\begin{equation}\label{eq:optimal_approx}
\resolventshort^{\trunc} = \sum_{i=0}^{\trunc} \reval_i \gefun_i \otimes \gefun_i,
\end{equation}
with approximation error $\|\resolventshort - \resolventshort^{\trunc}\|_{\Wii \to \Wii} = \reval_{\trunc+1}$. This same truncation is also optimal for the Hilbert-Schmidt norm when $\resolventshort$ has finite Hilbert-Schmidt norm in $\Wii$. Our neural network parameterization $\NNE \EVop \NNE^*$ is designed to recover this leading spectral subspace.

\subsection{Population loss derivation}

\paragraph{Hilbert-Schmidt formulation.}
Following \cite{Kostic2024ncp}, for compact operator $\resolventshort$ we formulate the population loss as:
\begin{align}
\exploss_{\regnn}(\param)
=& \overline{\bw}\lim_{k\to\infty}\Big(\|P_k\resolventshort{-}\NNE \EVop\NNE^*\|_{\mathrm{HS}(\Wii)}^2
-\|P_k\resolventshort\|_{\mathrm{HS}(\Wii)}^2\Big) + \regnn\,\overline{\bw}\,\|\NNE^*\NNE-\MX{I}\|_{F}^2, \label{eq:population_loss_hs_app}
\end{align}
where $\overline{\bw} = \EE_{x' \sim \bim}[e^{\beta\bias(x')}]$ accounts for importance weighting. The first term measures the approximation error; subtracting the constant $\|\resolventshort\|_{\mathrm{HS}(\Wii)}^2$ ensures the loss depends only on $\param$ and is minimized when $\NNE \EVop\NNE^*$ optimally approximates $\resolventshort$. The regularization term $\regnn\|\NNE^*\NNE-\MX{I}\|_{F}^2$ enforces orthonormality in the $\Wii$ metric.  

\paragraph{Key challenge.}
The loss \eqref{eq:population_loss_hs_app} cannot be directly computed because the resolvent is unknown and computing the Hilbert-Schmidt norm would require all eigenfunctions $\{\gefun_i\}_{i \in \mathbb{N}}$. The key insight is to reformulate using covariance matrices estimable from trajectory data.

\subsection{Reformulation via the $\Wii$ inner product}

\paragraph{The $\Wii$ inner product structure.}
Recall from Sec.  \ref{sec:problem} that the $\Wii$ inner product is defined by:
\begin{equation*}\label{eq:wii_inner_product}
\langle f, g \rangle_{\Wii} = \langle \resolventinv f, g \rangle_{\Lii} = \langle f, g \rangle_{\Lii} - \frac{1}{\shift }\langle f, \GE g \rangle_{\Lii}.
\end{equation*}
For Langevin dynamics with generator $\GE f = -\langle\nabla U, \nabla f\rangle + \beta^{-1}\Delta f$, integration by parts gives:
\begin{equation}\label{eq:wii_explicit}
\langle f, g \rangle_{\Wii} = \langle f, g \rangle_{\Lii} + \frac{1}{\shift \beta}\langle \nabla f, \nabla g \rangle_{\Lii} = \EE_{x \sim \im}\left[  f(x)g(x) + \frac{1}{\shift \beta}\nabla f(x)^\top \nabla g(x)\right].
\end{equation}

\paragraph{Computing $\NNE^*\NNE$ via the weighted feature matrix.}
The operator $\NNE^*\NNE: \mathbb{R}^{\trunc+1} \to \mathbb{R}^{\trunc+1}$ has matrix elements:
\begin{equation*}\label{eq:NNE_star_NNE}
(\NNE^*\NNE)_{ij} = \langle \resolventinv \NNek_i, \NNek_j \rangle_{\Lii} = \langle \NNek_i, \NNek_j \rangle_{\Wii}.
\end{equation*}
Using \eqref{eq:wii_explicit}:
\begin{equation}\label{eq:Wxt_population}
(\NNE^*\NNE)_{ij} = \EE_{x \sim \im}\left[ \NNek_i(x)\NNek_j(x) + \frac{1}{\shift\beta}\nabla \NNek_i(x)^\top \nabla \NNek_j(x)\right].
\end{equation}

\textbf{Importance reweighting.} Since we only have samples from $\bim$, not $\im$, we use the Radon-Nikodym derivative $\frac{d\im}{d\bim}(x) = e^{\beta\bias(x)}/\overline{\bw}$ where $\overline{\bw} = \int e^{\beta\bias(x)} d\bim(x)$. Define $\bw(x) = e^{\beta\bias(x)}$. Then:
\begin{equation}\label{eq:Wxt_biased}
\overline{\bw}(\NNE^*\NNE)_{ij} = \EE_{x' \sim \bim}\left[\bw(x')\left( \NNek_i(x')\NNek_j(x') + \frac{1}{\shift\beta}\nabla \NNek_i(x')^\top \nabla \NNek_j(x')\right)\right].
\end{equation}

\textbf{Matrix formulation.} Define the Jacobian matrix:
$$
J_\param(x') = [\partial_1\NNe(x') \vert \cdots \vert \partial_d\NNe(x')] \in \R^{(\trunc+1) \times d},
$$
and the weighted feature matrix:
\begin{equation*}\label{eq:w_features_detailed}
\wx(x') = \Big[\sqrt{\bw(x')}\NNe(x') \quad \sqrt{\tfrac{\bw(x')}{\shift\beta}}J_\param(x')\Big] \in \R^{(\trunc+1) \times (d+1)}.
\end{equation*}

Then:
\begin{align*}
\wx(x')\wx(x')^\top &= \begin{bmatrix} \sqrt{\bw(x')}\NNe(x') &\sqrt{\tfrac{\bw(x')}{\shift\beta}}J_\param(x') \end{bmatrix} \begin{bmatrix} \sqrt{\bw(x')}\NNe(x')^\top \\ \sqrt{\tfrac{\bw(x')}{\shift\beta}}J_\param(x')^\top \end{bmatrix}\nonumber\\
&= \bw(x')\NNe(x')\NNe(x')^\top + \frac{\bw(x')}{\shift\beta}J_\param(x')J_\param(x')^\top.
\label{eq:wx_expansion}
\end{align*}

The $(i,j)$-th element is:
\begin{align*}
[\wx(x')\wx(x')^\top]_{ij} &= \bw(x')\NNek_i(x')\NNek_j(x') + \frac{\bw(x')}{\shift\beta}\sum_{k=1}^d \partial_k\NNek_i(x')\partial_k\NNek_j(x')\\
&= \bw(x')\left(\NNek_i(x')\NNek_j(x') + \frac{1}{\shift\beta}\nabla\NNek_i(x')^\top\nabla\NNek_j(x')\right).
\end{align*}

Comparing with \eqref{eq:Wxt_biased} (noting the factor $\shift$):
\begin{equation}\label{eq:Wxt_final}
\Wxt := \overline{\bw}\NNE^*\NNE = \EE_{x' \sim \bim}[\wx(x')\wx(x')^\top].
\end{equation}

\paragraph{Computing $\NNE^*\resolventshort\NNE$.}
The operator $\NNE^*\resolventshort\NNE: \mathbb{R}^{\trunc+1} \to \mathbb{R}^{\trunc+1}$ has matrix elements:
\begin{equation*}\label{eq:resolvent_action}
(\NNE^*\resolventshort\NNE)_{ij} = \langle \resolventinv \NNek_i, \resolventshort \NNek_j \rangle_{\Lii}.
\end{equation*}

For any $f \in \Wii$, the action of $\NNE^*\resolventshort$ is, for any $j \in \{0, \dots, m\}$:
\begin{align}
(\NNE^*\resolventshort f)_j &= \langle \resolventinv \NNek_j, \resolventshort f \rangle_{\Lii}= \langle \NNek_j, f \rangle_{\Lii}= \EE_{x \sim \im}[\NNek_j(x) f(x)].\label{eq:resolvent_action_simplified}
\end{align}

Therefore:
\begin{equation*}\label{eq:Cxt_population}
(\NNE^*\resolventshort\NNE)_{ij} = \EE_{x \sim \im}[\NNek_i(x)\NNek_j(x)].
\end{equation*}

With importance reweighting \eqref{eq:Wxt_biased}:
\begin{equation}\label{eq:Cxt_final}
\Cxt := \overline{\bw}\,\NNE^*\resolventshort \NNE = \EE_{x' \sim \bim}[\bw(x')\NNe(x')\NNe(x')^\top].
\end{equation}

\subsection{Hilbert-Schmidt norm decomposition}

We derive the reformulation of the Hilbert-Schmidt norm in terms of $\Wxt$ and $\Cxt$. Expanding the squared norm:
\begin{align}
\|\resolventshort - \NNE\EVop\NNE^*\|_{\mathrm{HS}(\Wii)}^2 
&= \|\resolventshort\|_{\mathrm{HS}(\Wii)}^2 + \|\NNE\EVop\NNE^*\|_{\mathrm{HS}(\Wii)}^2 - 2\langle \resolventshort, \NNE\EVop\NNE^* \rangle_{\mathrm{HS}(\Wii)}.\label{eq:hs_expansion}
\end{align}

For the finite-rank operator $A = \NNE\EVop\NNE^*$ mapping to the $(\trunc+1)$-dimensional subspace $\mathrm{span}(\NNek_0, \ldots, \NNek_\trunc)$:
\begin{align}
\|\NNE\EVop\NNE^*\|_{\mathrm{HS}(\Wii)}^2 &= \sum_{j=0}^{\trunc} \|\NNE\EVop\NNE^*\NNek_j\|_{\Wii}^2
= \sum_{j=0}^{\trunc} \langle \EVop\NNE^*\NNek_j, \NNE^*\NNE\EVop\NNE^*\NNek_j \rangle_{\mathbb{R}^{\trunc+1}}\nonumber\\
&= \mathrm{tr}[(\EVop\NNE^*\NNE)^2]
= \overline{\bw}^{-2}\mathrm{tr}[(\EVop\Wxt)^2],\label{eq:approx_norm}
\end{align}
where we used the cyclic property of trace and $\Wxt = \overline{\bw}\NNE^*\NNE$.

Similarly, for the cross term:
\begin{align}
\langle \resolventshort, \NNE\EVop\NNE^* \rangle_{\mathrm{HS}(\Wii)} 
&= \sum_{i \in \mathbb{N}} \langle \resolventshort \gefun_i, \NNE\EVop\NNE^* \gefun_i \rangle_{\Wii}
= \sum_{i \in \mathbb{N}} \langle \gefun_i, \NNE\EVop\NNE^* \gefun_i \rangle_{\Lii}\nonumber\\
&= \mathrm{tr}[\NNE^*\resolventshort\NNE\EVop]
= \overline{\bw}^{-1}\mathrm{tr}[\Cxt\EVop].\label{eq:cross_term}
\end{align}

Substituting into \eqref{eq:hs_expansion}, multiplying by $\overline{\bw}$, and adding the orthonormality regularization $\regnn\|\NNE^*\NNE - \MX{I}\|_F^2 = \regnn\overline{\bw}^{-2}\mathrm{tr}[(\Wxt - \overline{\bw}\MX{I})^2]$ yields:
\begin{equation}\label{eq:population_loss_final_app}
\exploss_\regnn(\param)
=\mathrm{tr}\big[(\EVop\Wxt)^2
-2\,\overline{\bw}\EVop\,\Cxt{+}\regnn (\Wxt {-}\overline{\bw}\,\MX{I})^2\big],
\end{equation}
which is the loss presented in equation \eqref{eq:population_loss} of the main text.

\subsection{Empirical loss estimation}
\label{app:empirical_loss}

\paragraph{Challenge of unbiased estimation.}
The population loss \eqref{eq:population_loss_final_app} involves products of expectations: $(\EVop\Wxt)^2$, $\overline{\bw}\Cxt$, and $(\Wxt - \overline{\bw}\MX{I})^2$. Naively replacing expectations with sample means yields biased estimators.

\paragraph{Solution 1: Sample splitting.}
Split the batch into two independent subsets $\ebim_1, \ebim_2$ of size $\lceil b/2 \rceil$ and $\lfloor b/2 \rfloor$. Define empirical covariances $\EWx_\param^k = \EE_{x' \sim \ebim_k}[\wx(x')\wx(x')^\top]$, $\ECx_\param^k = \EE_{x' \sim \ebim_k}[\bw(x')\NNe(x')\NNe(x')^\top]$, and means $\widehat{\bw}^k = \EE_{x' \sim \ebim_k}[\bw(x')]$ for $k \in \{1,2\}$. For independent samples, $\EE[(\EVop\EWx_\param^1)(\EVop\EWx_\param^2)] = (\EVop\Wxt)^2$, and similarly for other terms. Averaging cross-terms for symmetry yields:
\begin{align}
\emploss_\regnn(\param)&=\mathrm{tr} \Big[ \frac{(\EWx_\param^1\EVop \EWx^2_\param\EVop {+} \EWx_\param^2\EVop \EWx^1_\param\EVop)}{2} 
{-}\widehat{\bw}^1\ECx_\param^2\EVop {-}\widehat{\bw}^2\ECx_\param^1\EVop { +} \regnn(\EWx_\param^1-\widehat{\bw}^1\MX{I})(\EWx_\param^2{-}\widehat{\bw}^2\MX{I}) \Big].\label{eq:emploss_balanced}
\end{align}

This simplicity allows one to exclude the trivial eigenpair $\geval_0{=}0$ and $\gefun_0{=}1$ by centering the features $\NNe(x')$ on each batch, as proposed in transfer operator setting in \cite{Kostic2024ncp}.

\textbf{Computational complexity:} 
Computing $\EWx_\param^k$ requires $\mathcal{O}((b/2)(\trunc+1)^2 d)$ operations for Jacobians and outer products. Computing $\ECx_\param^k$ requires $\mathcal{O}((b/2)(\trunc+1)^2)$ operations. Matrix multiplications require $\mathcal{O}((\trunc+1)^3)$ operations. Overall, the total computational cost is $\mathcal{O}\big(b(\trunc+1)^2 d + (\trunc+1)^3\big)$.

\paragraph{Solution 2: U-statistics (when centering features).}
Alternatively, constant features might be hard-coded with fixed weights, or when needed, c.f. \cite{PNAS}, learned from data. In these cases the products of expectations in the loss \eqref{eq:population_loss} can be estimated from iid samples $(x'_i)_{i\in[b]}$ from $\bim$ via U-statistics estimator:
\begin{align}
\emploss_\regnn(\param)&=\tfrac{1}{b(b{-}1)}\!\!\sum_{i\neq j\in[b]}\!\!\!\!\!\big[
\langle\EVop\wx(x_i'),\wx(x_j')\rangle_{F}^2{+}\regnn\bw(x_i')\bw(x_j'){-}2\bw(x_i')\bw(x_j')\|\EVop\NNe(x_j')\|^2\nonumber\\
&\quad{+}\regnn\langle\wx(x_i'),\wx(x_j')\rangle_F^2
{-} 2\regnn\,\bw(x_i')\|\wx(x_j')\|_F^2\big].\label{eq:ustats_loss_detailed}
\end{align}

\textbf{Computational complexity:} The double sum over $i \neq j$ requires $\mathcal{O}(b^2)$ evaluations. Each evaluation computes inner products $\langle\wx(x_i'),\wx(x_j')\rangle_F$ and norms at cost $\mathcal{O}((\trunc+1) d)$. Total: $\mathcal{O}(b^2 (\trunc+1) d)$.

Sample splitting is preferred for large batches ($b \gg \trunc$), while U-statistics are better when the latent dimension is large ($\trunc \gg b$).

\begin{table}[h!]
\centering
\begin{tabular}{lcc}
\toprule
Property & Sample Splitting & U-Statistics \\
\midrule
Computational complexity & $\mathcal{O}(b\trunc^2 d)$ & $\mathcal{O}(b^2\trunc d)$ \\
Handles trivial eigenpair & Yes & Requires centering \\
Batch efficiency & Better for large $b$ & Better for large $\trunc$ \\
\bottomrule
\end{tabular}
\caption{Comparison of empirical loss estimation strategies.}
\end{table}

\subsection{Proof of Theorem \ref{thm:main_dnn}}
\label{app:biased_learning}
\ThmMainDNN

\begin{proof}
We prove the theorem in four steps: (1-2) establish the lower bound using Eckart-Young-Mirsky and taking the limit $k\to\infty$, (3) characterize when equality holds, and (4) verify unbiasedness of the empirical estimator.

\textbf{Step 1: Lower bound via Eckart-Young-Mirsky.}
Let $P_k\colon\Wii\to\Wii$ be the orthogonal projector onto the span of the first $k$ eigenfunctions $\{\gefun_0, \ldots, \gefun_{k-1}\}$ of $\resolventshort$. Consider the truncated loss:
\begin{equation}\label{eq:truncated_loss}
\exploss^k(\param){=}\overline{\bw}\Big(\norm{P_k\resolventshort{-}\NNE \EVop\NNE^*}^2_{\mathrm{HS}(\Wii)} {-}\norm{P_k\resolventshort}^2_{\mathrm{HS}(\Wii)}\Big).
\end{equation}

By the Eckart-Young-Mirsky theorem, for any $k > \trunc$, the best approximation to $P_k\resolventshort$ in Hilbert--Schmidt norm is obtained by retaining its largest $\trunc+1$ eigenvalues, namely $P_{\trunc+1}\resolventshort$
whose spectral decomposition is $P_{\trunc+1}\resolventshort = \sum_{i=0}^{\trunc} \reval_i \gefun_i \otimes \gefun_i$, where $\reval_i = \shift/(\shift - \geval_i)$. Therefore:
\begin{align}
\exploss^k(\param) &\geq \overline{\bw}\left(\norm{P_{\trunc+1}\resolventshort - P_k\resolventshort}^2_{\mathrm{HS}(\Wii)} - \norm{P_k\resolventshort}^2_{\mathrm{HS}(\Wii)}\right)\nonumber\\
&= \overline{\bw}\left(\sum_{i=\trunc+1}^{k-1} \reval_i^2 - \sum_{i=0}^{k-1} \reval_i^2\right)= -\overline{\bw}\sum_{i=0}^{\trunc} \reval_i^2
= -\sum_{i=0}^{\trunc} \frac{\shift^2\overline{\bw}}{(\shift - \geval_i)^2}.\label{eq:lower_bound_k}
\end{align}

\textbf{Step 2: Taking the limit $k \to \infty$.}
Expanding the Hilbert-Schmidt norm in $\exploss^k$ via the trace:
\begin{align}
\exploss^k(\param) &= \overline{\bw}\Big(\norm{\NNE \EVop\NNE^*}^2_{\mathrm{HS}(\Wii)} - 2\langle \NNE \EVop\NNE^*, P_k\resolventshort \rangle_{\mathrm{HS}(\Wii)}\Big)\nonumber\\
&= \overline{\bw}\Big(\norm{\NNE \EVop\NNE^*}^2_{\mathrm{HS}(\Wii)} - 2\langle \NNE \EVop\NNE^*, \resolventshort \rangle_{\mathrm{HS}(\Wii)}\Big) + 2\overline{\bw}\langle \NNE \EVop\NNE^*, (I - P_k)\resolventshort \rangle_{\mathrm{HS}(\Wii)}\nonumber\\
&= \exploss(\param) + 2\overline{\bw}\langle \NNE \EVop\NNE^*, (I - P_k)\resolventshort \rangle_{\mathrm{HS}(\Wii)},\label{eq:loss_decomp}
\end{align}
where $\exploss(\param)$ is the loss without the subtracted constant term (i.e., the first two terms in \eqref{eq:population_loss_hs}).

By the H\"older inequality for the Hilbert-Schmidt inner product:
\begin{align}
\abs{\exploss(\param) - \exploss^k(\param)} &\leq 2\overline{\bw}\tr(\NNE \EVop\NNE^*)\norm{(I - P_k)\resolventshort}\nonumber\\
&= 2\tr(\Wxt \EVop)\reval_{k+1}.\label{eq:convergence_rate}
\end{align}
Since $\NNek_i \in \Wii$ for all $i \in \{0, \ldots, \trunc\}$ by assumption, we have $\tr(\Wxt\EVop) < \infty$. Moreover, $\geval_k \to -\infty$ as $k \to \infty$, and, hence
$$
\lim_{k \to \infty} \exploss^k(\param) = \exploss(\param).
$$

Taking the limit in \eqref{eq:lower_bound_k} and noting that $\exploss_{\regnn}(\param) \geq \exploss(\param)$ (since the orthonormality penalty is non-negative), we obtain:
\begin{equation}
\exploss_{\regnn}(\param) \geq -\sum_{i=0}^{\trunc} \frac{\shift^2\overline{\bw}}{(\shift - \geval_i)^2}.
\end{equation}

\textbf{Step 3: Characterizing equality.}
Equality in \eqref{eq:dnn_bound} requires:
\begin{enumerate}
\item $\exploss_{\regnn}(\param) = \exploss(\param)$, which implies the orthonormality penalty vanishes: $\|\NNE^*\NNE - I\|_F^2 = 0$, so $\NNE^*\NNE = I$.
\item $\exploss(\param) = -\sum_{i=0}^{\trunc} \frac{\shift^2\overline{\bw}}{(\shift - \geval_i)^2}$.
\end{enumerate}

By the Eckart-Young-Mirsky theorem, condition (2) holds if and only if $\NNE\EVop\NNE^*$ equals the optimal rank-$(\trunc+1)$ approximation $P_{\trunc+1}\resolventshort = \sum_{i=0}^{\trunc} \reval_i \gefun_i \otimes \gefun_i$. Combined with condition (1), this means:
$$
\mathrm{span}(\NNek_0, \ldots, \NNek_\trunc) = \mathrm{span}(\gefun_0, \ldots, \gefun_\trunc),
$$
or equivalently $\TP_\param = \TP_{\trunc+1}$.

If additionally $\regnn > 0$ and $\reval_{\trunc+1} < \reval_{\trunc}$ (i.e., $\geval_{\trunc+1} < \geval_{\trunc}$, ensuring a spectral gap), then the optimal rank-$(\trunc+1)$ approximation is unique. In this case, the orthonormality constraint $\NNE^*\NNE = I$ implies that $\{\NNek_i\}_{i=0}^{\trunc}$ form an orthonormal basis of $\mathrm{span}(\gefun_0, \ldots, \gefun_\trunc)$, and the eigenvalues satisfy $e^{-(d_i^\param)^2} = \reval_i$, i.e., $d_i^\param = \sqrt{-\log \reval_i}$, up to reordering.

\textbf{Step 4: Unbiasedness of empirical estimator.}
We now show that $\EE[\emploss_\regnn(\param)] = \overline{\bw}^2 \exploss_\regnn(\param)$. Recall the empirical loss \eqref{eq:emploss_2}:
\begin{align*}
\emploss_\regnn(\param) = \mathrm{tr}\Big[& (\EWx_\param^1\EVop \EWx^2_\param\EVop {+} \EWx_\param^2\EVop \EWx^1_\param\EVop)/2 
{-}\widehat{\bw}^1\ECx_\param^2\EVop{-}\widehat{\bw}^2\ECx_\param^1\EVop { +} \regnn(\EWx_\param^1-{\widehat{\bw}^1}\MX{I})(\EWx_\param^2{-}{\widehat{\bw}^2}\MX{I}) \Big].
\end{align*}

Since $\ebim_1$ and $\ebim_2$ are independent samples from $\bim$, we have:
\begin{align*}
\EE[\EWx_\param^1 \EVop \EWx_\param^2 \EVop]
&= \EE[\EWx_\param^1]\EVop \EE[\EWx_\param^2]\EVop
= (\EVop\Wxt)^2,\\
\EE[\widehat{\bw}^1 \ECx_\param^2]
&= \EE[\widehat{\bw}^1]\EE[\ECx_\param^2]
= \overline{\bw}\Cxt\\
\EE[\widehat{\bw}^2 \ECx_\param^1]
&= \EE[\widehat{\bw}^2]\EE[\ECx_\param^1]
= \overline{\bw}\,\Cxt,\\
\EE\big[(\EWx_\param^1-\widehat{\bw}^1 I)(\EWx_\param^2-\widehat{\bw}^2 I)\big]
&= (\EE[\EWx_\param^1]-\overline{\bw}I)(\EE[\EWx_\param^2]-\overline{\bw}I)
= (\Wxt-\overline{\bw}I)^2,
\end{align*}

Therefore:
\begin{align*}
\EE[\emploss_\regnn(\param)] &= \mathrm{tr}\big[(\EVop\Wxt)^2 - 2\overline{\bw}\EVop\Cxt + \regnn(\Wxt - \overline{\bw}I)^2\big]= \overline{\bw}^2 \exploss_\regnn(\param).
\end{align*}
\end{proof}

\subsection{Extension to general compact operators}

The formulation \eqref{eq:population_loss_hs_app} requires $\resolventshort$ to be Hilbert--Schmidt in $\Wii$ for the norm $\|\resolventshort\|_{\mathrm{HS}(\Wii)}^2$ to be finite. However, the equivalent form \eqref{eq:population_loss_final_app} expressed via covariance matrices $\Wxt$ and $\Cxt$ is well-defined whenever $\resolventshort$ is compact (guaranteed by confinement) and the neural network features $(\NNek_j)_{j=0}^\trunc$ and their gradients are in $\Wii$ (guaranteed by Assumption \ref{eq:BE}). This broader applicability is crucial since \textbf{compactness} (rather than Hilbert--Schmidt) is the natural condition ensured by standard confinement assumptions on the potential (Assumption \textbf{(CP)}). Indeed the covariance matrices $\Wxt$ and $\Cxt$ involve only $\Wii$-inner products between finitely many neural network features (always finite if features are in $\Wii$) and expectations under $\bim$ of bounded functions (bounded by Assumption \ref{eq:BE}). Crucially, no infinite sums or traces over the full spectrum of $\resolventshort$ appear in \eqref{eq:population_loss_final_app}, unlike in \eqref{eq:population_loss_hs_app}. Therefore, the covariance formulation extends naturally to all compact operators.

\subsection{Resolvent estimation in a given representation.}
%\massi{Should we elaborate this a bit more? maybe state it as a remark? Also comment on the difference from \cite{devergne2024biased} when the features are fixed}
Finally, we remark that our approach, defined by  equations \eqref{eq:empirical_gep} and \eqref{eq:empirical_operators},  differs fundamentally from regression-based methods that minimize the Hilbert-Schmidt approximation error of the resolvent (see in particular equation (16) in \cite{devergne2024biased}). In spirit our approach is closer to CCA approaches \cite{andrew2013deep,Kostic2024ncp,wu2020VAMP} based on $\Liishort$ geometry. However, a key novelty is that here we work in the Dirichlet energy space $\Wiishort$. This change of norm is instrumental in controlling gradient errors of the observables, and providing sharp finite sample guarantees for the estimation of the drift.

\section{Statistical bounds}\label{app:bounds}
\noindent

\subsection{Preliminaries: some concentration bounds}\label{app:concentration_bounds}
\paragraph{Useful concentration inequalities.}~All the statistical bounds we present will relay on two versions of Bernstein inequality. The first one is Pinelis and Sakhanenko inequality for random variables in a separable Hilbert space, see \citep[see e.g.][Proposition 2]{caponnetto2007}.

\begin{proposition}
\label{prop:con_ineq_ps}
Let $A_i$, $1 \leq i \leq n$ be i.i.d copies of a random variable $A$ in a separable Hilbert space with norm $\norm{\cdot}$. If there exist constants $\Lambda>0$ and $\sigma>0$ such that for every $m\geq2$, $\EE\norm{A}^m\leq \frac{1}{2} m! \Lambda^{m-2} \sigma^2$, 
then with probability at least $1-\delta$, 
\begin{equation}\label{eq:con_ineq_ps}
\left\|\frac{1}{n}\sum_{i=1}^n A_i - \EE A \right\|\leq \frac{4\sqrt{2}}{\sqrt{n}} \log\frac{2}{\delta} \sqrt{ \sigma^2 + \frac{\Lambda^2}{n} }.
\end{equation}
\end{proposition}

\noindent We also recall that a version of the non-commutative Bernstein inequality for bounded matrices. See for instance \cite{Tropp2015}.

\begin{proposition}[Matrix Bernstein]\label{prop:con_ineq_rectangular}
Let $Z_i$, $1 \leq i \leq n$ be independent random matrices with common dimension $d_1 \times d_2$. Assume that $\EE Z_i = 0$ and $\|Z_i\| \leq c$ for each index $i$. Define the sum
$$\bar{Z} = \frac{1}{n}\sum_{i=1}^n Z_i.$$
Let $\sigma^2_Z$ be the matrix variance statistic:
$$
\sigma^2_Z = \max\left\{\left\|\frac{1}{n}\sum_{i=1}^n \EE[Z_i Z_i^*]\right\|, \left\|\frac{1}{n}\sum_{i=1}^n \EE[Z_i^* Z_i]\right\|\right\}.$$
Then with probability at least $1-\delta$:
$$\|\bar{Z}\| \leq \sqrt{2\sigma^2_Z \frac{\mathcal{L}(\delta)}{n}} + \frac{c}{3} \frac{\mathcal{L}(\delta)}{n},$$
where $\mathcal{L}(\delta) = \log\frac{d_1 + d_2}{\delta}$.

Furthermore, the expectation satisfies:
$$\EE\|\bar{Z}\| \leq \sqrt{2\sigma^2_Z \frac{\log(d_1 + d_2)}{n}} + \frac{c}{3}\frac{\log(d_1 + d_2)}{n}.$$
\end{proposition}

\paragraph{Key concentration results.} Our proofs of the statistical guarantees rely on the following concentration results.

\begin{proposition}
\label{prop:bound_leftright}
Let $\esssup_{x'\sim\bim} |\partial_k \NNek_{j}(x')|\leq \cst<\infty$, for all $k=0,\ldots,d$ and $0\leq j\leq m$, with the convention $\partial_0 \NNe = \NNe$, let $\norm{V}_\infty<\infty$, and denote $c_\param:=2(m+1)\cst^2(1+d/(\shift\beta))e^{\beta\norm{\bias}_\infty}$. Given $\delta>0$, with probability in the i.i.d. draw of $(x_i')_{i=1}^n$ from $\bim$, it holds that
\begin{equation}\label{eq:bound_leftright}
\PP\left\{ \norm{ \Wxt^{-1/2} (\EWxt - \Wxt) \Wxt^{-1/2}  } \vee \norm{ \Wxt^{-1/2} (\ECxt - \Cxt) \Wxt^{-1/2}  } \leq \Errwhp \right\} 
\geq 1-\delta,    
\end{equation}
where 
\begin{equation}\label{eq:reg_eta1}
\Errwhp = \frac{2(c_\param{+}1)}{3n}\ln\frac{2m+2}{\delta} +\sqrt{\frac{2(c_\param{-}1)}{n}\ln\frac{2m+2}{\delta}}.
\end{equation}

Moreover, 
\begin{equation}\label{eq:bound_simetric}
\PP\left\{ \norm{ \Wxt^{1/2} \EWxt^{-1} \Wxt^{1/2}  } \leq \frac{1}{1-\Errwhp} \right\} 
\geq 1-\delta.
\end{equation}
\end{proposition}
\begin{proof}
The idea is to apply Proposition \ref{prop:con_ineq_rectangular} for operator $A(x') = [\Wxt^{-1/2} w_\param(x')] \otimes [\Wxt^{-1/2}w_\param(x')]$, where
\begin{equation}\label{eq:w_features}
w_\param(x')=\Big[\,\NNe(x')\,\Big\vert\, \tfrac{1}{\sqrt{\shift\beta}}\partial_1\NNe(x')\,\Big\vert\,\ldots\,\Big\vert\tfrac{1}{\sqrt{\shift\beta}}\partial_d\NNe(x')\,\Big]e^{\beta\bias(x')/2}\in\R^{(m+1)\times(d+1)}   
\end{equation} 
First, note that we have $\EE_{x'\sim\bim}[A(x')] = I$ and that under the assumptions on boundedness we have that $\norm{A}_{\infty}:=\esssup_{x'\sim\bim}\norm{A(x')}\leq 2(m+1)\cst^2(1+d/(\shift\beta))e^{\beta\norm{\bias}_\infty}=c_\param$, whenever $\norm{\Wxt-I}^2<1/2$.  Next, since $A(x')$ is symmetric positive semidefinite, we have $A(x')^2 \preceq \|A\|_\infty\,A(x')$.
Consequently,
\[
(A(x') - I)^2
= A(x')^2 - 2A(x') + I
\preceq (\|A\|_\infty - 2)A(x') + I.
\]
Then, using $\EE[A(x')] = I$ yields
\[
\EE\big[(A(x') - I)^2\big]
\preceq (\|A\|_\infty - 1)\,I.
\]
Hence, the matrix variance statistic satisfies
\[
\sigma_Z^2
= \|\EE[Z_i^2]\|
\le \|A\|_\infty - 1.
\]
Applying then the Proposition~\ref{prop:con_ineq_rectangular} with $d_1=d_2=\trunc+1$, we obtain that, with probability at least $1-\delta$,
\[
\left\|
\frac1n\sum_{i=1}^n A(x'_i) - I
\right\|
\le
\sqrt{
\frac{2(c_\param - 1)}{n}
\log\Big(\frac{2(\trunc+1)}{\delta}\Big)}
+
\frac{c_\param + 1}{3n}\,
\log\Big(\frac{2(\trunc+1)}{\delta}\Big).
\]

To obtain the same for $\ECxt-\Cxt$, it suffices to observe that $\NNe(\cdot)\otimes \NNe(\cdot)\preceq w_\param(\cdot)\otimes w_\param(\cdot)$, and, hence \eqref{eq:bound_leftright} follows from union bound. 

To complete the proof, observe that
\[
\norm{I - \Wxt^{-1/2} \EWxt \Wxt^{-1/2}} = \norm{\Wxt^{-1/2} (\Wxt-\EWxt) \Wxt^{-1/2}}\leq \Errwhp,
\]
and, hence for $\Errwhp$ smaller than one we obtain 
\[
\norm{ \Wxt^{1/2} \EWxt^{-1} \Wxt^{1/2}} =  \norm{(\Wxt^{-1/2} \EWxt \Wxt^{-1/2})^{-1}}\leq \frac{1}{1-\norm{I - \Wxt^{-1/2} \EWxt \Wxt^{-1/2}}}.
\]
\end{proof}

\begin{proposition}
\label{prop:bound_xi}
Given $x'\in\R^d$, let $\xi(x') = \big[\NNe(x')(x')^\top +\tfrac{1}{\shift\beta}J_\param(x')\big]e^{\beta\bias(x')}\in\R^{(\truncbis)\times d}$, where the Jacobian of features is denoted by $J_\param(x')=[\partial_1 \NNe(x')\,\vert\,\ldots\,\vert \partial_d\NNe(x')]$. Let $\esssup_{x'\sim\bim} |\partial_k \NNek_{j}(x')|\leq \cst<\infty$, for all $k\in\{0,\ldots,d\}$ and $j\in\{0,\ldots, m\}$, with the convention $\partial_0 \NNe = \NNe$, let $\norm{V}_\infty<\infty$ and $\norm{\Wxt/\overline{\bw}-I}<1/2$. 
%\helene{$\norm{\Wxt/\overline{\bw}-I}<1/2$??}
Assume that the potential $U'$ is sufficiently confining such that $x'\sim\bim$ is $K$-sub-Gaussian with $K\in (0,\infty)$. Given $\delta>0$, with probability in the i.i.d. draw of $(x_i')_{i=1}^n$ from $\bim$, it holds that
\begin{equation}\label{eq:bound_xi}
\PP\left\{ \norm{ \Wxt^{-1/2} \Big[\tfrac{1}{n}\sum_{i=1}^n\xi(x'_i) - \EE_{x'\sim\bim}\xi(x')\Big] }_{\R^{\trunc+1}\to\R^d} \leq  \ErrwhpXi / \sqrt{\overline{\bw}}\right\} 
\geq 1-\delta,    
\end{equation}
where 
\begin{align}
\label{eq:reg_eta_xi}
 \ErrwhpXi   
 \nonumber
 &{=}C \cst e^{\beta \norm{V}_\infty}     \max\Big\{   \sqrt{\frac{\min(m+1,d)}{\shift^2\beta^2}}  \bigg(  \sqrt{ \frac{\log((m + d+1)/\delta)}{n} }  +  \frac{\log((m + d+1)/\delta)}{n}  \bigg)  , \\
    &\hspace{7cm} \sqrt{\frac{m+1}{n}}    \sqrt{ \EE[\norm{x'}_2^{2}] + \frac{K d}{n} }\, \log\frac{2}{\delta}  \Big\},
\end{align}
for a large enough absolute constant $C>0$.
\end{proposition}

\begin{proof}
Recalling $\xi(x') = \big[\NNe(x')(x')^\top +\tfrac{1}{\shift\beta}J_\param(x')\big]e^{\beta\bias(x')}\in\R^{(\truncbis)\times d}$, where the Jacobian of features is denoted by $J_\param(x')=[\partial_1 \NNe(x')\,\vert\,\ldots\,\vert \partial_d\NNe(x')]$. For brevity we set $\norm{\cdot } = \norm{ \cdot }_{\R^{\trunc+1}\to\R^d}$ and $\norm{\cdot}_2 = \norm{\cdot}_{l_2(\R^{\trunc+1})}$. We have
$$
\Wxt^{-1/2}\xi(x') = \underbrace{e^{\beta\bias(x')}\Wxt^{-1/2}\NNe(x')(x')^\top}_{=:\bar{A}}  +\underbrace{e^{\beta\bias(x')} \tfrac{1}{\shift\beta} \Wxt^{-1/2} J_\param(x')}_{=:\tilde{A}}.
$$

We note that for any $t_1,t_2>0$ with $t = t_1 + t_2$, we have
$$
\PP\left( \norm{\sum_i A_i} >t \right)  \leq \PP\left( \norm{\sum_i\bar{A}_i} >t_1 \right) + \PP\left( \norm{\sum_i\tilde{A}_i} >t_2 \right).
$$
We will control $\norm{\sum_i\bar{A}_i} $ and $\norm{\sum_i\tilde{A}_i}$ separately.

\noindent
\textbf{Uniform bound on $\norm{\sum_i\tilde{A}_i}$.}~
Note that under the boundedness assumption on the entries of $\NNe(x')$ and its partial derivatives we have 
$\esssup_{x'\sim\bim}\norm{J_\param(x') } \leq \cst \sqrt{\min(m+1,d)}$. Hence we get the following $\bim$-almost surely upper bound on the Frobenius norm of $\tilde{A}$:
\begin{align}
\label{eq:opboundA2}
    \norm{\tilde{A}} &=  \tfrac{e^{\beta\bias(x')}}{\shift\beta}\norm{\Wxt^{-1/2}J_\param(x')}\leq \tfrac{e^{\beta\bias(x')}}{\shift\beta}\norm{\Wxt^{-1/2}}\norm{J_\param(x')}\leq\cst \sqrt{\frac{2\min(m+1,d)}{\overline{\bw}}}e^{\beta \norm{V}_\infty},
\end{align}
where we have used that $\norm{\Wxt/\overline{\bw}-I}<1/2$, implying that $\norm{\Wxt^{-1/2}} \leq \sqrt 2/\sqrt{\overline{\bw}}$.
Indeed, using that $\|A^{-1/2}\|=1/\sqrt{\lambda_{\min}(A)}\leq 1/\sqrt{1-\|I-A\|}$ with $A=\Wxt/\overline{\bw}$, we get
\begin{equation*}
\|\Wxt^{-1/2}\|
\le\frac{1}{\sqrt{\overline{\bw}\left(1-\left\|\Wxt/\overline{\bw}-I\right\|\right)}}\leq \frac{1}{\sqrt{\overline{\bw}\left(1-1/2\right)}}=\sqrt 2/\sqrt{\overline{\bw}}.
\end{equation*}
\color{black}

We define the random matrix 
$$Z_i :=  \tfrac{e^{\beta\bias(x_i')} }{\shift\beta} \Wxt^{-1/2} J_\param(x_i') - \EE\left[ \tfrac{e^{\beta\bias(x')} }{\shift\beta} \Wxt^{-1/2} J_\param(x')\right] \in \R^{(m+1) \times d}.
$$
Then $\EE[Z_i] = 0$ and in view of \eqref{eq:opboundA2}, we have $\norm{Z_i} \leq  2 \sqrt{2} \cst \sqrt{\min(m+1,d)/\overline{\bw}} \,\frac{e^{\beta \norm{V}_\infty}}{\shift\beta} =: c$, $\bim$-a.s.

Since $Z_i$ is not square, we need to use the non-commutative Bernstein matrix concentration inequality for rectangular matrices.

\noindent
\textbf{Left covariance operator.}
We have $Z_i \in \R^{(m+1) \times d}$, so $Z_i Z_i^* \in \R^{(m+1) \times (m+1)}$. Also
\begin{multline*}
    \EE[Z_i Z_i^*]  = \EE\left[   \tfrac{e^{2\beta\bias(x_i')} }{\shift^2\beta^2} \Wxt^{-1/2} J_\param(x_i')  J_\param(x_i')^*  \Wxt^{-1/2}\right]\\ - \EE\left[   \tfrac{e^{\beta\bias(x_i')} }{\shift\beta} \Wxt^{-1/2} J_\param(x_i') \right] \EE\left[   \tfrac{e^{\beta\bias(x_i')} }{\shift\beta}  J_\param(x_i')^*  \Wxt^{-1/2}\right].
\end{multline*}
Hence
\begin{align*}
\EE[Z_i Z_i^*] &\preceq \EE\left[\frac{e^{2\beta\bias(x')}}{\shift^2\beta^2} \Wxt^{-1/2} J_\param(x') J_\param(x')^* (\Wxt^{-1/2})^*\right] \\
&=\frac{e^{2\beta \norm{V}_\infty}}{\shift^2\beta^2} \Wxt^{-1/2} \EE[J_\param(x') J_\param(x')^*] \Wxt^{-1/2}\preceq  \frac{e^{2\beta \norm{V}_\infty}}{\shift\beta} I_{m+1}=:V_L,
\end{align*}
where we used $\frac{1}{\shift\beta}\EE_{x'\in\bim}[J_\param(x') J_\param(x')^*]= \Wxt - \Cxt \preceq \Wxt$ .

\noindent
\textbf{Right covariance operator.}
Similarly, $Z_i^* Z_i \in \R^{d \times d}$ and:
\begin{align}
 \EE[Z_i^* Z_i] &\preceq \EE\left[J_\param(x')^* (\Wxt^{-1/2})^* \frac{e^{2\beta\bias(x')}}{\shift^2\beta^2} \Wxt^{-1/2} J_\param(x')\right] \\
 &= \frac{e^{2\beta \norm{V}_\infty}}{\shift^2\beta^2} \EE[J_\param(x')^* \Wxt^{-1} J_\param(x')]  :=V_R 
% \notag \\
% &\preceq 2 \frac{e^{2\beta \norm{V}_\infty}}{\shift^2\beta^2} \cdot \EE[J_\param(x')^* J_\param(x')] 
% %\preceq 2 \frac{e^{2\beta \norm{V}_\infty}}{\shift^2\beta^2} \cst^2 m\, d \cdot I_d
% :=V_R
.
\end{align}
Note that 
\begin{align*}
\norm{V_L} &= \frac{e^{2\beta \norm{V}_\infty}}{\shift^2\beta^2}, \quad \norm{V_R} \leq 2 \frac{e^{2\beta \norm{V}_\infty}}{\shift^2\beta^2\overline{\bw}} \cst^2 \min(m+1,d)%, \quad\tr(V_L) = \tr(V_R) = \frac{e^{2\beta \norm{V}_\infty}}{\shift^2\beta^2} m
.
\end{align*}
Define $\sigma^2_Z := \max(\norm{V_L} ,\norm{V_R} )$. Since $\tau\geq 1 $, we have $\norm{V_L} \leq \norm{V_R} $, and consequently:
\begin{align}
\sigma^2_Z &= 2 \frac{e^{2\beta \norm{V}_\infty}}{\shift^2\beta^2\overline{\bw}} \cst^2 \min(m+1,d).%, \quad\tr(V_L) = \tr(V_R) = \frac{e^{2\beta \norm{V}_\infty}}{\shift^2\beta^2} m
\end{align}

\noindent
\textbf{Final concentration bound.}
With probability at least $1-\delta$:
\begin{multline}
\left\|\frac{1}{n}\sum_{i=1}^n Z_i\right\| \leq \sqrt{2\sigma^2_Z \frac{\mathcal{L}(\delta)}{n}} + \frac{c}{3} \frac{\mathcal{L}(\delta)}{n} \\
\leq  2 \sqrt{2} \cst \sqrt{\min(m+1,d)} \frac{e^{\beta \norm{V}_\infty}}{\overline{\bw}\shift\beta} \Big(  \sqrt{ \frac{\log((m + d+1)\delta^{-1})}{n} }  +  \frac{\log((m + d +1)\delta^{-1})}{3 n}  \Big)
\end{multline}

\noindent
\textbf{Uniform bound on $\norm{\sum_i\bar{A}_i}$.}~Since $\bar{A}$ is of rank one, we have $\norm{\bar{A}}_F = \norm{\bar{A}}$ and consequently:
\begin{align*}
   \norm{\bar{A}}_F = \norm{\bar{A}} &= \norm{\Wxt^{-1/2}\NNe(x')(x')^\top  e^{\beta\bias(x')}}\\
    &\leq \norm{\Wxt^{-1/2}\NNe(x')(x')^\top }e^{\beta\bias(x')}\\
    &\leq e^{\beta\bias(x')} \norm{\Wxt^{-1/2}\NNe(x')}_2\norm{x'}_2 \\
    &\leq \frac{\sqrt{2}}{\overline{\bw}}e^{\beta\bias(x')} \norm{\NNe(x')}_2\norm{x'}_2,
\end{align*}
where we have used that $\norm{\Wxt^{-1/2}\NNe(x')}_2\leq \norm{\Wxt^{-1/2}} \norm{\NNe(x')}_2$ and $\norm{\Wxt^{-1/2}} \leq \sqrt{2}/\overline{\bw}$.
Note that under the boundedness assumption on the entries of $\NNe(x')$, we have $\esssup_{x'\sim\bim}\norm{\NNe(x')}_2 \leq \cst \sqrt{m+1}$. Hence we get the following $\bim$-almost surely bound:
\begin{align}
\label{eq:bound_A1}
    \norm{\bar{A}}_F
    &\leq \cst\sqrt{\frac{2(m+1)}{\overline{\bw}}}e^{\beta \norm{V}_\infty}  \norm{x'}_2 .
\end{align}

\noindent
\textbf{Moment condition on $\norm{\bar{A}}$.}~
We need to verify that there exist constants $\Lambda > 0$ and $\sigma^2 > 0$ such that for every $p \geq 2$:
\begin{equation}
\label{eq:moment_cond}
\EE[\norm{\bar{A}}_F^p] \leq \frac{1}{2} p! \Lambda^{p-2} \sigma^2.
\end{equation}
From \eqref{eq:bound_A1}, we have
\begin{align}
\EE[\norm{\bar{A}}_F^p] &\leq \EE\left[\left(\cst\sqrt{m+1}e^{\beta \norm{V}_\infty}  \norm{x'}_2  \right)^p\right] = (\cst\sqrt{m+1}e^{\beta \norm{V}_\infty})^p \EE[\norm{x'}_2^p] .
\end{align}

Since $x' \sim \bim$ is $K$-sub-Gaussian, we have for any unit vector $u \in \R^d$, 
$\EE[e^{t\langle u, x' \rangle}] \leq e^{K^2 t^2/2}$, for all $t \in \R$. This implies that $\norm{x'}_2$ is $K\sqrt{d}$-sub-Gaussian. 

Consequently, 
for the moments of $\norm{x'}_2$, we use Cauchy-Schwarz for $p \geq 2$:
\begin{equation}
\label{eq:moment_x_subgauss}
\EE[\norm{x'}_2^p] = \EE[\norm{x'}_2^{2} \cdot \norm{x'}_2^{p-2}] \leq (\EE[\norm{x'}_2^{4}])^{1/2}(\EE[\norm{x'}_2^{2(p-2)}])^{1/2}.
\end{equation}

Since $\norm{x'}_2$ is $K\sqrt{d}$-sub-Gaussian, by equivalence of moments for sub-Gaussian random variables, we have:
$
(\EE[\norm{x'}_2^{4}])^{1/2} \lesssim \EE[\norm{x'}_2^{2}]
$, where the implicit constant is absolute. For higher-order moments with $p \geq 2$, the standard sub-Gaussian moment bound gives for some absolute constant $C > 0$:
\begin{equation}
\EE[\norm{x'}_2^p] \leq C (K\sqrt{d})^p \sqrt{p!}.
\end{equation}

Combining these, there exists an absolute constant $C' > 0$ such that for $p \geq 2$:
\begin{equation}
\EE[\norm{x'}_2^p] \leq C' \EE[\norm{x'}_2^{2}] \cdot (K\sqrt{d})^{p-2} \sqrt{(p-2)!}.
\end{equation}

For the moment condition \eqref{eq:moment_cond}, we substitute:
\begin{align}
\EE[\norm{\bar{A}}_F^p] &\leq  C \Big(\cst\sqrt{\frac{m+1}{\overline{\bw}}}e^{\beta \norm{V}_\infty}\Big)^p \EE[\norm{x'}_2^{2}] \cdot (K\sqrt{d})^{p-2} \sqrt{(p-2)!}.
\end{align}

Hence we observe that the moment condition is satisfied with 
\begin{align}
\Lambda &= C' \cst e^{\beta \norm{V}_\infty} \sqrt{\frac{K(m+1)d}{\overline{\bw}}},\quad
\sigma^2 = C' \cst^2 e^{2\beta \norm{V}_\infty} \frac{(m +1)}{\overline{\bw}} \EE[\norm{x'}_2^{2}],
\end{align}
for an appropriate absolute constant $C' > 0$.

\noindent
\textbf{Step 3: Application of concentration inequality.}
Applying Proposition \ref{prop:con_ineq_ps} with the verified moment condition, we obtain with probability at least $1-\delta$:
\begin{equation}
\left\|\frac{1}{n}\sum_{i=1}^n \bar{A}_i - \EE [\bar{A}] \right\|_F\leq \frac{4\sqrt{2}}{\sqrt{n}} \log\frac{2}{\delta} \sqrt{ \sigma^2 + \frac{\Lambda^2}{n} }.
\end{equation}

Therefore, there exists an absolute constant $C>0$, such that w.p.a.l. $1-\delta$,
\begin{equation}
\left\|\frac{1}{n}\sum_i\bar A_i  - \EE[\bar A]\right\|_F \leq \ErrwhpXi / \sqrt{{\overline{\bw}}} := C \; \cst e^{\beta \norm{V}_\infty}\, \sqrt{\frac{m+1}{{\overline{\bw}}n}}    \sqrt{ \EE[\norm{x'}_2^{2}] + \frac{K d}{n} }\, \log\frac{2}{\delta}.
\end{equation}

\end{proof}

\paragraph{Estimation of $\overline{\bw}$.}~ We derive an elementary concentration inequality for the estimation of the renormalizing constant $\overline{\bw}$.

\begin{lemma}[Relative error bound]
\label{lem:weight_concentration_bernstein}
Let $\bw(x) = e^{\beta \cdot \bias(x)}$ where $\|\bias\|_\infty < \infty$. Let $x_1, \ldots, x_n$ be i.i.d. samples from distribution $\mu$ on $\mathcal{X}$, and define:
$$
\overline{\bw} = \EE_{x \sim \mu}[\bw(x)],\qquad \widehat{\bw} = \frac{1}{n}\sum_{i=1}^n \bw(x_i).
$$
Let $\sigma^2 = \text{Var}(\bw(x))$. Then with probability at least $1-\delta$:
$$
\left|\frac{\widehat{\bw} - \overline{\bw}}{ \overline{\bw}}\right| \leq \frac{\sigma}{ \overline{\bw}}\sqrt{\frac{2\log(2/\delta)}{n}} + \frac{e^{\beta\|\bias\|_\infty}\log(2/\delta)}{3n  \overline{\bw}} = : \Errwhpv.
$$
\end{lemma}

\begin{proof}
Let $\bw_i = \bw(x_i) - \overline{\bw}$. Then $\EE[\bw_i] = 0$, $|W_i| \leq e^{\beta \|\bias\|_\infty}$ a.s., and $\text{Var}(W_i) = \sigma^2$. By Bernstein's inequality:
$$\mathbb{P}\left\{\left|\sum_{i=1}^n \bw_i\right| \geq t\right\} \leq 2\exp\left(-\frac{t^2/2}{\sigma^2 n + e^{\beta \|\bias\|_\infty}t/3}\right).$$
Setting the right side equal to $\delta$ and solving for $t$ yields 
\begin{equation*}
t = \sigma\sqrt{2n\log(2/\delta)} + \frac{e^{\beta \|\bias\|_\infty}\log(2/\delta)}{3}.
\end{equation*}
Dividing by $n \overline{\bw}$ gives the result.
\end{proof}

\subsection{Assumptions \ref{eq:BE}, \ref{eq:CP} and \ref{eq:RA}: discussion}\label{app:bounds_ass}
Before presenting the proofs for resolvent (App. \ref{app:resolvent_bounds}), generator (App. \ref{app:generator_bounds}), and gradient estimation (App. \ref{app:gradient_bounds}), we briefly discuss the assumptions used throughout the analysis, namely \ref{eq:BE}, \ref{eq:CP} and \ref{eq:RA}.\vspace{3pt}\\
\noindent
\textbf{Bounded embeddings: \ref{eq:BE}}
\textit{There exist $\cst \geq 1$ such that $\esssup_{x'\sim\bim} |\partial_k \NNek_{j}(x')|\leq \cst<\infty$, $\forall k=0,\ldots,d$ and $\forall j= 0, \ldots, \trunc$, with the convention $\partial_0 \NNe = \NNe$}.\\
Assumption \textbf{(BE)} requires then uniform boundedness of the embedding functions $(\NNek_j)_{j=0}^{\trunc}$ and their first-order derivatives under the biased sampling distribution $\bim$. This condition ensures finite variance of empirical spectral and gradient estimators and allows uniform concentration over basis indices and coordinates. It is satisfied by many standard choices of embeddings, including truncated trigonometric bases, bounded neural feature maps, or other smooth finite-dimensional representations. \\
When the dynamics evolve on a \textbf{compact manifold}, such as the sphere $\mathbb S^{d-1}$, this assumption can be substantially relaxed. Indeed, compactness implies that any continuous embedding is automatically bounded, and any $C^1$ embedding has bounded (intrinsic) first derivatives. Consequently, for spherical Langevin dynamics, a \ref{eq:BE}-type condition holds automatically for any fixed truncation level $\trunc$, provided the embeddings are $C^1$ on the manifold.\vspace{5pt}\\
\noindent
\textbf{Confining potential:} \ref{eq:CP} 
\textit{The potential $\PT$ satisfies $\PT(x)\ge 2\alpha\|x\|^{2}\mathbf{1}_{\{\|x\|\ge R\}}-c$, $\forall x$, for some $\alpha>0$, $R>0$, $c\in\R$, and the bias obeys $\|\bias\|_\infty<\infty$}.\\
Assumption \textbf{(CP)} ensures stability and the existence of an invariant distribution for Langevin dynamics in Euclidean space. The quadratic lower bound on $\PT$ outside a compact set guarantees exponential tail decay of the invariant measure, implying finite moments of all orders. This integrability is essential when working with \textbf{biased samples}, as it ensures that reweighting by a bounded bias does not lead to ill-conditioned estimators. Together with boundedness of the bias $\bias$, this assumption yields equivalence between the biased and unbiased measures, ergodicity of the dynamics, and consistency of empirical averages. Furthermore, crucially, \ref{eq:CP} implies that  $x\sim\im$ and $x'\sim\bim$ are $K$-sub-Gaussian with $K^2\lesssim(\beta \alpha)^{-1}$, allowing one to   derive concentration bounds for the gradient estimation.\\
In contrast, spherical Langevin dynamics evolve on a compact state space. Compactness prevents escape to infinity and ensures the existence of an invariant distribution under much weaker conditions. In this setting, no confining growth of $\PT$ is required; it suffices that the target density is well-defined with respect to the surface measure and that the drift is sufficiently regular.
\vspace{5pt}\\
\noindent
\textbf{Gradient regularity:} \ref{eq:RA} \textit{$\exists p\geq 1$ such that $\forall k\in[d]$, $\norm{\partial_k \PT}_{\Wiip{p}}<\infty$, that is $\norm{\nabla \PT}_{\Wiip{p}}<\infty$.}\\
Assumption \textbf{(RA)} requires the gradient of the potential to belong to a Sobolev space $\Wiip{p}$ for some $p\ge 1$. This ensures that $\nabla \PT$ is well-defined in a weak sense and has finite $p$-th moments under the invariant distribution. The condition is used to control the bias–variance trade-off in gradient estimation and to establish stability and convergence of the resulting gradient-based dynamics. Importantly, it does not require global Lipschitz continuity or pointwise smoothness, and is therefore compatible with nonconvex or moderately rough potentials.

\subsection{Resolvent estimation}\label{app:resolvent_bounds}

As a reminder, $\ECxt$ and $\EWxt$ are the empirical counterparts of the  covariance matrices associated to the parametrized features $\Cxt= {\overline{\bw}} \NNop^*\resolventshort\NNop$ and $\Wxt={\overline{\bw}} \NNE^*\NNE$, where $\NNop:\R^{{m+1}}\to\Wiix$ is the (parametrized) injection operators. Further we note that for $\error_m(\theta)<1$, $\Wxt$ is invertible implying that   $\sqrt{\overline{\bw}}\NNE\Wxt^{-1/2}$ is a partial unitary operator, that is  $[\sqrt{\overline{\bw}}\NNE\Wxt^{-1/2}]^*\sqrt{\overline{\bw}}\NNE\Wxt^{-1/2}=\MX{I}$, and, hence, partial isometry.

Let $\Nt = \Wxt^{-1}\Cxt\Wxt^{-1}$ and $\Qt=\Wxt^{-1/2}\Cxt\Wxt^{1/2}$, and further denote $\w{M}_\param=\Wxt^{1/2}\EWxt^{-1}\Wxt^{1/2}$ and $\ENt=\EWxt^{-1}\ECxt\EWxt^{-1}$. Recall that $\TPt= {\overline{\bw}} \NNE\Wxt^{-1}\NNE^*$ is the orthogonal projector on $\RKHS_\param:=\RKHS_{\NNE}$ and $\RKHS_{\NNE}$ is the space spanned by the set of basis functions $\NNek_j:\mathcal X\to\R$, $j\in{\{0,\dots,\trunc\}}$. Let $P_\trunc$ be the spectral projector associated with the set ${\{0,\dots,\trunc\}}$.
Using that
\begin{align*}
\norm{\TPt^\bot\TP_\trunc \resolventshort}_{\Wii\to\Wii} \vee \norm{\TPt\TP_\trunc^\bot\resolventshort}_{\Wii\to\Wii} 
\le
(1 \vee \reval_{\truncbis})\norm{\TPt-\TP_\trunc}_{\Wii\to\Wii} \leq \ErrSpc,
\end{align*}
since $\reval_{j} \in (0,1)$ for any $j\geq 0$,
we get
\begin{align}
\label{Eq:resolvent_bound_NN_2}
\norm{\resolventshort^\trunc- \overline{\bw} \NNE \Nt\NNE^*}_{\Wii\to\Wii}
\nonumber
&=\norm{(\TPt+\TPt^\bot)(\resolventshort^\trunc-\overline{\bw} \NNE \Nt\NNE^*)(\TPt+\TPt^\bot)}_{\Wii\to\Wii}\\
\nonumber
&\le  3\norm{\TPt^\bot\resolventshort^\trunc}_{\Wii\to\Wii} +\norm{\TPt\resolventshort^\trunc\TPt-\overline{\bw}\NNE\Nt\NNE^*}_{\Wii\to\Wii}
\\
\nonumber
&\leq 3 \norm{\TPt^\bot\resolventshort^\trunc}_{\Wii\to\Wii}+ \norm{\TPt \TP_\trunc^\bot\resolventshort}_{\Wii\to\Wii} \\
&\qquad+ \norm{\TPt\resolventshort\TPt-\overline{\bw}\NNE\Nt\NNE^*}_{\Wii\to\Wii} \\
&\le (3+ \reval_{\truncbis}) \ErrSpc{+} \norm{\TPt\resolventshort\TPt-\overline{\bw}\NNE\Nt\NNE^*}_{\Wii\to\Wii} \\
&\le 4 \,\ErrSpc,
\end{align}
where we have used in the last line that $\TPt\resolventshort\TPt=\overline{\bw}\NNE\Nt\NNE^*$ and that $\reval_{\truncbis}\le 1$.

Next, recalling {$\Wxt= \overline{\bw}\NNE^*\NNE$} and using the partial isometry property we have 
\begin{align}
\label{Eq:resolvent_bound_NN_2bis}
\norm{\overline{\bw} \NNE \Nt\NNE^*- \overline{\bw} \NNE \ENt\NNE^*}_{\Wii\to\Wii}
&=\norm{\Wxt^{1/2}[\Nt-\ENt]\Wxt^{1/2} } \\
&=  \norm{\Wxt^{-1/2}\Cxt\Wxt^{-1/2}- \w{M}_\param [\Wxt^{-1/2}\ECxt\Wxt^{-1/2}]\w{M}_\param } \notag\\
&\le  \norm{\MX{I}-\w{M}_\param}(1+\norm{\w{M}_\param}) + \norm{\w{M}_\param}^2\norm{\Wxt^{-1/2}[\ECxt-\Cxt]\Wxt^{-1/2}},
\end{align}
where we obtain the second line by adding and subtracting terms $\Qt\w{M}_\param$ and $\w{M}_\param\Qt \w{M}_\param$, and using that $\norm{\Qt}\leq1$. But since,  
\begin{align*}
\norm{I-\w{M}_\param}
&=\norm{I-\Wxt^{1/2}\EWxt^{-1}\Wxt^{1/2}}\\
&=\norm{\Wxt^{1/2}\Wxt^{-1}\Wxt^{1/2}-\Wxt^{1/2}\EWxt^{-1}\Wxt^{1/2}}\\
&=\norm{\Wxt^{1/2}\Wxt^{-1}(\EWxt-\Wxt)\EWxt^{-1}\Wxt^{1/2}}=\norm{\Wxt^{-1/2}(\EWxt-\Wxt)\Wxt^{-1/2}\Wxt^{1/2}\EWxt^{-1}\Wxt^{1/2}}\\
&\le \norm{\Wxt^{-1/2}(\EWxt-\Wxt)\Wxt^{-1/2}}\norm{\w{M}_\param} ,
\end{align*}
we conclude 
\begin{multline*}
\norm{\overline{\bw} \NNE \Nt\NNE^*- \overline{\bw} \NNE \ENt\NNE^*}_{\Wii\to\Wii}
\leq \norm{\w{M}_\param}\Big[ \norm{\w{M}_\param}\norm{\Wxt^{-1/2}(\ECxt{-}\Cxt)\Wxt^{-1/2}}\\
+(1{+}\norm{\w{M}_\param})\norm{\Wxt^{-1/2}(\EWxt{-}\Wxt)\Wxt^{-1/2}} \Big].
\end{multline*}
Finally, we observe that  $\norm{\overline{\bw} \NNE \Nt\NNE^*}_{\Wii\to\Wii}\leq 1$ implies
\[
\norm{\overline{\bw} \NNE \ENt\NNE^*- {\ebw} \NNE \ENt\NNE^*}_{\Wii\to\Wii}\leq \tfrac{|\ebw-\overline{\bw|}}{\overline{\bw}}\big[1+ \norm{\overline{\bw} \NNE \Nt\NNE^*- \overline{\bw} \NNE \ENt\NNE^*}_{\Wii\to\Wii}\big].
\]
By Lemma \ref{lem:weight_concentration_bernstein}, we have $|\ebw-\overline{\bw}|/\overline{\bw}\leq \Errwhpv$. Combining this with the above bounds and Proposition \ref{prop:bound_leftright}, we conclude that, for $\Errwhp\leq 1/3$, the following holds with probability at least $1-2\delta$:
\begin{equation}\label{eq:resolvent_bound}
\norm{\resolventshort^\trunc-\ebw \NNE \ENt\NNE^*}_{\Wii\to\Wii}\le 4\,\ErrSpc + 6\,\Errwhp + 3\, \Errwhpv.
\end{equation}

We can think of the empirical estimation of the spectral decomposition of $\resolventinvshort$ as a two step perturbation scheme, where each steps originates from the corresponding learning scheme of LITL. 

In the first phase we learn representations, yielding that the spectral decomposition of $\resolventinvshort$ is approximated by the one of $\TPt\resolventshort\TPt$,  which, using the properties of low rank eigenvalue problems \cite{SS1990}, is characterized by the definite generalized eigenvalue problem $\Cxt u_i^{\param} = \reval^{\param}_i \Wxt u_i^{\param}$ and $[u^\param_i]^\top\Wxt u_i^\param = \overline{\bw}$, $0{\leq} i{\leq} m$, where normalized eigenfunctions of $\TPt\resolventshort\TPt$ are given by $\gefun_i^\param {=} \NNE u_i^{\param}$. On the other hand, in the second phase  perturbations are induced by replacing $\Wxt$ and $\Cxt$ by their empirical counterparts. So, to control these errors and obtain the final spectral bounds, we use the following standard perturbation result on definite matrix pencils. 

\begin{theorem}[\cite{SS1990}]
\label{thm:stewart-sun}
Let $A$ and $B$ be compact self-adjoint operators on a Hilbert space, with $B$ positive definite. Consider the generalized eigenvalue problem $A x {=} \lambda B x$ with eigenvalues $\lambda_1 {\le} \lambda_2 {\le} \cdots$ and corresponding $B$-orthonormal eigenvectors $x_1, x_2, \dots$ satisfying $x_i^\top B x_j {=} \delta_{ij}$. Let $\widehat{A}$ and $\widehat{B}$ be self-adjoint perturbations, and denote
$\Delta A {=} \widehat{A} {-} A$ and $\Delta B {=} \widehat{B} {-} B$. If $\|\Delta B\|<|B^{-1}\|^{-1}$, then for eigenvalues $\widehat{\lambda}_i$ and $\widehat{B}$-orthonormal eigenvectors $\widehat{x}_i$ of the generalized eigenvalue problem $\widehat{A} \widehat{x} = \widehat{\lambda} \widehat{B} \widehat{x}$ the following bounds hold.
\begin{enumerate}
    \item \textbf{Eigenvalue perturbation:} For every $i\in\mathbb N_0$,
    \[
    |\lambda_i - \widehat{\lambda}_i|
    \le \frac{  \|\Delta A\| + \lambda_i \|\Delta B\| }
              {\|B^{-1}\|^{-1}- \|\Delta B\|}.
    \]
    \item \textbf{Eigenvector perturbation:} Let $\mathcal{X}$ and $\widehat{\mathcal{X}}$ be invariant subspaces corresponding to eigenvalues $\lambda_{i_1}, \dots, \lambda_{i_k}$ and $\widehat{\lambda}_{i_1}, \dots, \widehat{\lambda}_{i_k}$, respectively.
    Define the gap $\gap_{\{i_1,\ldots,i_k\}} = \min\bigl\{ |\mu - \lambda_{i_j}| : \mu \in \Spec(B^{-1}A)\setminus\{\lambda_{i_1},\ldots\lambda_{i_k}\}\,j\in[k]\bigr\}$, and assume $\gap_{\{i_1,\ldots,i_k\}}  > 0$. Let $\Theta_B$ be the $B$-based angles between $\mathcal{X}$ and $\widehat{\mathcal{X}}$ and $\|\Delta A\|\wedge \|\Delta B\|\le \eps$. Then
    \[
    \|\sin \Theta_B\|
    {\le} \frac{ \|\Delta A\|+ (\max_j |\lambda_{i_j}|) \|\Delta B\|}
              {\gap_{\{i_1,\ldots,i_k\}} \big(\|B^{-1}\|^{-1} {-} \|\Delta B\| \big)},\]
            and, consequently, 
    \[\|x - \widehat{x}\|_B
    {\le}  \frac{\sqrt{2} (1 + |\lambda|)\eps}
                        {\gap_{\{i_1,\ldots,i_k\}} \big( \|B^{-1}\|^{-1} {-}\eps \big)}
        {+ }{\tiny{\cal O}}(\eps^2).
    \]
 
\end{enumerate}
\end{theorem}

First, recalling \eqref{Eq:resolvent_bound_NN_2}, we apply the above result in the space $\Wii$ when $\Delta B = 0$ and obtain
\begin{equation}\label{eq:spec_bound_param}
|\reval_i^\param-\reval_i|\leq 3\,\ErrSpc\quad\text{ and }\quad \norm{\gefun_i^\param - \gefun_i}_{\Wii\to\Wii}\leq\frac{3\,\ErrSpc}{\gap_i},
\end{equation}
where $\gap_i=\min_{\reval\in\Spec(\resolventinvshort)\setminus\{i\}}|\reval-\reval_i|$

Next, relying on the Proposition \ref{prop:bound_leftright} and, as before assuming $\ErrSpc\le1/2$ and $\Errwhp\le \overline{\bw}/3$, we have that $\|\Wxt^{-1}\|^{-1}\geq \overline{\bw}/2$ and w.p.a.l. $1-\delta$

\begin{equation}\label{eq:eval_bound_emp}
\frac{|\reval_i^\param-\ereval_i|}{1+\reval_i^\param}\leq \frac{6\,\Errwhp}{\overline{\bw}}\quad\text{ and }\quad \norm{u_i^\param /\sqrt{\overline{\bw}}- \w{u}_i/\sqrt{\ebw}}_{\Wxt}\leq\frac{6\sqrt{2}\,\Errwhp}{\overline{\bw}(\gap_i-3\,\ErrSpc)}(1+\reval_i^\param).
\end{equation}

Recalling that $\sqrt{\overline{\bw}}\NNE\Wxt^{-1/2}$ is partial isometry, we have 
\[
\norm{u_i^\param /\sqrt{\overline{\bw}}- \w{u}_i/\sqrt{\ebw}}_{\Wxt}=\sqrt{\overline{\bw}}\norm{\NNE [u_i^\param /\sqrt{\overline{\bw}}- \w{u}_i/\sqrt{\ebw}]}_{\Wii\to\Wii} = \norm{\gefun_i^\param - \egefun_i \sqrt{\overline{\bw}}/\sqrt{\ebw}}_{\Wii\to\Wii} 
\]
and, hence
\[
\norm{\gefun_i^\param - \egefun_i}_{\Wii\to\Wii}\leq\frac{6\sqrt{2}\,\Errwhp}{\overline{\bw}(\gap_i-3\,\ErrSpc)}(1+\reval_i^\param) + (1+\norm{\gefun_i^\param - \egefun_i}_{\Wii\to\Wii}) \Big|1-\sqrt{\tfrac{\overline{\bw}}{\ebw}}\Big|.
\]
Now, assuming that $\Errwhpv\le1/2$, we have that $|1-\sqrt{\overline{\bw}}/\sqrt{\ebw}|\leq \Errwhpv$, and we conclude
\begin{equation}\label{eq:efun_bound_emp}
\norm{\gefun_i^\param - \egefun_i}_{\Wii\to\Wii}\leq\frac{12\sqrt{2}\,(1+\reval_i^\param)}{\overline{\bw}(\gap_i-3\,\ErrSpc)} \, \Errwhp+  2\,\Errwhpv.
\end{equation}

Finally, collecting \eqref{eq:spec_bound_param}, \eqref{eq:eval_bound_emp} and \eqref{eq:efun_bound_emp}
we have shown that 
\begin{equation}\label{eq:spec_bound_empirical}
|\ereval_i-\reval_i|{\leq} 3(\ErrSpc{+}2\,\Errwhp/\overline{\bw})
\end{equation}
and
\begin{equation}\label{eq:spec_bound_empirical2}
\norm{\egefun_i {-} \gefun_i}_{\Wii\to\Wii}{\leq}\frac{3\,\ErrSpc}{\gap_i}{+}\frac{12\sqrt{2}\,(2{+}3\,\ErrSpc)}{\overline{\bw}(\gap_i{-}3\,\ErrSpc)} \, \Errwhp{+}2\,\Errwhpv.
\end{equation}

\subsection{Generator estimation}\label{app:generator_bounds}

Start by observing that 
\[
\EQt=\EWxt^{-1/2}\ECxt\EWxt^{-1/2}=\frac{1}{\ebw}\sum_{i=0}^m \ereval_i \EWxt^{1/2}\w{u}_i \w{u}_i^\top \EWxt^{1/2},
\]
which, if all $\ereval_i>0$, implies that 
\begin{multline*}
\EGE {=} \sum_{i=0}^m \egeval_i \egefun_i \otimes_{\Wii}\egefun_i = \shift \sum_{i=0}^m (1-1/\ereval_i) \NNE \w{u}_i\w{u}_i^\top\NNE^* \\=\shift\ebw \NNE\EWxt^{-1/2}\Big(\frac{1}{\ebw}\sum_{i=0}^m \EWxt^{1/2}\w{u}_i\w{u}_i^\top\EWxt^{1/2} -\EQt^{-1}\Big)\EWxt^{-1/2}\NNE^*.
\end{multline*}
But, then, since $\big(\tfrac{1}{\sqrt{\ebw}}\EWxt^{1/2}\w{u}_i\big)_{i=0}^{{m}}$
forms orthonormal basis, we obtain $\EGE {=} {\ebw}\shift \NNE[\EWxt^{-1}-\ECxt^{-1}]\NNE^*$. 

To proceed, we in the following lemma we guarantee that indeed $\ereval_i$'s are positive and $\EQt$ is invertible matrix with high probability. 

\begin{lemma}\label{lem:Qinv}
Under the assumptions of Proposition \ref{prop:bound_leftright} and Lemma \ref{lem:weight_concentration_bernstein}, if $\Errwhpv\leq1/2$ and $\, 12\,\Errwhp/\overline{\bw}+6\,\ErrSpc<\reval_\trunc$, then with probability in the i.i.d. draw of $(x_i')_{i=1}^n$ from $\bim$, it holds that, $\EQt$ is invertible and $\norm{\EQt^{-1}}\le 2 \reval_\trunc^{-1}$. Consequently, $\ebw\,\norm{\NNE\ECxt^{-1}\NNE^*}_{\Wiishort\to\Wiishort}\leq 4 \,/\reval_\trunc$ and  $\norm{\EGE}_{\Wiishort\to\Wiishort}\leq 6\shift /\reval_\trunc$. 
\end{lemma}
\begin{proof}
First, observe that due to \eqref{eq:spec_bound_empirical} and \eqref{eq:spec_bound_empirical2}, %$14 \Errwhp \KLou{\cancel{\reval_\trunc}}+ 8\ErrSpc<\reval_\trunc^{\KLou{\cancel{2}}}$ 
%\helene{this is inconsistent with lemma assumption \textcolor{blue}{$\, 12\,\Errwhp+6\,\overline{\bw}\,\ErrSpc<\reval_\trunc$}. Should be $3\ErrSpc + 6\Errwhp/\overline{\bw}<\reval_m/\textcolor{blue}{2\overline{\bw}}$}
$3\ErrSpc + 6\Errwhp/\overline{\bw}<\reval_m/2$ implies that $\lambda_{\min}(\EQt)=\ereval_m>\reval_m/2>0$. To conclude, observe that
\begin{multline*}
\ebw\norm{\NNE\ECxt^{-1}\NNE^*}_{\Wiishort\to\Wiishort}=\ebw\norm{\NNE\EWxt^{-1/2}\EQt^{-1}\EWxt^{-1/2}\NNE^*}_{\Wiishort\to\Wiishort}\\\leq \frac{\ebw}{\overline{\bw}}\norm{\w{M}_\param} \norm{\EQt^{-1}}\leq 2\frac{1+\Errwhpv}{1-\Errwhp}\reval_m^{-1}.
\end{multline*}
and $\norm{\EGE}_{\Wiishort\to\Wiishort}\leq \shift\,\frac{\ebw}{\overline{\bw}}\norm{\w{M}_\param}\,(1+\norm{\EQt^{-1}})$.
\end{proof}
Next, we obtain the bound on $\norm{\GE_\trunc - \EGE}_{\Wiishort\to\Wiip{s}}$, for $s\in[0,1]$. To that end, we split the analysis as 
\begin{equation}
\norm{\GE_\trunc - \EGE}_{\Wiishort\to\Wiip{s}}\leq \norm{\TP_\trunc^\bot\resolventshort^{\frac{1-s}{2}}[\GE_\trunc - \EGE]}_{\Wiishort\to\Wiishort} + \norm{\TP_\trunc\resolventshort^{\frac{1-s}{2}}[\GE_\trunc - \EGE]}_{\Wiishort\to\Wiishort}.    
\end{equation}
For the first term, observe that 
\begin{align*}
 \norm{\TP_\trunc^\bot\resolventshort^{\frac{1-s}{2}}[\GE_\trunc - \EGE]}_{\Wiishort\to\Wiishort}
 &=\norm{\TP_\trunc^\bot\resolventshort^{\frac{1-s}{2}}\TP_\param\EGE}_{\Wiishort\to\Wiishort}\le {\reval_{\truncbis}^{\frac{1-s}{2}}} \norm{\TP_\trunc^\bot\EGE}_{\Wiishort\to\Wiishort}\\
 &\le {\reval_{\truncbis}^{\frac{1-s}{2}}}\norm{[\TP_\trunc^\bot-\TP_\param^\bot]\EGE}_{\Wiishort\to\Wiishort}\le {\reval_{\truncbis}^{\frac{1-s}{2}}} \ErrSpc \norm{\EGE}_{\Wiishort\to\Wiishort} \\&\le {\reval_{\truncbis}^{\frac{1-s}{2}}}\frac{6\,\shift\,\ErrSpc }{\reval_\trunc}.
\end{align*}
For the second term, since $\GE_\trunc = \shift(\TP_\trunc - [\resolventshort^\trunc]^\dagger)$, start by bounding 
\begin{align*}
\norm{&\TP_\trunc\resolventshort^{\frac{1-s}{2}}[\TP_\trunc{-} {{\ebw}}\NNE\EWxt^{-1}\NNE^*]}_{\Wiishort\to\Wiishort}
 \\
 &\le {\norm{\TP_\trunc\resolventshort^{\frac{1-s}{2}}}_{\Wiishort\to\Wiishort}} \left(\norm{\TP_\trunc{-}\TP_\param}_{\Wiishort\to\Wiishort}+{{\ebw}}\norm{\NNE[\Wxt^{-1}{-}\EWxt^{-1}]\NNE^*}_{\Wiishort\to\Wiishort}\right)\\
 &\le \ErrSpc + \frac{\ebw}{\overline{\bw}}\norm{I-\w{M}_\param}\le\ErrSpc +\Errwhp \tfrac{1+\Errwhpv}{1-\Errwhp}\leq\ErrSpc +\Errwhp, 
\end{align*}
and, then,
\begin{align*}
\norm{[\resolventshort^\trunc&]^{\frac{1-s}{2}}[[\resolventshort^\trunc]^\dagger\!\!-\!\!  {\ebw} \NNE\ECxt^{-1}\NNE^*]}_{\Wiishort\to\Wiishort}\\
 & = \norm{[\resolventshort^\trunc]^{\frac{1-s}{2}}[[\resolventshort^\trunc]^\dagger\!\!-\!\! \TP_\trunc  {\ebw} \NNE\ECxt^{-1}\NNE^*]}_{\Wiishort\to\Wiishort}\\
& {\le} \norm{[\resolventshort^\trunc]^{\frac{1-s}{2}}[[\resolventshort^\trunc]^\dagger\TP_\param\!-\!  {{\ebw}}\TP_\trunc\NNE\ECxt^{-1}\NNE^*]}_{\Wiishort\to\Wiishort} {+} \norm{[\resolventshort^\trunc]^{\frac{-1-s}{2}}\TP_\param^\bot}_{\Wiishort\to\Wiishort} \\
 &\le \norm{[\resolventshort^\trunc]^{\frac{1-s}{2}}[[\resolventshort^\trunc]^\dagger\TP_\param\!-\!  {{\ebw}}\TP_\trunc\NNE\ECxt^{-1}\NNE^*]}_{\Wiishort\to\Wiishort}\!+\!\frac{
\norm{\TP_\trunc\TP_\param^\bot}_{\Wiishort\to\Wiishort}}{\reval_{\trunc}^{\frac{1+s}{2}}}.
\end{align*}
%\textcolor{red}{KARIM: $A^\dagger(B-A)B^\dagger = A^\dagger P_{\range(B)} - P_{\range(A)}B^\dagger$ }
But then, using $A^\dagger(B-A)B^\dagger = A^\dagger P_{\range(B)} - P_{\range(A)}B^\dagger$ with $A= \resolventshort^\trunc$ and
\begin{align*}
B&{=}({{\ebw}}\NNE\ECxt^{-1}\NNE^*)^{-1}{=} \big(\tfrac{\ebw}{\overline{\bw}} \sqrt{\overline{\bw}}\NNE\Wxt^{-1/2}[\Wxt^{1/2}\ECxt^{-1}\Wxt^{1/2}]\Wxt^{-1/2}\NNE^*\sqrt{\overline{\bw}}\big)^{-1} \\
&{=}\tfrac{\overline{\bw}^2}{\ebw}\NNE\Wxt^{-1/2}[\Wxt^{1/2}\ECxt^{-1}\Wxt^{1/2}]^{-1}\Wxt^{-1/2}\NNE^*=\tfrac{\overline{\bw}^2}{\ebw}\NNE\Wxt^{-1}\ECxt\Wxt^{-1}\NNE,
\end{align*}
with Lemma \ref{lem:Qinv}, we have 
\begin{align*}
 \norm{[\resolventshort^\trunc]^{\frac{1-s}{2}}[[\resolventshort^\trunc]^\dagger\!\!&-\!\!  {\ebw} \NNE\ECxt^{-1}\NNE^*]}_{\Wiishort\to\Wiishort}\\
 & {\le} \frac{\norm{\resolventshort^\trunc\!-\! \tfrac{\overline{\bw}^2}{\ebw} \NNE\Wxt^{-1}\ECxt\Wxt^{-1}\NNE^*}_{\Wiishort\to\Wiishort} \norm{{\ebw} \NNE\ECxt^{-1}\NNE^* }_{\Wiishort\to\Wiishort} \!+\!\ErrSpc}{\reval_{\trunc}^{\frac{1+s}{2}}}\\
  & {\le} 4\frac{\norm{\resolventshort^\trunc\!-\!\tfrac{\overline{\bw}^2}{\ebw}\NNE\Wxt^{-1}\ECxt\Wxt^{-1}\NNE^*}_{\Wiishort\to\Wiishort}}{\reval_{\trunc}^{\frac{3+s}{2}}} +  \frac{\ErrSpc}{\reval_{\trunc}^{\frac{1+s}{2}}}
\end{align*}

Now, by splitting the norm in the first term as
\begin{multline*}
\norm{\resolventshort^\trunc\!-\!\tfrac{\overline{\bw}^2}{\ebw}\NNE\Wxt^{-1}\ECxt\Wxt^{-1}\NNE^*}_{\Wiishort\to\Wiishort}{\leq}\norm{\resolventshort^\trunc\!-\!\overline{\bw} \NNE\Nt\NNE^*}_{\Wiishort\to\Wiishort}\\{+}\norm{\Wxt^{-1/2}[\ECxt{-}\Cxt]\Wxt^{-1/2}} {+}\norm{\Wxt^{-1/2}\ECxt\Wxt^{-1/2}}\big|1{-}\tfrac{\overline{\bw}}{\ebw}\big|, 
\end{multline*}
we conclude 
\begin{align*}
 \norm{[\resolventshort^\trunc]^{\frac{1-s}{2}}[[\resolventshort^\trunc]^\dagger\!\!-\!\!  {\ebw} \NNE\ECxt^{-1}\NNE^*]}_{\Wiishort\to\Wiishort}&{\le} 4\frac{4\ErrSpc+\Errwhp+\Errwhpv\frac{1+\Errwhp}{1-\Errwhpv}}{\reval_{\trunc}^{\frac{3+s}{2}}} +  \frac{\ErrSpc}{\reval_{\trunc}^{\frac{1+s}{2}}}\\
&\le 4\, \frac{4\ErrSpc+\Errwhp+3\Errwhpv}{\reval_{\trunc}^{\frac{3+s}{2}}} +  \frac{\ErrSpc}{\reval_{\trunc}^{\frac{1+s}{2}}}.
\end{align*}

Therefore, collecting the above, for all $s\in[0,1]$ we obtain 
\begin{multline*}
\shift^{-1}\norm{\GE_\trunc - \EGE}_{\Wiishort\to\Wiip{s}}{\leq} \frac{6{\reval_{\truncbis}^{\frac{1-s}{2}}}\,\ErrSpc }{\reval_{\trunc}}{+} {\ErrSpc} {+} \Errwhp\\{+}4\, \frac{4\ErrSpc+\Errwhp+3\Errwhpv}{\reval_{\trunc}^{\frac{3+s}{2}}} +  \frac{\ErrSpc}{\reval_{\trunc}^{\frac{1+s}{2}}},
\end{multline*}
and, hence, 
\begin{align}\label{eq:Lbound}
\norm{\GE_\trunc - \EGE}_{\Wiishort\to\Wiip{s}} &\le \shift\left[\frac{(7 \reval_\trunc +16)\,\ErrSpc+ 4\,\Errwhp + 12\Errwhpv}{\reval_{\trunc}^{\frac{3+s}{2}}}+ \ErrSpc + \Errwhp\right].
\end{align}

\subsection{Gradient estimation}\label{app:gradient_bounds}

To estimate the gradient, recall that $\w{\nabla \PT} = \w{\MX{G}}\NNe$, where, using the notation of Proposition \ref{prop:bound_xi}, 
\[
\w{\MX{G}}^\top=-\sum_{i=0}^{m}\egeval_i\,\w{u}_i\w{u}_i^\top \Big[\tfrac{1}{n\ebw} \sum_{j=1}^n\xi(x_j')\Big]=-\shift\,[\EWxt^{-1}-\ECxt^{-1}]\w{\xi},
\]
 where $\w{\xi}=\tfrac{1}{n}\sum_{i=1}^n\xi(x_i')=[\w{\xi}_1\,\vert\,\cdots\,\vert \w{\xi}_d]\in\R^{(m+1)\times d}$.

We recall also that for any $p\geq 1$
$$
\norm{{\nabla \PT} -\w{\nabla \PT} }_{\Wiip{p}} = \sqrt{\sum_{k\in[d]}\norm{\partial_k \PT -\w{\partial_k \PT}}^2_{\Wiip{p}}}. 
$$

\bigskip

In the following we consider the regularity assumption \textbf{(RA)}: There exists $p\geq 1$ such that for every $k\in[d]$, $\norm{\partial_k \PT}_{\Wiip{p}}<\infty$, that is $\norm{\nabla \PT}_{\Wiip{p}}<\infty$. Under this assumption, for every $s\leq p$ and every $k\in[d]$, we have that: 
\begin{align}
\norm{\TP_\trunc^\bot \partial_k \PT}_{\Wiip{s}} &= \norm{\resolventinvshort^{(p-s)/2}\TPTo\partial_k\PT}_{\Wiip{p}}\le \reval_{\truncbis}^{(p-s)/2}
\norm{\partial_k \PT}_{\Wiip{p}} 
\label{Eq:bound_projected_forces}
\end{align}

Next, recalling $\w{\xi}=[\w{\xi}_1\,\vert\,\cdots\,\vert \w{\xi}_d]\in\R^{(m+1)\times d}$ and denoting $\xi = \EE_{x'\sim\bim}\xi(x')= [\xi_1\,\vert\,\cdots\,\vert \xi_d(x')]\in\R^{(m+1)\times d}$, we have 
\[
-\EGE f_k {=}- \shift\,\frac{\ebw}{\overline{\bw}}\,\NNE[\EWxt^{-1}{-}\ECxt^{-1}]\EE_{x'\in\bim}[\xi_k(x')]\quad\text{ and }\quad \w{\partial_k\PT} {=} \shift\, \NNE[\EWxt^{-1}{-}\ECxt^{-1}]\w{\xi}_k,
\]
which than implies 
\begin{align}
\sqrt{\sum_{k\in[d]}\norm{\TPT\partial_k \PT -\w{\partial_k \PT}}^2_{\Wiip{s}}}
\nonumber
&\le \sqrt{2\sum_{k\in[d]}\norm{(\GE_\trunc - \EGE)f_k}^2_{\Wiip{s}}} + \sqrt{2\sum_{k\in[d]}\norm{\EGE f_k -\w{\partial_k \PT}}^2_{\Wiip{s}}}\\
\nonumber
& \le \norm{\GE_\trunc - \EGE}_{\Wiishort\to\Wiip{s}}\sqrt{2\sum_{k\in[d]}\norm{f_k}^2_{\Wiishort}}+ \sqrt{2\sum_{k\in[d]}\norm{\EGE f_k -\w{\partial_k \PT}}^2_{ {\Wiishort}} }\\
\nonumber
&=\sqrt{2 \EE_{x\sim\im}\norm{x}^2}\norm{\GE_\trunc - \EGE}_{\Wiishort\to\Wiip{s}}\\
\nonumber
&\qquad+ 
\sqrt{2}\shift\norm{\NNE[\EWxt^{-1}{-}\ECxt^{-1}](\w{\xi}-\tfrac{\ebw}{\overline{\bw}}\xi)} \\ \nonumber
&=\sqrt{2 \EE_{x\sim\im}\norm{x}^2}\norm{\GE_\trunc - \EGE}_{\Wiishort\to\Wiip{s}}\\
\nonumber
&\qquad+ 
\sqrt{2}\shift\norm{\NNE[\EWxt^{-1}{-}\ECxt^{-1}]\Wxt^{1/2}\Wxt^{-1/2}(\w{\xi}-\tfrac{\ebw}{\overline{\bw}}\xi)} \\
\nonumber
&\le \sqrt{2 \EE_{x\sim\im}\norm{x}^2}\norm{\GE_\trunc - \EGE}_{\Wiishort\to\Wiip{s}} \\
\nonumber
&\qquad+ \sqrt{2}\norm{\EGE}_{\Wiishort\to {\Wiishort}}\,\tfrac{\sqrt{\overline{\bw}}}{\ebw} \norm{\Wxt^{-1/2}[\w{\xi}-\tfrac{\ebw}{\overline{\bw}}\xi]}\notag\\
\nonumber
& \le \sqrt{2 \EE_{x\sim\im}\norm{x}^2}\norm{\GE_\trunc - \EGE}_{\Wiishort\to\Wiip{s}} \\
&\qquad+  \sqrt{2}\norm{\EGE}_{\Wiishort\to {\Wiishort}}\,\sqrt{\overline{\bw}} \norm{\Wxt^{-1/2}[\tfrac{1}{\ebw}\w{\xi}-\tfrac{1}{\overline{\bw}}\xi]}. 
\label{Eq:bound_partialU}
\end{align}

Using that $\Errwhpv<1/2$ from Lemma \ref{lem:weight_concentration_bernstein} we have that $|1/\ebw-1/\overline{\bw}|\leq 2\Errwhpv/\overline{\bw}$ and $1/\ebw \leq 2/\overline{\bw}$. Moreover, from the proof of Proposition \ref{prop:bound_xi} we have 
\begin{multline*}
\overline{\bw}\norm{\Wxt^{-1/2}\xi}^2\leq \overline{\bw}\EE_{x'\sim\bim}\norm{\Wxt^{-1/2}\xi(x')}^2\\\le \overline{\bw}(\EE_{x'\sim\bim}\norm{x'}^2\vee1) \esssup_{x'\sim\bim}\norm{\Wxt^{-1/2}\wx(x')\otimes\wx(x')\Wxt^{-1/2}} \leq  c_\param\,(\sigma_{\bim}^2\vee1) \overline{\bw},
\end{multline*}
where $\sigma_{\bim}^2 = \EE_{x'\sim\bim}\|x'\|^2$. 

Therefore, bounding the last term in \eqref{Eq:bound_partialU} as 
\[
\sqrt{\overline{\bw}} \norm{\Wxt^{-1/2}[\tfrac{1}{\ebw}\w{\xi}-\tfrac{1}{\overline{\bw}}\xi]}\leq \tfrac{1}{\ebw}\, \sqrt{\overline{\bw}} \norm{\Wxt^{-1/2}[\w{\xi}-\xi]}+ \sqrt{\overline{\bw}} \norm{\Wxt^{-1/2}\xi}|\tfrac{1}{\ebw}-\tfrac{1}{\overline{\bw}}|\leq \frac{2\ErrwhpXi}{\overline{\bw}}+\frac{2\Errwhpv \sigma_{\bim}\sqrt{c_\param}}{\sqrt{\overline{\bw}}}
\]
and collecting all bounds
%assuming that $\shift/\sqrt{\reval_\trunc}\leq1$ 
we conclude that
\begin{multline}
\norm{\nabla\PT {-} \w{\MX{G}}\NNe}_{\Wiip{s}}
{\le} \reval_{\truncbis}^{(p-s)/2}
\norm{\nabla\PT}_{\Wiip{p}} {+} \shift\,\sigma_{\bim}\, \sqrt{2}\Big[\frac{24\,\ErrSpc {+} 5\Errwhp{+}12\Errwhpv}{\reval_{\trunc}^{\frac{3+s}{2}}}\\  + 12\frac{\ErrwhpXi/\sigma_{\bim}{+}\Errwhpv \sqrt{c_\param\,\overline{\bw}}}{\overline{\bw}\,\reval_{\trunc}}\Big].
\label{Eq:bound_forces}
\end{multline}

\subsection{Proof of Theorem \ref{thm:main_stat_1}}
\label{app:proof_main_stat_1}

\begin{proof}[Proof of Theorem \ref{thm:main_stat_1}]

First, observe that, since $\geval_i = \shift(1 - 1/\reval_i)$, we have $\reval_i = \shift/(\shift - \geval_i)$, and, hence:
\begin{align*}
\gap_i
&= \min_{\ell \neq i} |\reval_\ell - \reval_i| 
= \min_{\ell \neq i} \left|\frac{\shift}{\shift - \geval_\ell} - \frac{\shift}{\shift - \geval_i}\right| 
= \shift \min_{\ell \neq i} \left|\frac{\geval_i - \geval_\ell}{(\shift - \geval_\ell)(\shift - \geval_i)}\right|\\
&= \frac{\shift}{{|\shift - \geval_i|}} \min_{\ell \neq i} \frac{|\geval_i - \geval_\ell|}{|\shift - \geval_\ell|} \in (0,1).
\end{align*}

Next, to prove spectral bounds we use \eqref{eq:spec_bound_empirical}, \eqref{eq:spec_bound_empirical2} and substitute definitions of $\ErrSpc$, $\Errwhp$ and $\Errwhpv$ as follows
\begin{align*}
\Errwhp/\overline{\bw}&\lesssim \varepsilon_{n}(\delta) := \cst e^{2\beta\norm{\bias}_\infty}\sqrt{\frac{m\, d}{\shift\,\beta\,n}\log\frac{m}{\delta}},\quad\text{ and}\\
\Errwhpv  &\lesssim \frac{\sigma_\bw }{\overline{\bw}} \sqrt{\frac{2\log(2/\delta)}{n}} \leq e^{\beta\norm{\bias}_\infty} \sqrt{\frac{2\log(2/\delta)}{n}}\lesssim \varepsilon_{n}(\delta),
\end{align*}
where we used that $\sigma_{\bim}^2 = \EE_{x'\sim\bim}\|x'\|^2\leq\sigma_{\im}^2e^{\beta\norm{\bias}_\infty}$, $\sigma_{v}^2 = \EE_{x'\sim\bim}\bw(x')^2\leq \overline{\bw}e^{\beta\norm{\bias}_\infty}$ and $\overline{\bw}\geq e^{-\beta\norm{\bias}_\infty}$.

Therefore, 
\begin{equation}\label{eq:spec_bound_empirical_bis}
|\ereval_i-\reval_i|\lesssim \ErrSpc+ \varepsilon_{n}(\delta) \quad\text{ and }\quad \norm{\egefun_i {-} \gefun_i}_{\Wii\to\Wii}\lesssim \frac{\ErrSpc}{\gap_i}{+}\frac{ \varepsilon_{n}(\delta)}{[\gap_i{-}3\,\ErrSpc]_{+}},
\end{equation}
and the bound follows, after some algebra,  by expressing gap in generator eigenvlues w.r.t. gap in the resolvent ones.

To conclude the proof,  observe that
\begin{align*}
\ErrwhpXi \reval_m^{\frac{1+s}{2}} / \overline{\bw}&\lesssim \cst e^{2\beta\norm{\bias}_\infty} \sqrt{\frac{\shift\,m}{(\shift-\eval_m)n}} \max\left\{\frac{1}{\shift\beta}\sqrt{\log\frac{m+d}{\delta}},\, \sigma_{\im}e^{\beta\norm{\bias}_\infty/2}\log\frac{1}{\delta}\right\} \\
&\lesssim \sigma_{\im}\varepsilon_{n}(\delta),
\end{align*}
whenever $\beta(\shift-\eval_m)\geq1$, while
\[
\sigma_{\bim}\Errwhpv\sqrt{c_\param/\overline{\bw}}\leq \sigma_{\bim}\cst e^{\beta\norm{\bias}_\infty}\sqrt{\frac{m\, d}{\shift\,\beta}}\Errwhpv\lesssim \sigma_{\bim} e^{2\beta\norm{\bias}_\infty} \sqrt{\frac{m\, d}{\shift\,\beta\,n}\log\frac{2}{\delta}}\lesssim\varepsilon_{n}(\delta).
\]

So, substituting this in \eqref{Eq:bound_forces} we obtain
\[
\norm{\nabla\PT - \w{\MX{G}}\NNe}_{\Wiip{s}}
{\le} \reval_{\truncbis}^{(p-s)/2}
\norm{\nabla\PT}_{\Wiip{p}} {+}   \shift\reval_{\trunc}^{-\frac{3+s}{2}} \sigma_{\im}\,\big(\varepsilon_{n}(\delta)+\ErrSpc\big).
\]

\end{proof}

\section{Experiments}\label{app:exp}

In this section we expand on the experimental validation of LITL and its applications in physics and ML. While details for each experiment are provided, the code for reproducing them is available upon request and will be included in the camera ready version of the paper.  

\subsection{1D double-well potential}\label{app:1d2w}

The goal of this experiment is to recover the spectral decomposition of the generator corresponding to the 1D double-well potential. We compared LILT to prior methods \cite{devergne2024biased} and \cite{ZHANG2022diffusionNN} for the estimation of the four eigenpairs $(\lambda_j, \refun_j)_{j=1}^{4}$. We note that for the method of \cite{ZHANG2022diffusionNN}, even in the best-performing trial, we were able to identify only two distinct eigenfunctions, with eigenvalues $0.6$ and $8$, which are close to the reference values.

For our method LILT in this experiment, we used standard feed forward neural networks, with tanh activation function. We used $\shift=100$, $\regnn=5$ and trained 20 features. One of the strengths of LITL is the estimation of the drift. In this section, we compare our estimation with respect to the ground truth. In Figure \ref{fig:double_well_drift}, left panel we show how the drift can be projected on the basis of the eigenfunctions of $\GE$ obtained by numerically solving the eigenvalue problem. We note, that already even with the ground-truth eigenfunctions, one needs a sufficiently high number of basis functions to express the drift. On the right panel, we show our estimation of the drift using LITL: it matches very well the ground truth on the training set. Finally, we compare the dynamics produced from the estimation of the drift from LITL with the true one in Figure \ref{fig:1D2W_flow}. We started 500 simulations in the left well and looked at the time evolution of the probability distribution, both with the target potential and with LITL. LITL reproduces very well the time evolution of the probability distribution, without having to be trained with data coming from the target potential.

In Figure \ref{fig:double_well}
%\vladi{maybe we keep only table here and have figures in the appendix?} 
we show the accuracy of the estimated eigenfunctions by computing uncertainties over 4 different trials with different random initialization. The results of \cite{ZHANG2022diffusionNN} are omitted due to frequent uncontrolled collapse of learned subspace to slower modes, likely caused by larger number of sensitive hyper-parameters. The other two methods perform very well on the first eigenfunction, which characterizes the transition between the two wells, the only slow mode of the system. But then, as we increase the eigenfunction index, the variance between trials for the second method increases, while LITL's one remains. 

\begin{figure}
    \centering
    \vspace{-0.3cm}
    \includegraphics[width=0.3\linewidth]{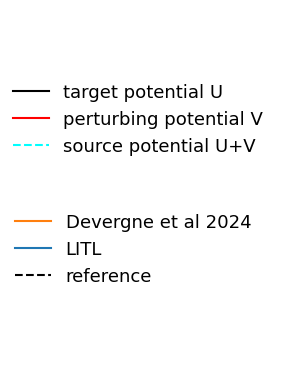}\includegraphics[width=0.65\linewidth]{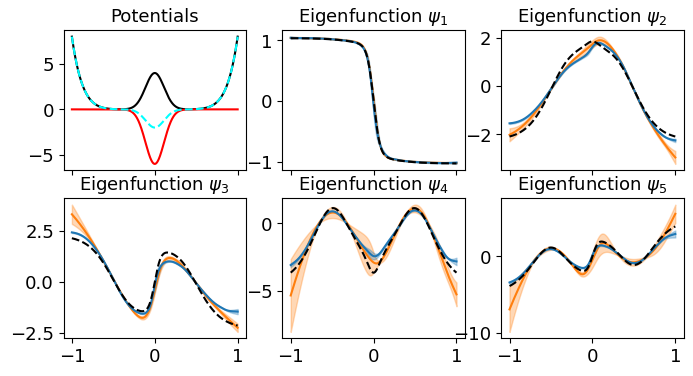}
    \caption{Top left panel: the different potentials involved in this experiment in units of $1/\beta$. Other panels: comparison of the eigenfunctions obtained with LITL and with the methods presented in \cite{devergne2024biased} and \cite{ZHANG2022diffusionNN}. Error bars have been computed using 4 different models with different random initial parameters}
    \label{fig:double_well}
    \vspace{-0.7cm}
\end{figure}

\begin{figure}
    \centering
    \includegraphics[width=\linewidth]{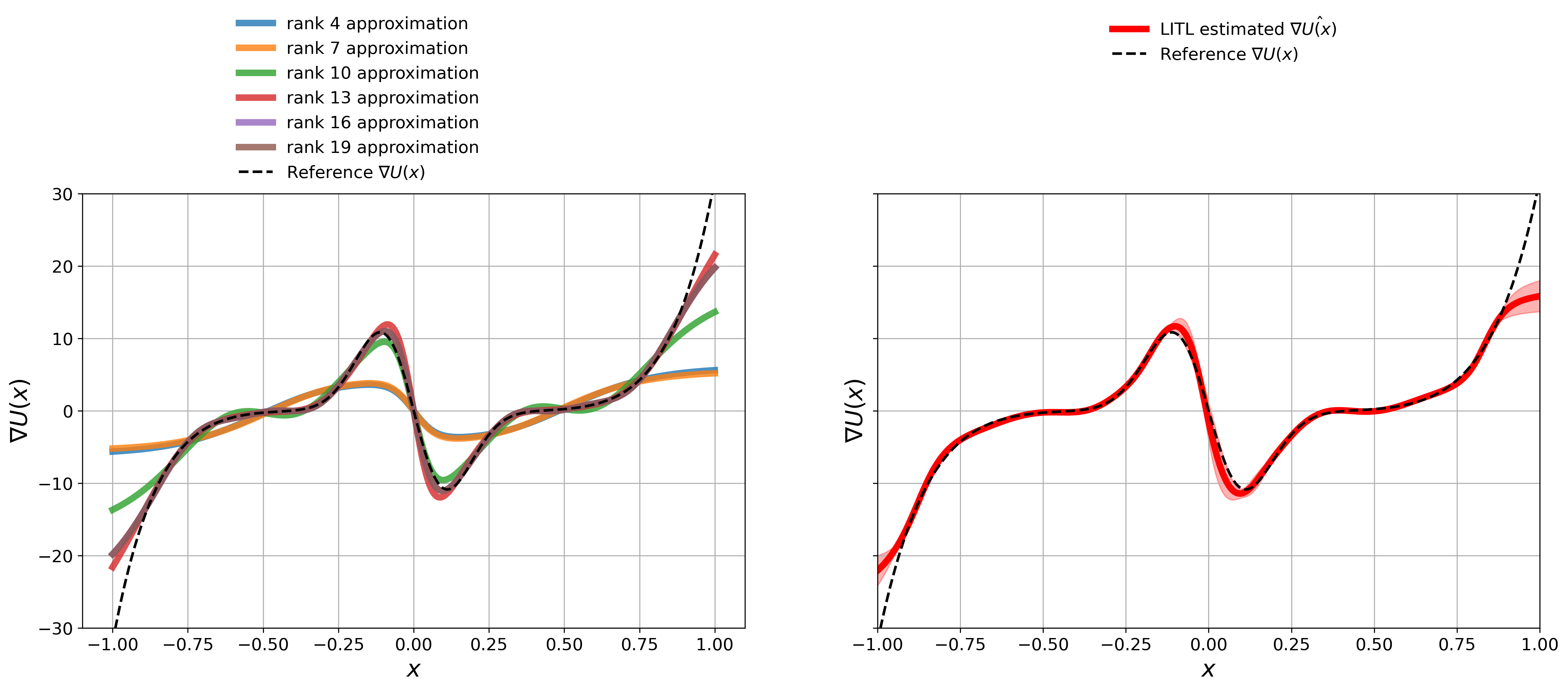}
    \caption{\textit{1D double-well Langevin dynamics.} Gradient of the target potential projected on the leading eigenspaces (left), and its estimation from data (right). Uncertainty is computed over 4 different models}
    \label{fig:double_well_drift}
\end{figure}

\begin{figure}[ht!]
    \centering
   \includegraphics[width=0.9\linewidth]{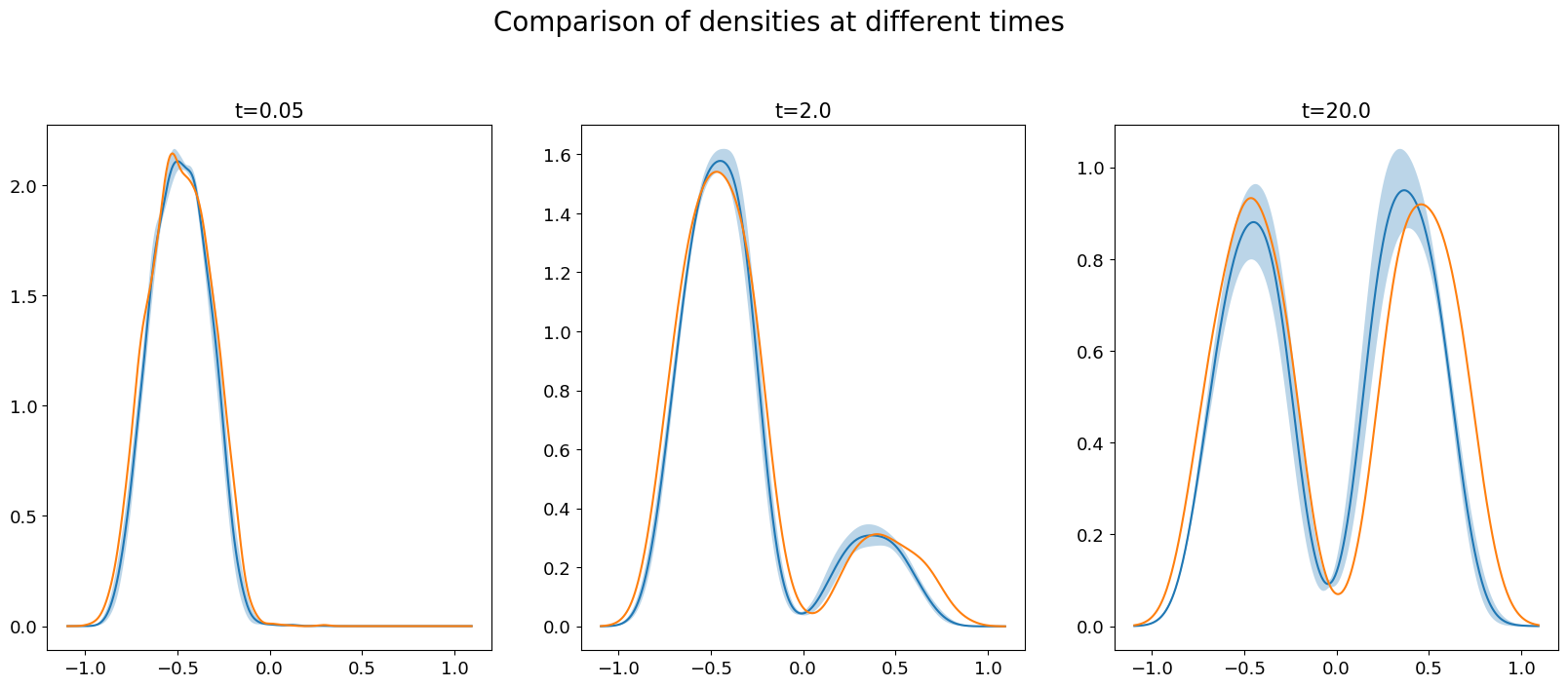}
    \caption{\textit{1D double-well Langevin dynamics.} Target flow from an initial distribution and its estimation from data.}
    \label{fig:1D2W_flow}
\end{figure}

\subsection{Chignolin}\label{app:chignolin}

For this experiment, we used the schnet architecture \cite{schnet} as implemented in the mlcolvar library \cite{mlcolvar}. Only the positions of the alpha carbons were taken into account, and for more expressibility, the one hot encoding of the atomic number node feature information was replaced with the residue type. We used with 3 interaction blocks, 14 RBF functions and an hidden dimension of 64. The value of $\shift$ used was 5000 with $\alpha=1.0$. To optimize the weights, we used the AdamW optimizer. The variational ansatz for the features was of the form: $Z_{\theta}(x) = G_\theta {\rm softmax}(h_\theta)$. Where $G_\theta$ is a linear layer and $h_\theta$ is the output of the graph neural network. 

\subsection{Cobalt nanoparticle}\label{app:cobalt}

\begin{figure}[h!]
    \centering
    \includegraphics[width=0.75\linewidth]{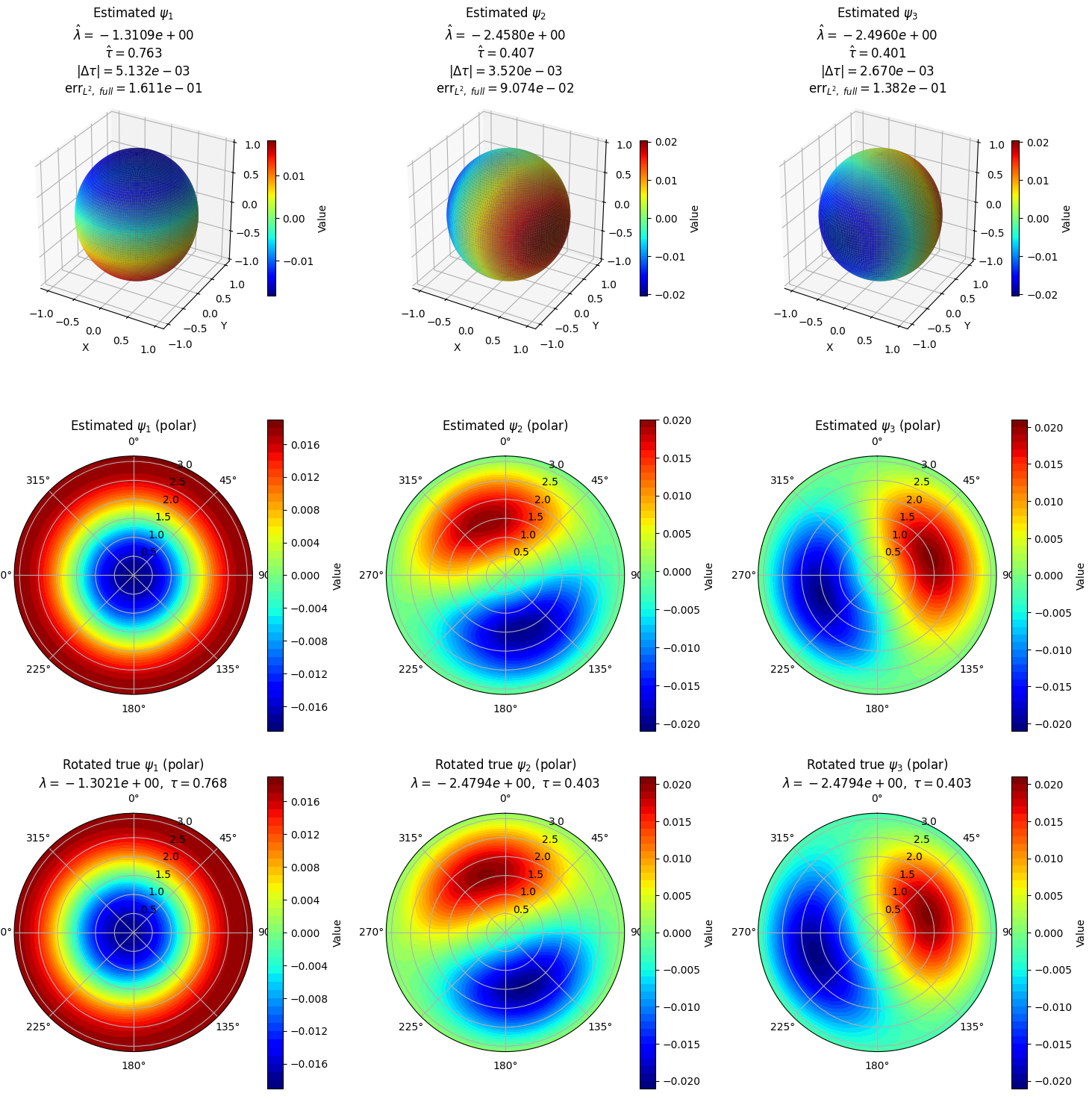}\\
    \caption{\textit{First triplet of leading eigenfunctions of the Langevin generator for a cobalt nanoparticle potential.} First row show 3D visualization on the unit sphere $\mathbb{S}^2$, colored by eigenfunction value, while the second one shows polar coordinate representation $\psi(\theta,\varphi)$. Third row shows reult obtained by LITL. Multiplicity of learned eigenvalues (i.e. time-scales) and spatial symmetry of eigenfunction are consistent with symmetries of the potential energy, and eigenvalue and eignfunction errors are reported in the header.}
    \label{fig:eigenfunctions_app_1}
\end{figure}

\begin{figure}[h!]
    \centering
    \includegraphics[width=0.75\linewidth]{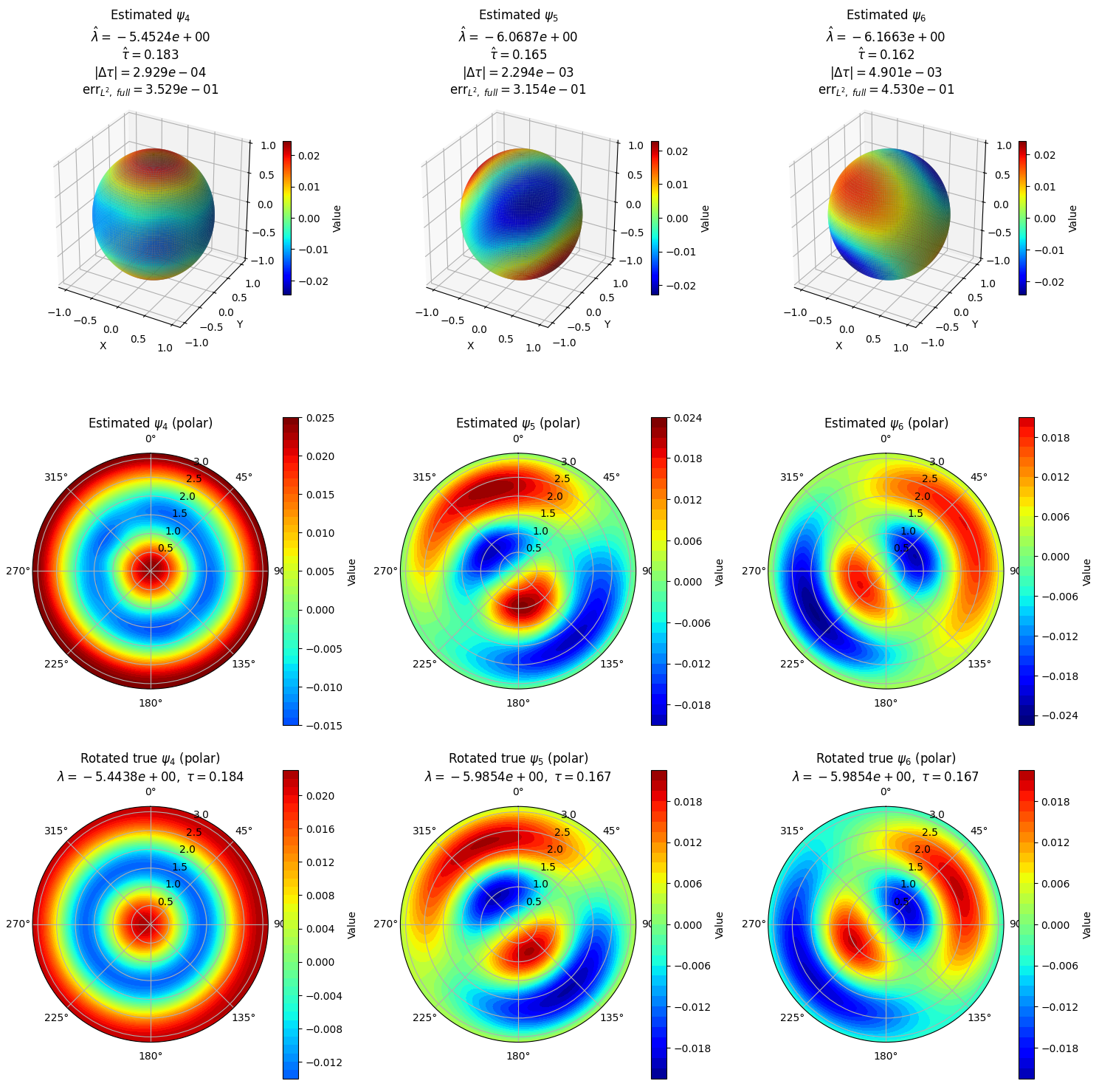}\\
    \caption{\textit{Second triplet of leading eigenfunctions of the Langevin generator for a cobalt nanoparticle potential.} First row show 3D visualization on the unit sphere $\mathbb{S}^2$, colored by eigenfunction value, while the second one shows polar coordinate representation $\psi(\theta,\varphi)$. Third row shows reult obtained by LITL. Multiplicity of learned eigenvalues (i.e. time-scales) and spatial symmetry of eigenfunction are consistent with symmetries of the potential energy, and eigenvalue and eignfunction errors are reported in the header.}
    \label{fig:eigenfunctions_app_2}
\end{figure}

\begin{figure}[h!]
    \centering
    \includegraphics[width=0.75\linewidth]{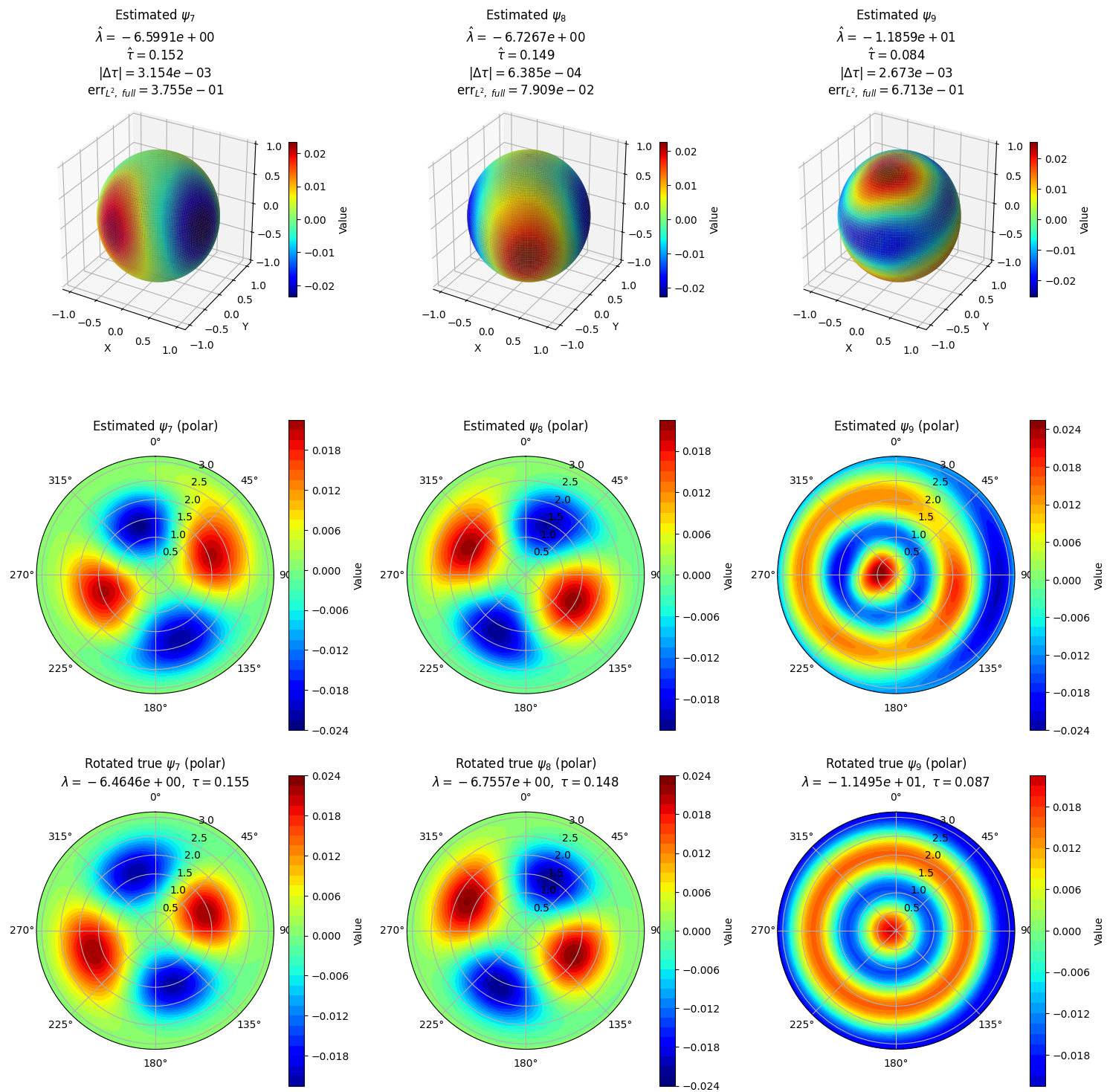}\\
    \caption{\textit{Third triplet of leading eigenfunctions of the Langevin generator for a cobalt nanoparticle potential.} First row show 3D visualization on the unit sphere $\mathbb{S}^2$, colored by eigenfunction value, while the second one shows polar coordinate representation $\psi(\theta,\varphi)$. Third row shows reult obtained by LITL. Multiplicity of learned eigenvalues (i.e. time-scales) and spatial symmetry of eigenfunction are consistent with symmetries of the potential energy, and eigenvalue and eignfunction errors are reported in the header.}
    \label{fig:eigenfunctions_app_3}
\end{figure}

\textbf{Cobalt nanoparticle magnetization dynamics.} We learn the spectral properties of the Langevin generator on a sphere $\mathbb{S}^2$, modeling the magnetization dynamics of a cobalt nanoparticle under the potential $U(\mathbf{m}) \,{=}\, {-}K_u (\mathbf{m} \,{\cdot}\, \mathbf{n}_u)^2 \,{+}\, K_c (m_x^2 m_y^2 \,{+}\, m_y^2 m_z^2 \,{+}\, m_z^2 m_x^2)$, $\mathbf{n}_u$ being easy axis of uniaxal anisotropy. \citep{garanin1997fokker,garcia1998langevin}. For uniaxal and cubic anisotropy constants $K_u \,{=}\, 1.0$ and $K_c \,{=}\, 0.3$, respectively, four metastable states are separated by energy barriers. Understanding the thermodynamics and kinetics of cobalt nanoparticles, in particular the metastable states and transition rates revealed by the Langevin generator, helps set switching fields for memory devices, control blocking temperatures and data retention times, and optimize heat dissipation in hyperthermia applications.
Our method infers the generator's eigenvalues and eigenfunctions using only potential evaluations - without accessing its gradient - simulating experimental scenarios where analytic form of the potential is unknown, but instead  anisotropy landscapes are empirically reconstructed via techniques like torque magnetometry or micro-SQUID measurements \cite{jamet2001magnetic} which probe switching field distributions without direct gradient information. Furthermore, learning the generator directly from sampled trajectories enables prediction of long-time magnetization relaxation, identification of dominant transition pathways, and extraction of effective energy barriers, all critical aspect for designing nanoparticle-based hyperthermia agents, where heating efficiency depends on the slowest dynamical modes and only equilibrium data may be available.

In magnetic nanoparticle systems, knowledge of metastable states directly governs key technological metrics. The number, depth, and connectivity of these states determine the switching field in magnetic memory devices—each metastable orientation can encode a distinct bit. The energy barriers between states set the blocking temperature and data retention time in storage applications, with $\beta\Delta U > 60$ enabling non-volatile memory stable for years. For magnetic hyperthermia cancer treatment, the transition rates between states control the heat dissipation rate under alternating fields; optimizing the anisotropy landscape maximizes tumor heating while minimizing nanoparticle dosage. In magnetic sensors, engineered metastable states enhance sensitivity by creating steep magnetization transitions at specific field thresholds. Furthermore, the multi-state dynamics produce multi-exponential relaxation in AC susceptibility measurements, allowing experimentalists to fingerprint nanoparticle size distributions and anisotropy disorders.

Remarkably, our method was able to recover over ten eigenfunctions with the correct symmetry groups of the model, see Figures \ref{fig:eigenfunctions_app_1}-\ref{fig:eigenfunctions_app_3}, purely from black-box potential evaluations, without prior knowledge of the functional form. The learned eigenfunctions exhibit degeneracy patterns characteristic of the underlying $D2h$ point group symmetry including a singlet stationary distribution and doublets of symmetry-related eigenfunctions linked by rotations of $\pi/2$ in azimuthal angle $\phi$, and reflections through the equatorial plane. Identifying these symmetries is fundamental because they dictate the multiplicity of relaxation timescales.

In magnetic nanoparticles, such symmetry dictates the multiplicity of relaxation timescales, and knowing it reveals whether the anisotropy is purely uniaxial or includes cubic contributions, directly impacting applications from data storage (bit stability) to hyperthermia (heat dissipation pathways). The ability of our method to extract this symmetry structure from potential-only queries demonstrates that the method can infer not just dynamical rates but also the underlying geometric constraints of experimentally measured systems.  To summarize, LITL's capacity to map the metastable landscape isn't merely academic, it enables rational design of nanoparticles for targeted applications by linking atomic-scale anisotropy to macroscopic magnetic performance unlocking exciting new avenues of research.

\subsection{Fairness-aware transfer via latent diffusion.}\label{app:fairness}

We evaluate the proposed method on the Adult Income dataset ({\it https://archive.ics.uci.edu/dataset/2/adult}), a standard benchmark for fairness in binary classification, in the challenging unawareness case \cite{chen2019fairness,donini2018empirical,zemel2013learning}. The task is to predict whether an individual’s annual income exceeds \$50K, using demographic and employment-related features. Following common practice, we consider \emph{gender} as the sensitive attribute and measure fairness by Equal Opportunity (EO) given by the \emph{True Positive Rate (TPR) difference} between males and females. The dataset is split into \textit{training} (60\%), \textit{audit} (25\%) and \textit{test} (15\%) datasets, with the sensitive variable removed from the training set. The audit data is used to returns black-box feedback indicating whether a model  satisfies the desired fairness criterion. 

We first learn a classifier $f_{\omega}(x)\,{=}\,\softmax(a_{\omega}^\top e_{\omega}(x))$ via empirical risk minimization on the training set, where $\omega$ denote the classifier's parameters. In particular, $e_{\omega}(x) \in \mathbb{S}^{d-1}$ denotes the network latent (inner) representation, normalized to lie on the sphere, and $a_\omega\in\R^{d}$. This base classifier achieves strong predictive performance but exhibits a non-negligible TPR disparity between groups, both on the held-out audit set and on the test set.

We adopt a geometric viewpoint in which the latent representations produced by the trained encoder induces a compact state space on $\mathbb{S}^{d-1}$, allowing us to model their variability using intrinsic diffusion processes. As a \textit{source} reference dynamics, we consider a \emph{pure diffusion} (i.e., $\PT'=0$) on the latent sphere, whose generator is the Laplace-Beltrami operator. This source diffusion does not encode any task-specific objective; rather, it captures the intrinsic geometry and dominant modes of variation of the learned representation subject to pure randomness. 

\begin{figure}[h!]
    \centering
\includegraphics[width=0.49\linewidth]{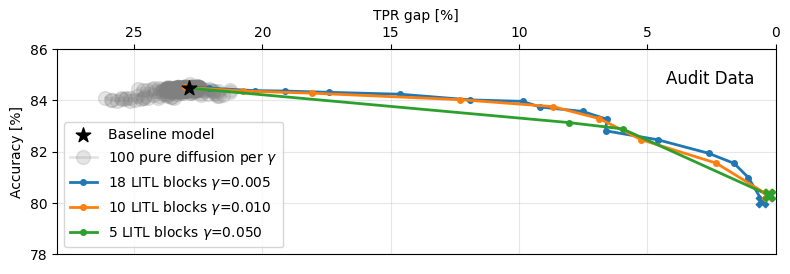}
\includegraphics[width=0.49\linewidth]{figs/LITL_fairness_pareto_tst.png}
    \caption{Accuracy--fairness Pareto curve for the Adult Income experiment. Each point corresponds to a different number of LITL steps. LITL hyperparameters are calibrated on audit set, and then used to achieve substantial reduction in TPR disparity with a limited decrease in accuracy on a test set.}
%\vspace{-.5truecm}
\label{fig:pareto_adult_app}
%\vspace{-.5truecm}
\end{figure}

The EO score objective induces an unknown  target potential $\PT$ on the latent space $\mathbb{S}^{d-1}$, whose values can be evaluated locally by probing the model on the audit data, but whose gradient is not directly accessible. To be precise, we define $\PT(e){=}|\mathrm{TPR}_{\mathrm{male}}({\cal N}_e) - \mathrm{TPR}_{\mathrm{female}}({\cal N}_e)|$, where ${\cal N}_e$ denotes the 100 nearest neighbors of $e\in\mathbb{S}^{d-1}$.

Unlike the source Laplace-Beltrami operator with potential $\PT'{=}\,0$, now the target Langevin generator's eigenfunctions define a hierarchy of fairness related slow and fast directions on the latent manifold. In particular, the leading nontrivial eigenfunctions correspond to large-scale, semantically meaningful variations in the representation, while higher modes capture localized or noisy fluctuations. This multiscale spectral structure provides a natural coordinate system for analyzing and adjusting the model’s fairness behavior in a controlled and interpretable way. 
So, LITL leverages samples from the source diffusion and scalar evaluations of the bias functional $\bias {=} {-}\PT$ obtained from the audit set to estimate the gradient of the induced potential \emph{projected onto the slow eigenspaces of the source generator}. Importantly, this estimation does not rely on differentiating the bias measure with respect to model parameters. Instead, LITL withing its two phases learns a vector field that approximates the intrinsic spherical gradient of the bias-induced potential in directions that are stable under latent diffusion.

\begin{figure}[h!]
    \centering
   \includegraphics[width=0.80\linewidth]{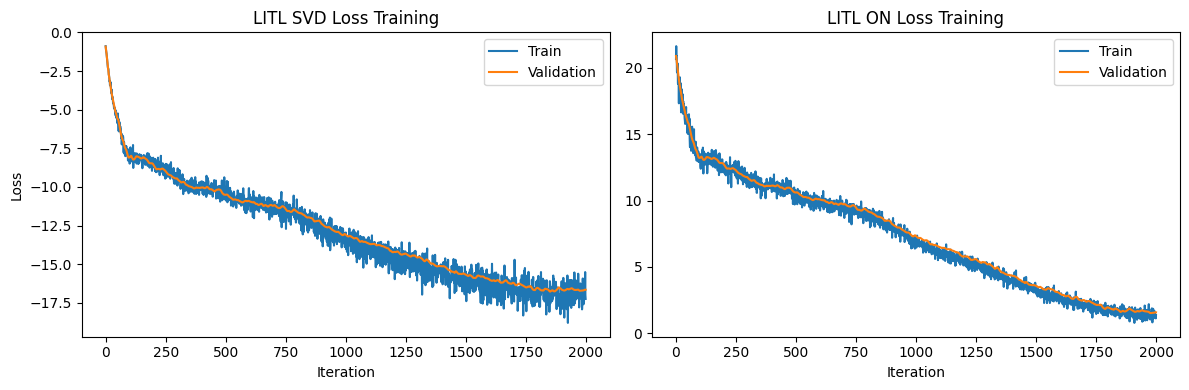}
    \caption{Training dynamics of the LITL encoder. Physics-informed contrastive loss measurign the discrepancy of }
    \label{fig:training_adult}
\end{figure}

In this experiment we sampled $n=5000$ samples from $d=32$ dimensional latent sphere, parametrized $\NNe$ as a multilayer perceptron $[32, 66, 33, 66, 33]$ with $\mathrm{Tanh}$ activations up to penultimate layer. In Figure \ref{fig:training_adult} we show the training and validation losses (on 200 samples) across the epochs. Main SVD loss is shown on the left, while the orthonormal penalty is shown on the right. 

Once LITL encoder $\NNe$ is trained, we proceed with the estimation of the spectrum and gradient field using all training samples. 

The learned gradient field $\w{\nabla}_{\mathbb{S}^{d{-}1}}\PT(e){=}(\MX{I}{-}ee^\top)\w{G}\NNe(e)$ is used to perform short-horizon, geometry-respecting updates on the latent sphere 
\[
e^{(k+1)}_\omega{=}\tfrac{(1+\gamma {e^{(k)}_\omega}^\top\w{G}\NNe(e^{(k)}_\omega)e^{(k)}_\omega{-}\gamma\w{G}\NNe(e^{(k)}_\omega)}{\|(1+\gamma {e^{(k)}_\omega}^\top\w{G}\NNe(e^{(k)}_\omega)e^{(k)}_\omega{-}\gamma\w{G}\NNe(e^{(k)}_\omega)\|},\, k\geq0,
\]
where each update corresponds to a small $\gamma>0$ step along the estimated spherical gradient, followed by projection to the sphere. Repeating this procedure steers the pretrained model $f_{\omega}$ toward configurations with reduced TPR disparity while preserving the dominant representational structure learned during initial training. Since the intervention acts only along slow diffusion modes and is applied over short time scales, it avoids the instability and performance collapse typically associated with direct re-optimization using fairness regularizers. 

We apply the LITL update iteratively and monitor both predictive accuracy and TPR difference on the audit set. Once we have calibrated hyperparameters step size and number of gradient steps, we perform the evaluation of the validated model on a test set. 

\begin{figure}[h!]
    \centering
    \includegraphics[width=0.34\linewidth]{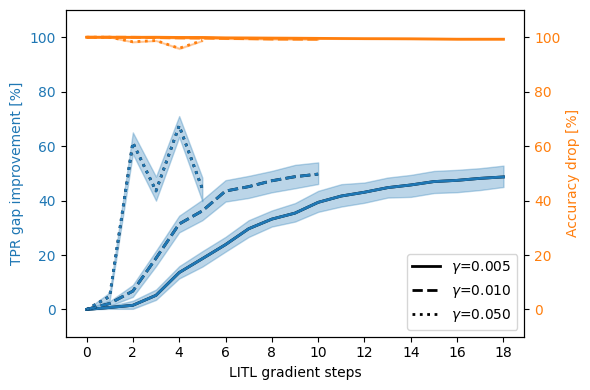}
    \includegraphics[width=0.3\linewidth]{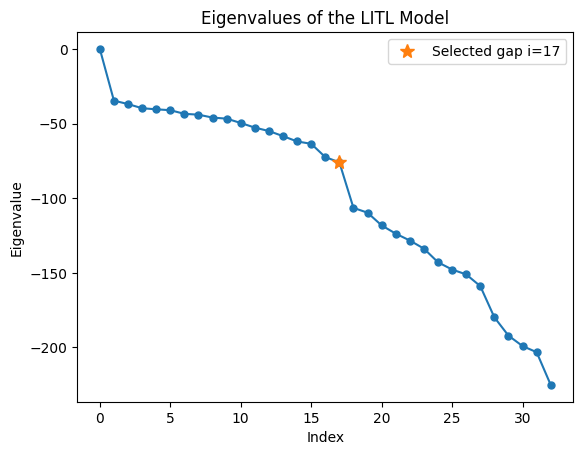}
    \includegraphics[width=0.34\linewidth]{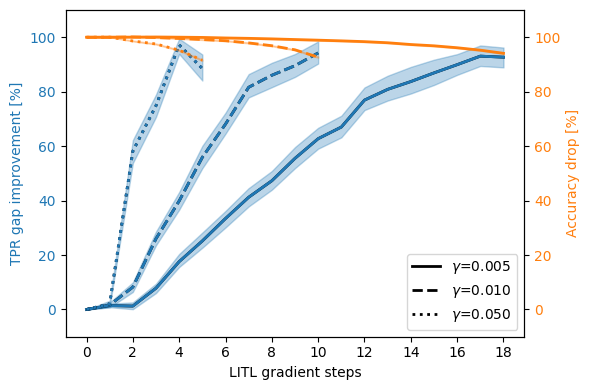}
    \caption{Accuracy deterioration and TPR disparity improvement across LITL gradient steps on 100 random test datasets. On the left the LITL model uses 16 slowest directions, while on the right the flow is across additional 16 faster ones. While slowest manifold corrects the TPR gap while minimally changing the accuracy it reaches the metastable state at around 50\% gap reduction. On the other hand, including faster dynamics allows further reduction, but the transformations of the latent space geometry induce accuracy deterioration. The gap in the spectrum shown in the middle plot shows how the two manifolds are separated. This demonstrates how spectral structure of the generator related to the feedback potential allows one to introduce alignment of the model in a control manner.}
\label{fig:improvement_adult}
\end{figure}

In Figure ~\ref{fig:pareto_adult_app} we report the resulting Pareto curves on audit and test set, illustrating the trade-off between accuracy and fairness induced by the LITL intervention. Each point on the curve corresponds to a different number of LITL steps. As the intervention increases, the TPR difference significantly decreases to acceptable 1\%, while accuracy degrades gracefully. 

Furthermore, since by inspecting the estimated spectrum of the generator, c.f. Figure \ref{fig:improvement_adult} (middle), one observes a clear separation of the slow and faster manifold, we examine how much improvement one obtains by using only first 16 nontrivial eigenfunctions in projecting the gradient of $\PT$ Figure \ref{fig:improvement_adult} (left) compared to using all 32 learned ones, Figure \ref{fig:improvement_adult} (right). Shown results are obtained via 100 randomly chosen  test sets, where the mean improvement is shown with a line, while the shaded are represents 5th and 95th percentile.

\end{document}